\documentclass[letterpaper]{article} 
\usepackage{aaai2026}  
\usepackage{times}  
\usepackage{helvet}  
\usepackage{courier}  
\usepackage[hyphens]{url}  
\usepackage{graphicx} 
\usepackage{natbib}  
\usepackage{caption} 
\usepackage{algorithm}
\usepackage{algorithmic}

\usepackage{newfloat}
\usepackage{listings}

\usepackage{amsmath}
\usepackage{amssymb}
\usepackage{mathtools}
\usepackage{amsthm}
\usepackage{multirow}
\usepackage{threeparttable}
\usepackage{colortbl}
\usepackage{array} 
\usepackage[table]{xcolor}
\usepackage{subfigure}  

\usepackage{mathptmx}
\usepackage{cases}
\usepackage{bm}
\usepackage{booktabs}

\newtheorem{theorem}{Theorem}
\newtheorem{lemma}[theorem]{Lemma}

\DeclareCaptionStyle{ruled}{labelfont=normalfont,labelsep=colon,strut=off} 
\floatstyle{ruled}
\newfloat{listing}{tb}{lst}{}
\floatname{listing}{Listing}
\title{DeepBooTS: Dual‑Stream Residual Boosting for Drift‑Resilient Time‑Series Forecasting}
\author {
    Daojun Liang \textsuperscript{\rm 1,\rm 2},
    Jing Chen \textsuperscript{\rm 1,\rm 2},
    Xiao Wang \textsuperscript{\rm 3, \rm 4} \thanks{Corresponding author.},
    Yinglong Wang \textsuperscript{\rm 1, \rm2} \footnotemark[1],
    Shuo Li \textsuperscript{\rm 5, \rm 6}
}
\affiliations{
    \textsuperscript{\rm 1} Key Laboratory of Computing Power Network and Information Security, Ministry of Education, Shandong Computer Science Center (National Supercomputing Center in Jinan), Qilu University of Technology (Shandong Academy of Sciences), Jinan, 250103, China \\
    \textsuperscript{\rm 2} Shandong Provincial Key Laboratory of Computing Power Internet and Service Computing, Shandong Fundamental Research Center for Computer Science, Jinan, 250103, China \\
    \textsuperscript{\rm 3} School of Intelligent Manufacturing and Control Engineering, Qilu Institute of Technology, Jinan, 250200, China \\
    \textsuperscript{\rm 4} Shandong Provincial Key Laboratory of Industrial Big Data and Intelligent Manufacturing, Qilu Institute of Technology, Jinan, 250200, China \\
    \textsuperscript{\rm 5} Department of Computer and Data Sciences, Case Western Reserve University, Cleveland, OH 44106, USA \\
    \textsuperscript{\rm 6} Department of Biomedical Engineering, Case Western Reserve University, Cleveland, OH 44106, USA \\
    liangdj@sdas.org, chenj@sdas.org, wangxiao@qlit.edu.cn, wangyinglong@qlu.edu.cn, shuo.li11@case.edu
}

\usepackage{bibentry}

\begin{document}

\maketitle

\begin{abstract}
Time-Series (TS) exhibits pronounced non-stationarity. Consequently, most forecasting methods display compromised robustness to concept drift, despite the prevalent application of instance normalization. We tackle this challenge by first analysing concept drift through a bias-variance lens and proving that weighted ensemble reduces variance without increasing bias. These insights motivate DeepBooTS, a novel end-to-end dual-stream residual-decreasing boosting method that progressively reconstructs the intrinsic signal. In our design, each block of a deep model becomes an ensemble of learners with an auxiliary output branch forming a highway to the final prediction. The block‑wise outputs correct the residuals of previous blocks, leading to a learning‑driven decomposition of both inputs and targets. This method enhances versatility and interpretability while substantially improving robustness to concept drift. Extensive experiments, including those on large-scale datasets, show that the proposed method outperforms existing methods by a large margin, yielding an average performance improvement of 15.8\% across various datasets, establishing a new benchmark for TS forecasting.
\end{abstract}

\begin{links}
    \link{Code}{https://github.com/Anoise/DeepBooTS}
    \link{Extended version}{https://arxiv.org/abs/2511.06893}
\end{links}

\section{Introduction}
\label{sec_intro}

Time Series (TS) from natural and engineered systems often evolve under transient conditions \cite{1976_timeseries} and therefore violate stationarity assumptions \cite{hyndman2018forecasting}. Classical statistical models such as ARIMA \cite{piccolo1990distance} and exponential smoothing \cite{gardner1985exponential}, which depend on fixed statistical properties, struggle to track these shifts \cite{de200625}. Deep neural networks have become attractive alternatives because of their powerful nonlinear modeling capabilities \cite{hornik1991approximation}, yet even the latest attention‑based \cite{Zhou2021Informer,nie2022time,liu2023itransformer, liang2024minusformer,liang2025distpred,ye2024frequency,wang2024timemixer, wang2024timexer, liang2025distpred, yu2025drift, yu2025learning} and graph‑neural approaches \cite{wu2019graph,yi2024fouriergnn} provide only marginal gains over simple multilayer perceptrons and incur substantial inference overhead \cite{shao2022spatial, zeng2023transformers, yi2024filternet}.

\begin{figure}[t]
  \centerline{\includegraphics[width=\columnwidth]{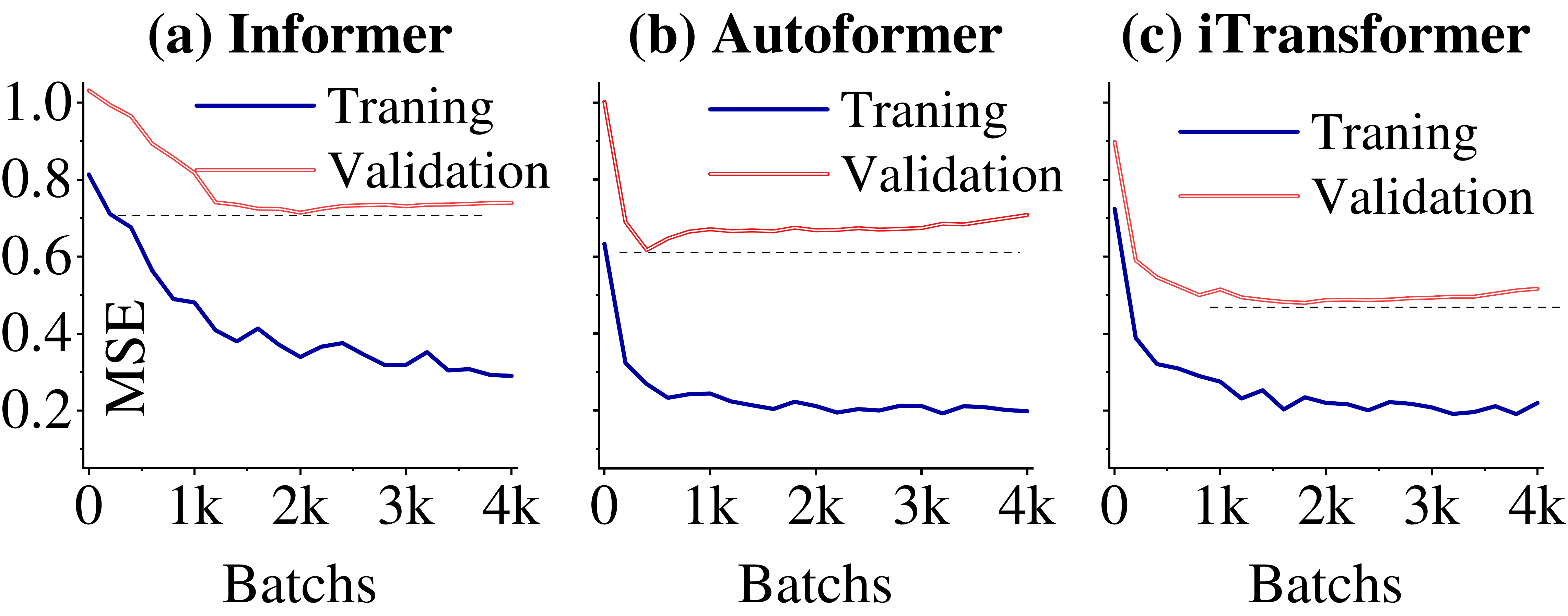}}
  \caption{TS forecasting models suffer from concept drift problems, more experiments are provided in Appendix H.}
  \label{fig_overfit}
\end{figure}

\begin{figure*}[t]
  \centerline{\includegraphics[width=\textwidth]{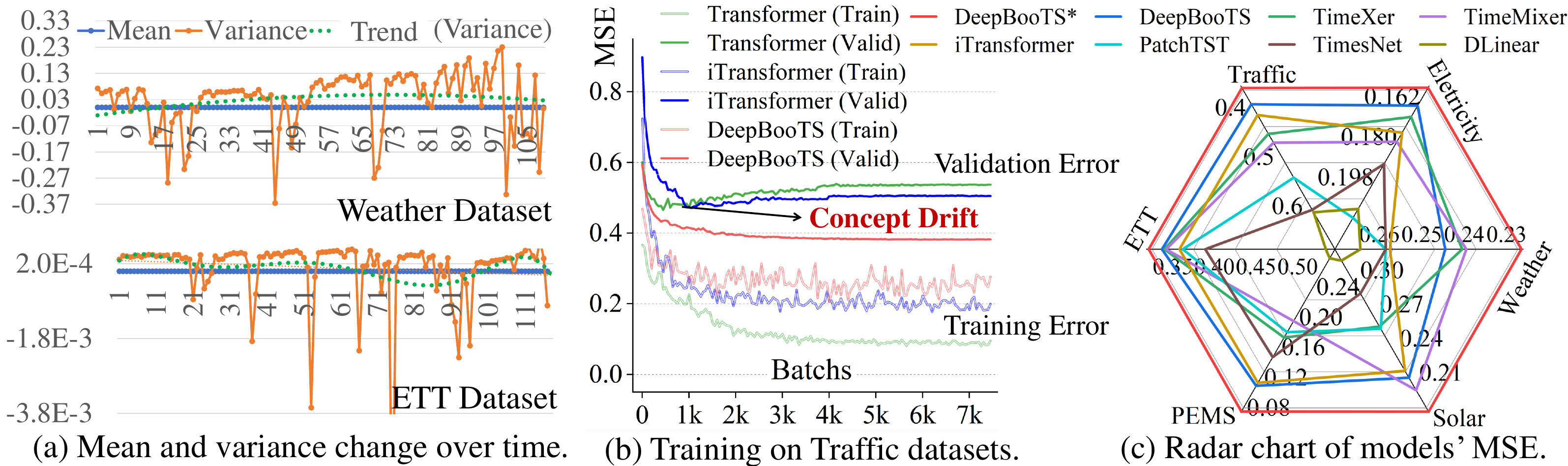}}
  \caption{(a) Concept drift exists in current TS datasets. (b) Concept drift causes the model training error to decrease while the validation error increases. (c) The performance of DeepBooTS after reducing concept drift.
}
  \label{fig_temp}
\end{figure*}

A deeper issue is that leading architectures like Transformers \cite{NIPS2017_Transformer} and their variants are highly susceptible to concept drift---a mismatch between training and testing distributions. Empirical evidence shows that validation error rises early in training even as training error falls, as shown in Fig. \ref{fig_overfit}, a sign that the learned model fails to generalize as the underlying data distribution shifts. Prevailing solutions, including reversible instance normalization (RevIN) \cite{kim2021reversible} and temporal re‑aggregation, e.g., DLinear \cite{zeng2023transformers} and iTransformer \cite{liu2023itransformer}, alleviate mean shifts but leave the variance unstable, which leads to large fluctuations and persistent concept drift. 
This phenomenon occurs across multiple TS forecasting models and datasets, as shown in Fig. \ref{fig_temp}(a), underscoring the need for an approach that directly targets the high‑variance component.

This work addresses the problem by framing concept drift through the bias-variance decomposition. We theoretically show that when the mean bias and noise are fixed, the degree of drift is controlled by the variance of the model predictions. We then prove that ensemble averaging reduces the prediction variance while preserving bias and derive new bounds showing that weighted ensembles achieve a strictly lower error under distribution shift. These results suggest that a properly designed ensemble can substantially mitigate concept drift.
Building on this insight, we introduce DeepBooTS, a novel dual-stream residual-decreasing boosting architecture that progressively reinstates the intrinsic values of a non‑stationary TS. The key innovations are:


\begin{itemize}
  \item We rigorously analyze concept drift through the lens of the bias-variance trade-off and prove that even simple deep ensembles can substantially reduce prediction variance without increasing bias.
  \item An efficient implementation of DeepBooTS is presented. Specifically, the outputs of subsequent blocks subtract the predictions of previous blocks, causing the network to explicitly model and reduce residual errors layer by layer. This residual‑learning mechanism is analogous to gradient boosting, but implemented within a deep network, enhancing robustness to distributional shifts.
  \item A dual‑stream decomposition is designed that decomposes both the input and labels, enabling the model to learn complementary representations while enhancing model versatility and interpretability. 
\end{itemize}

We validate DeepBooTS on a wide range of TS benchmarks, including large‑scale TS datasets. As shown  in Fig. \ref{fig_temp}(c), our method consistently outperforms state‑of‑the‑art baselines, achieving an average 15.8\% improvement. These results, together with our theoretical analysis, demonstrate that DeepBooTS not only advances the accuracy of TS forecasting but also provides a principled solution to the longstanding problem of concept drift.

\section{Methodology}
\label{sec_arch}

\begin{figure*}[!ht]
  \centering 
  \includegraphics[width=0.9\textwidth]{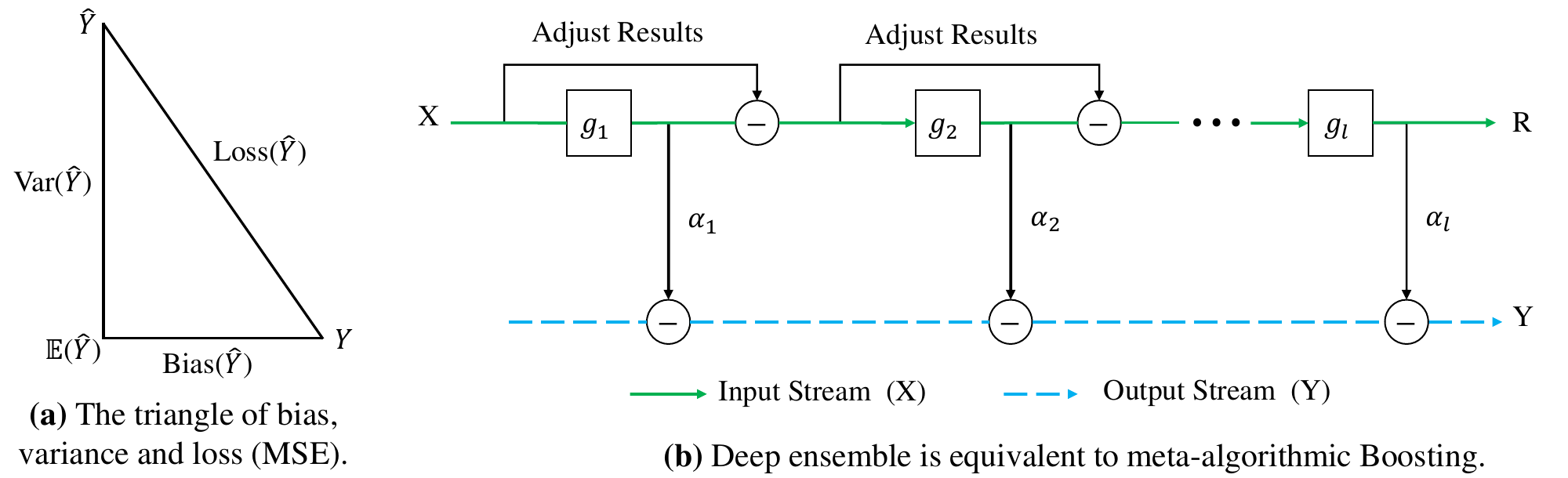}
  \caption{(a) Relationship among model's bias, variance, and loss. (b) Deep boosting ensemble learning process.}
  \label{fig_boost} 
\end{figure*}

\subsection{Deep Ensemble helps Alleviate Concept Drift}
\label{app_overfit}

As analyzed above, RevIN-processed (Z-Score) TS data maintains mean stationarity while exhibiting substantial variance instability.
We consider a forecasting task where the data distribution shifts from an initial distribution $P_0$ to a new distribution $P_t$ at time $t$. 
Formally, let $Y = f_0(X) + \varepsilon$ under $P_0$ (with true regression function $f_0$), and at time $t$ let $Y = f_t(X) + \varepsilon$ under $P_t$, where $f_t$ is the new underlying function after drift, $\varepsilon \sim \mathcal{N}(0, \sigma )$.
For forecasting results $\hat{Y}$, the estimation error (MSE) of the model is
\begin{equation}
  \underbrace{\text{Var}(\hat{Y}) + (\text{Bias}(\hat{Y}))^2 + \sigma^2}_{\text{Test Error}}  =  \underbrace{\mathbb{E}[ (\hat{Y} - Y )^2] + 2\mathbb{E}(\varepsilon(\hat{Y} - \mathcal{Y} ))}_{\text{Training Error}}.
  \label{eq_est_err}
\end{equation}
The proof is given in Appendix B.1.
Eq. \ref{eq_est_err}, which is illustrated in Fig. \ref{fig_boost}(a), shows that when the mean $\text{Bias}(\hat{Y})$ and noise level $\sigma^2$ are fixed, the variance of the data governs the extent of concept drift.
Thus, models incapable of adapting to large data distribution shifts exhibit high prediction variance.
Now, we show that ensemble reduces the variance. 
Without loss of generality, using MSE as the metric and applying the simple average ensemble method, we have
\begin{theorem}\label{th1}
  Let $\hat{Y}_1, \cdots, \hat{Y}_N$ be $N$ i.i.d random variables drawn from some unknown distribution, and the ensemble method is $\bar{Y}=\frac{1}{N}\sum_t \hat{Y}^t$. Then, 
\begin{align}
    & \text{Bias}(\bar{Y}) = \text{Bias}(\hat{Y}), \quad \text{Var}(\bar{Y}) \le \text{Var}(\hat{Y}). \label{eq_th0_b}
    \end{align}
\end{theorem}
\noindent
The proof is given in Appendix B.2.
Although it adopts the simplest ensemble form, it proves that ensembling reduces variance while conserving bias.
If substituting $\hat{f}_0(X)$ (which was trained to approximate $f_0$ on $P_0$) for $\hat{Y}$ in Eq. \ref{eq_est_err}, we have
\begin{equation}
  \text{MSE}(f_t(X) - \hat{f}_0(X)) = \text{Var}(\hat{f}_0(X)) + (\text{Bias}(\hat{f}_0(X)))^2 + \sigma^2
  \label{eq_train_loss}.
\end{equation}
After the distribution shifts to $P_t$, the forecasting error is
\begin{equation}
\text{MSE}_{P_t}(f_0(X)) = \mathbb{E}_{P_t}[\text{Bias}_{P_0}(X)^2] + \mathbb{E}_{P_t}[\text{Var}_{P_0}(f_0(X))] + \sigma^2.
  \label{eq_f0_mse}
\end{equation}
Let $\mathrm{Var}_{P_t}[X] = \sigma_t^2$ with $c^2 = \sigma^2_t/\sigma^2$ to quantify the change in input variance, thus we have
\begin{theorem}\label{th1_5}
  Given a weighted ensemble models $\hat{f}(X) = \frac{1}{L}\sum_{l=1}^L \alpha_l g_l(x)$, where $\sum_{l=1}^L \alpha_l = 1$, we have
\begin{align}
  \text{MSE}_{P_t}(\hat{f}_0) - \text{MSE}_{P_t}(\hat{f}) \approx \frac{(L-1)(1-\alpha_l)}{L}c^2\sigma^2_t \ge 0. 
\end{align}
\end{theorem}
\noindent
The proof is given in Appendix B.3.
It indicates that the ensemble's MSE under $P_t$ is lower than that of the single one.

\subsection{Deep Ensemble Model for TS Forecasting}
\label{sec3_ensem_tsf}

Based on Theorem \ref{th1_5}, we consider a ensemble learner set $\{g_l\}_{l=1}^L$ to minimize
\begin{align}\label{eq_min_loss}
   & \ \ \mathbb{E}(Y - \hat{f}(X))^2 = \mathbb{E}(Y - \sum_{l=1}^{L} \alpha_l g_l(X))^2, \\
  & s.t. \ \ \sum_{l=1}^L \alpha_l = 1, \text{and} \ \  \alpha_l \ge 0, \ \ l=1, \cdots, L. \notag
\end{align}
Here, the objective in Eq. \ref{eq_min_loss} implies that
\begin{equation}
  (Y - \hat{f}(X))^2 = \sum_{l=1}^L \alpha_l (Y-g_l(X))^2 - \sum_{l=1}^L \alpha_l (\hat{f}(X) - g_l(X))^2
  \label{eq_nn_decom_loss}.
\end{equation}
The proof is given in Appendix B.4.
Eq. \ref{eq_nn_decom_loss} proves that the ensemble generalization error can be reduced by increasing the ambiguity without increasing the bias, since the ambiguity item (the 2nd term on the right), which measures the disagreement among the base predictors, is always nonnegative.
Thus, Eq. \ref{eq_nn_decom_loss} indicates that we can design an ensemble learning algorithm to reduce concept drift.

Motivated by this, we propose a deep residual-decreasing ensemble method based on prediction variance minimization and TS decomposition principles, named DeepBooTS. 
As shown in Fig. \ref{fig_boost}(b), DeepBooTS is a two-stream architecture corresponding to Eq. \ref{eq_nn_decom_loss}, which is equivalent to a meta-algorithmic Boosting approach, reducing complexity of the model and thus mitigating the risk of concept drift.
As shown in Fig. \ref{fig_boost}, the `\textit{input}' stream is obviously a decomposition of $X$ since
\begin{align}
  X & = \sum_{l=0}^{L-1} g_l(X) + R_L, \label{eq_a1} 
\end{align}
where $R_L$ is the residual term. 
Further, the aim of the `\textit{output}' stems is to learn $L$ simple learners in a hierarchy where each learner gives more attention (larger weight) to the hard samples from the previous learner. This is equivalent to the Boosting ensemble learning process, where the final prediction is a weighted sum of $L$ simple learners, and the weights are determined by the previous learner.
Let $g_l(x)$ denote the $l$-th learner, and $\alpha_l$ denote the weight of the $l$-th learner.
The overall estimation $\hat{f}(X)$ is a weighted subtraction of the $L$ estimations. For a sample $X$, we have
\begin{align}
    & \hat{f}(X) = i \sum_{l=0}^{\hbar} \alpha_{2l+1} g_{2l+1}(X) - i \sum_{l=0}^{\hbar} \alpha_{2l} g_{2l}(X), \label{eq_a2} \\
    & s.t. \ \  i \sum_{l=0}^{\hbar} \alpha_{2l+1} - i \sum_{l=0}^{\hbar} \alpha_{2l} = 1, \text{and} \ \  \alpha_l \ge 0. \notag
\end{align}
where $ i = 1$ if $L\bmod 2 = 1$, else $i = -1$, and $ \hbar = \lfloor \frac{L}{2} \rfloor$.
Similar to Eq \ref{eq_min_loss}, Eq. \ref{eq_a2} can be rewritten as
\begin{align} \label{eq_a2b} 
          (Y - \hat{f}(X))^2 & = \sum_{l=0}^{\hbar} i(\alpha_{2l+1} - \alpha_{2l}) (Y-g_l(X))^2 \notag \\
          & - \sum_{l=0}^{\hbar} i(\alpha_{2l+1} - \alpha_{2l}) (\hat{f}(X) - g_l(X))^2
\end{align}
Obviously, Eq. \ref{eq_a2b} derived from Eq. \ref{eq_a2} is similar to Eq. \ref{eq_nn_decom_loss}. 
However, Eq. \ref{eq_a2} has more advantages, that is
\begin{theorem}
  Assume that the estimation error of block $g_l(X)$ is $e_l$, $e_l \overset{i.i.d}{\sim}  \mathcal{N}(0, \nu)$. Let $\alpha_l = \alpha$ be the weight of $g_l$, $l\in [0,L] $, and the covariance of estimations of two different blocks by $\mu$, we have
  \begin{align}
    \text{Var}(\hat{Y}) < \frac{4}{L} \alpha^2 (\nu + \mu). \label{eq10}
  \end{align}
\label{th2}
\end{theorem}
The proof is given in Appendix B.5.
Clearly, the variance of DeepBooTS is bound by the estimation error (noise error) of each block, and the covariance between blocks.
It is evident that the subtraction adopted in DeepBooTS can reduce the variance, thereby mitigating concept drift. On the contrary, switching the aggregation operation of the output stream in the DeepBooTS to addition results in an approximate variance of $\frac{4}{L} \alpha^2 \nu + 3\alpha^2 \mu $, which is much larger than that of subtraction in Eq. \ref{eq_a2}.
This phenomenon is also supported by our experiments in Section \ref{ssec_ablation}.
Furthermore, Theorem \ref{th2} also demonstrates that increasing the number of layers $L$ does not escalate the risk of concept drift. It proves that the deep ensemble models can go deeper, and the test error approaches $\text{Bias}(\hat{f}(X))^2 + \sigma^2$ when $L$ is infinite and doesn't consider performance-efficient trade-off.

\subsection{Implementation}
\label{sec_pf_arch}

Here, we employ deep neural networks to implement DeepBooTS.
As shown in Fig. \ref{fig_arch}, the model consists of two primary data streams. One is the input stream decomposed through multiple residual learners, while the other is the output stream that progressively learns the residuals of the supervised signals. 
Along the way, they pass through multiple learners capable of converting signals.
The model is simple and versatile, yet powerful and interpretable. 
The pseudocode is given in Appendix L, and complexity analysis is given in Appendix G.
We now delve into how these properties are incorporated into DeepBooTS.

\textbf{Backbone}: The fundamental building backbone features a fork architecture, which accepts one input $X_l$ and produces two distinct streams, $R_l$ and $O_l$. 
Concretely, $R_l$ is the remaining portion of $X_l$ after it has undergone processing within a neural module, which can be expressed as
\begin{align}
  \hat{X}_l  = \text{Block}_l (X_l), \qquad
  R_l  = X_l - \hat{X}_l, \label{eq2}
\end{align}
Equation \ref{eq2} represents an implicit decomposition of $X$, which differs from the moving average adopted by \cite{wu2021autoformer,zhou2022fedformer,liang2023does} but is similar to \cite{oreshkin2019n}.
Differently, the residual $R_l$ captures what remains unaltered, providing a basis for comparison with the transformed portion. 

In the subsequent steps, the intention is to maximize the utilization of the subtracted portion $\hat{X}_l$. First, $\hat{X}_l$ is projected into the same dimension as the label that is anticipated, $Y$. 
This process can be expressed as
\begin{align}
  O_l & = \text{Predictor}_l (\hat{X}_l), \label{eq3}
\end{align}
where $O_l$ is the prediction results of the $l$-th predictor.
Then, $O_l$ will be subtracted from the outputs of the next predictor sequentially until the final prediction $\hat{Y}$ is achieved.

\textbf{Learner}: 
The basic learner can be constructed utilizing widely-used neural structures, such as fully connected, convolutional, and attention layers.
A ready-made solution is to utilize Attention to learn subtle relationships between attributes. 
As analyzed in Section \ref{sec3_ensem_tsf}, we implement a corrective measure by subtracting the output from the input, i.e.,
\begin{align}
  \hat{X}_{l,1} =  \text{Attention}_{\theta_{l,1}} (X_{l,1}), \quad
  R_{l,1} = X_{l,1} - \delta \hat{X}_{l,1}, \label{eq4}
\end{align}
where $\delta$ is 1 when the module exists, otherwise it is 0. $\delta$ serves to eliminate the Attention layer when it exerts adverse effects, thereby enabling the unimpeded flow of input towards the FeedForward (FF) layers:
\begin{align}
  \hat{X}_{l,2} = \text{FF}_{\theta_{l,2}} (R_{l,1}), \qquad R_{l,2} = X_{l,2} - \hat{X}_{l,2}. \label{eq5}
\end{align}
Also, the Attention can be implemented in the frequency domain by Fast Fourier Transform (FFT) to achieve lightweight and efficient processing, as
\begin{align}
  & Q_l, K_l,V_l   = \text{FFT}_{\theta_Q,\theta_K,\theta_V} (X_l), \\
  & \hat{X}_{l,1} = \text{FFT}^{-1} (\phi (Q_lK_l^T)V).
\end{align}
where $\phi$ is the activation function.

\begin{figure}
  \centering
  \includegraphics[width=\columnwidth]{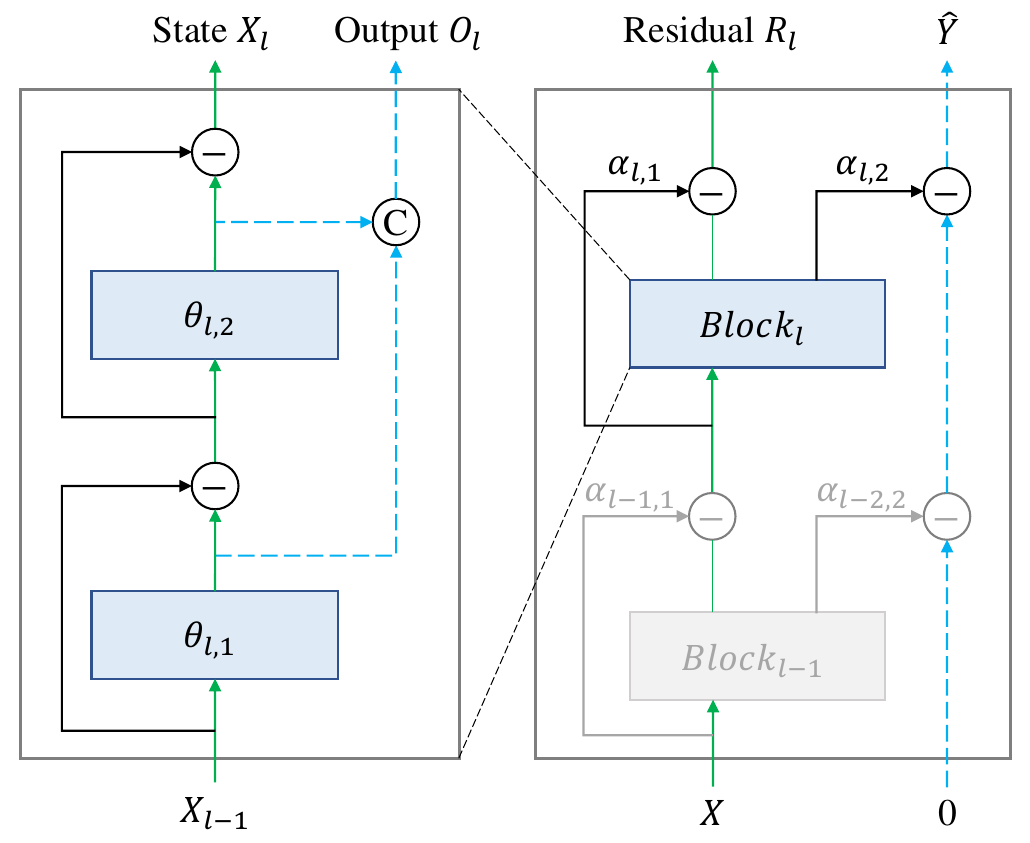}
  \caption{The implementation of DeepBooTS.}
  \label{fig_arch}
\end{figure}

\textbf{Coefficient}: 
Upon being scaled by a coefficient ($\alpha_l$), the dual streams are directed toward the next block or projected into the output space.
Here, we adopt a learnable gating mechanism as the coefficient. The aim is to let each learner regulate the stream-transmission pace autonomously.
Thus, for input and output streams, we have
\begin{align}
  X_{l+1} & = \varphi  (\theta_1(R_{l,2})) \cdot \theta_2(R_{l,2}), \label{eq6} \\
  {O}_{l+1} & = \varphi (\theta_3([\hat{X}_{l,1},\hat{X}_{l,2}])) \cdot \theta_4([\hat{X}_{l,1},\hat{X}_{l,2}]), \label{eq7}
\end{align}
where $\theta_1$ and $\theta_2$ are learnable parameters, $\varphi$ is the sigmoid, and the brackets `[ ]' are a concatenation operation.

\begin{table*}[!ht]
    \centering
      \setlength{\tabcolsep}{1mm}
      \fontsize{9pt}{9pt}\selectfont
        \begin{tabular}{c cc  cc  cc cc  cc cc cc cc  cc}
        \toprule
        Model          & \multicolumn{2}{c}{DeepBooTS*} & \multicolumn{2}{c}{DeepBooTS} & \multicolumn{2}{c}{TimeXer}  & \multicolumn{2}{c}{TimeMixer} & \multicolumn{2}{c}{iTransformer} & \multicolumn{2}{c}{PatchTST} & \multicolumn{2}{c}{Crossformer} & \multicolumn{2}{c}{DLinear} & \multicolumn{2}{c}{FEDformer}\\ \toprule
        Metric                         & MSE       & MAE               & MSE       & MAE       & MSE       & MAE    & MSE              & MAE     & MSE              & MAE       & MSE              & MAE                   & MSE               & MAE              & MSE            & MAE          & MSE            & MAE                       \\ \toprule
        ETT  & {\bf\color{red}0.346}     & {\bf\color{red}0.378}        & \underline{\color{blue}0.362}     & \underline{\color{blue}0.382}      & 0.365  & 0.388   & 0.367 & 0.388  & 0.383            & 0.407      & 0.381  & 0.397            & 0.942             & 0.684                   & 0.559          & 0.515          & 0.437             & 0.449             \\ \toprule
        Traffic  & {\bf\color{red}0.373}     & {\bf\color{red}0.261}          & \underline{\color{blue}0.406}     & \underline{\color{blue}0.270}  & 0.466  & 0.287  & 0.484   & 0.297      & 0.428            & 0.282     & 0.481  & 0.304        & 0.550             & 0.304                  & 0.625          & 0.383          & 0.610             & 0.376           \\ \toprule
        ELC  & {\bf\color{red}0.158}     & {\bf\color{red}0.252}          & \underline{\color{blue}0.166}     & \underline{\color{blue}0.259}     & 0.171  & 0.270   & 0.182  & 0.272  & 0.178            & 0.270        & 0.205  & 0.290      & 0.244             & 0.334                    & 0.212          & 0.300          & 0.214             & 0.327             \\ \toprule
        Weather  & {\bf\color{red}0.227}     & {\bf\color{red}0.266}         & 0.245     & \underline{\color{blue}0.271}   & 0.241  & \underline{\color{blue}0.271}   & \underline{\color{blue}0.240}  & \underline{\color{blue}0.271}     & 0.258            & 0.279    & 0.259  & 0.281               & 0.259             & 0.315           & 0.265          & 0.317          & 0.309             & 0.360            \\ \toprule
        Solar    & {\bf\color{red}0.197}     & {\bf\color{red}0.240}        & 0.227     & \underline{\color{blue}0.250}  & 0.272  & 0.300 & \underline{\color{blue}0.216}  & 0.280     & 0.233            & 0.262    & 0.270  & 0.307                 & 0.641             & 0.639            & 0.330          & 0.401          & 0.291             & 0.381              \\ \toprule
        PEMS  & {\bf\color{red}0.075}     & {\bf\color{red}0.178}     & \underline{\color{blue}0.109}     & \underline{\color{blue}0.218}   & 0.141  & 0.225   & 0.138   & 0.242   & 0.113            & 0.221     & 0.180  & 0.291             & 0.169             & 0.281                  & 0.278          & 0.375          & 0.213             & 0.327               \\ \toprule
        {\bf \color{red}$1^{st}$} or \underline{\color{blue}$2^{st}$} & {\bf\color{red} 30}  & {\bf\color{red}30}   & \underline{\color{blue}15}   & \underline{\color{blue}22}  &  \underline{\color{blue}3}  &  \underline{\color{blue}4}  &  \underline{\color{blue}8}  &  \underline{\color{blue}2}   & \underline{\color{blue}3}   & \underline{\color{blue}2}    & 0  & 0   & \underline{\color{blue}3}    & 0    & \underline{\color{blue}1}  & 0    & 0       & 0   \\  \bottomrule
        \end{tabular}
    \caption{Multivariate TS forecasting results (average). The results for all forecasting lengths are provided in Appendix M.}
    \label{tb2}
  \end{table*}

\begin{table*}[htpb]
  \centering
    \fontsize{9pt}{10pt}\selectfont
    \setlength{\tabcolsep}{1.3mm}
    \begin{tabular}{cc  cc  cc  cc  cc  cc  cc  c}
      \toprule
      Model       & \multicolumn{2}{l}{DeepBooTS} & \multicolumn{2}{l}{Periodformer} & \multicolumn{2}{l}{FEDformer} & \multicolumn{2}{l}{Autoformer} & \multicolumn{2}{l}{Informer} & \multicolumn{2}{l}{LogTrans} & \multicolumn{2}{l}{Reformer} \\
      \hline
      Metric      & MSE              & MAE             & MSE               & MAE               & MSE             & MAE            & MSE             & MAE             & MSE            & MAE            & MSE            & MAE            & MSE            & MAE            \\
      \hline
      ETTh1       & {\bf\color{red}0.072}            & {\bf\color{red}0.206}           & \underline{\color{blue}0.093}             & \underline{\color{blue}0.237}             & 0.111           & 0.257          & 0.105           & 0.252           & 0.199          & 0.377          & 0.345          & 0.513          & 0.624          & 0.600          \\
      ETTh2       & {\bf\color{red}0.185}            & {\bf\color{red}0.337}           & \underline{\color{blue}0.192}             & \underline{\color{blue}0.343}             & 0.206           & 0.350          & 0.218           & 0.364           & 0.243          & 0.400          & 0.252          & 0.408          & 3.472          & 1.283          \\
      ETTm1       & {\bf\color{red}0.052}            & {\bf\color{red}0.172}           & \underline{\color{blue}0.059}             & \underline{\color{blue}0.201}             & 0.069           & 0.202          & 0.081           & 0.221           & 0.281          & 0.441          & 0.231          & 0.382          & 0.523          & 0.536          \\
      ETTm2       & \underline{\color{blue}0.118}            & \underline{\color{blue}0.254}           & {\bf\color{red}0.115}             & {\bf\color{red}0.253 }            & 0.135           & 0.278          & 0.130           & 0.271           & 0.147          & 0.293          & 0.130          & 0.277          & 0.136          & 0.288          \\
      Traffic     & {\bf\color{red}0.132}            & {\bf\color{red}0.212}           & \underline{\color{blue}0.150}             & \underline{\color{blue}0.233}             & 0.177           & 0.270          & 0.261           & 0.365           & 0.309          & 0.388          & 0.341          & 0.417          & 0.375          & 0.434          \\
      Electricity & \underline{\color{blue}0.314}            & \underline{\color{blue}0.401}           & {\bf\color{red}0.298 }            & {\bf\color{red}0.389}             & 0.347           & 0.434          & 0.414           & 0.479           & 0.372          & 0.444          & 0.410          & 0.473          & 0.352          & 0.435          \\
      Weather     & {\bf\color{red}0.0015}           & {\bf\color{red}0.0293}          & \underline{\color{blue}0.0017}            & \underline{\color{blue}0.0317}            & 0.008           & 0.067          & 0.0083          & 0.0700          & 0.0033         & 0.0438         & 0.0059         & 0.0563         & 0.0115         & 0.0785         \\
      Exchange    & \underline{\color{blue}0.429}            & \underline{\color{blue}0.453}           & {\bf\color{red}0.353}             & {\bf\color{red}0.434}             & 0.499           & 0.512          & 0.578           & 0.537           & 1.511          & 1.029          & 1.350          & 0.810          & 1.028          & 0.812          \\ 
      \hline
      { $1^{\text{st}}$ Count} & {\bf\color{red} 24}           & {\bf\color{red}27}           & \underline{\color{blue}14}                   & \underline{\color{blue}12}       & 0               & 1               & 0             & 0               & 0               & 0               & 2             & 0             & 0               & 0             \\
      \bottomrule
    \end{tabular}
\caption{Univariate TS forecasting results (average). The results for all forecasting lengths are provided in Appendix N.}
\label{tb3}
\end{table*}

\section{Experiments}
\label{sec_exp}


We conduct a comprehensive comparison with the 18 latest SOTA methods,
including TimeXer \cite{wang2024timexer}, TimeMixer \cite{wang2024timemixer}, iTransformer \cite{liu2023itransformer}, PatchTST \cite{nie2022time}, Crossformer \cite{zhang2022crossformer}, SCINet \cite{liu2022scinet}, DLinear \cite{zeng2023transformers}, FEDformer \cite{zhou2022fedformer} for multivariate forecasting,  and FreTS \cite{yi2024frequency}, PSLD \cite{liang2024psld}, and FourierGNN \cite{yi2024fouriergnn} for large-scale TS datasets, 
as well as Periodformer \cite{liang2024periodformer}, Autoformer \cite{wu2021autoformer}, Informer \cite{Zhou2021Informer}, LogTrans \cite{li2019enhancing}, Reformer \cite{Kitaev2020Reformer}, N-BEATS \cite{oreshkin2019n} and N-Hits \cite{challu2023nhits} for univariate forecasting.
More experimental settings can be found in Appendix D.

\subsection{Main Experimental Results}

All datasets are adopted for both multivariate (multivariate predict multivariate) and univariate (univariate predicts univariate) tasks. 
The detailed information pertaining to the datasets can be located in Appendix C.
The models used in the experiments are evaluated over a wide range of prediction lengths to compare performance on different future horizons: 96, 192, 336 and 720. 
Please refer to Appendix K for more experiments on the full ETT dataset.

\textbf{Multivariate Results}:
The results for multivariate TS forecasting are outlined in Table \ref{tb2}, with the optimal results highlighted in {\bf \color{red} red} and the second-best results emphasized with \underline{\color{blue} underlined}.
Due to variations in input lengths among different methods, for instance, PatchTST and DLinear employing an input length of 336, while Crossformer and Periodformer search for the input length without surpassing the maximum setting (720 in Crossformer, 144 in Periodformer), we have configured two versions of DeepBooTS: 336 for DeepBooTS* and 96 for DeepBooTS.

As shown in Table \ref{tb2}, DeepBooTS achieves the consistent SOTA performance across all datasets and prediction length configurations. 
Compared with the latest models acknowledged for their exceptional average performance, the proposed DeepBooTS demonstrates an average performance increase of {\bf 15.8\%}, achieving a substantial performance improvement.
Obviously, it achieves advanced performance, averaging {\bf 30} items across six datasets, with an improvement in average performance on each dataset. 
These experimental results confirm that the proposed DeepBooTS demonstrates superior prediction performance across different datasets with varying horizons. 
For a more detailed analysis, please refer to Appendix O.

\textbf{Univariate Results}: 
The average results for univariate TS forecasting are shown in Table \ref{tb3}. 
It is evident that the proposed DeepBooTS continues to maintain a SOTA performance across various prediction length settings compared to the benchmarks.
In summary, compared with the hyperparameter-searched Periodformer, 
DeepBooTS yields an average {\bf 4.8\%} reduction across five datasets, and it achieves an average of {\bf 26} best terms.
For example, under the input-96-predict-96 setting, DeepBooTS yields a reduction of {\bf 11.2\%} (0.143$\rightarrow$0.127) in MSE for Traffic. 
Obviously, the experimental results again verify the superiority of DeepBooTS on univariate TS forecasting tasks.

\subsection{Evaluation on Monash TS Datasets}

Further, we evaluate the proposed method on 7 Monash TS datasets \cite{godahewa2021monash} (e.g., NN5, M4 and Sunspot, etc.) and 7 diverse metrics (e.g., MAPE, sMAPE, MASE and Quantile, etc.) to systematically evaluate our model. All experiments are compared under the same input length (e.g., $I$=96) and output lengths (e.g., $O$=\{96, 192, 336 and 720\}). As shown in Table \ref{tb_metric_1}, the proposed DeepBooTS emerged as the frontrunner, achieving a score of \textbf{41 out of 54}. Please refer to Appendices Q and R for details about the definition and experimental settings, respectively.

\begin{table*}[!ht]
    \centering
    \fontsize{9pt}{10pt}\selectfont
    \setlength{\tabcolsep}{0.4mm}
    \begin{threeparttable}
    \begin{tabular}{c|cccccc|cccccc|ccccccc}
      \toprule
      Mulvariate           & \multicolumn{6}{c|}{ILI}                                                        & \multicolumn{6}{c|}{Oik\_Weather}                                                                                               & \multicolumn{6}{c}{NN5}                                                                                \\ \toprule
    Metric          & MSE   & MAE   & RMSP                           & sMAPE & MASE  & Q75   & MSE                           & MAE   & RMSP                       & sMAPE & MASE                          & Q75   & MSE                           & MAE   & RMSP                          & sMAPE & MASE  & Q75    \\ \toprule
    DeepBooTS  & {\bf\color{red}2.016} & {\bf\color{red}0.86}  & {\bf\color{red}0.367} & {\bf\color{red}0.650}  & {\bf\color{red}0.574} & {\bf\color{red}0.955} & \underline{\color{blue}0.689} & {\bf\color{red}0.618} & {\bf\color{red}0.773}  & {\bf\color{red}1.082} & \underline{\color{blue}0.888} & {\bf\color{red}0.628} & {\bf\color{red}0.719} & {\bf\color{red}0.578} & {\bf\color{red}0.541}  & {\bf\color{red}0.861} & {\bf\color{red}0.530}  & {\bf\color{red}0.559} \\ 
    iTransformer & \underline{\color{blue}2.262} & \underline{\color{blue}0.958} & \underline{\color{blue}0.415}                    & \underline{\color{blue}0.703} & \underline{\color{blue}0.648} & \underline{\color{blue}1.107} & 0.720                          & 0.632 & \underline{\color{blue}0.781} & \underline{\color{blue}1.090}  & 0.935                         & \underline{\color{blue}0.636} & \underline{\color{blue}0.723}               & \underline{\color{blue}0.583} & \underline{\color{blue}0.554}   & \underline{\color{blue}0.877} & \underline{\color{blue}0.542} & \underline{\color{blue}0.565} \\
    DLinear      & 2.915 & 1.188 & 0.574                         & 0.865 & 0.791 & 1.524 & 0.701                         & 0.637 & 0.805             & 1.185 & 0.943                         & 0.641 & 1.448                         & 0.929 & 0.969            & 1.304 & 0.835 & 0.978 \\
    Autoformer   & 3.187 & 1.224 & 0.596                         & 0.909 & 0.740  & 1.349 & 0.853                         & 0.712 & 0.883                      & 1.201 & 0.992                         & 0.704 & 0.851                         & 0.659 & 0.636  & 0.975 & 0.603 & 0.637 \\
    Informer     & 5.208 & 1.576 & 0.792                         & 1.183 & 0.907 & 2.304 & {\bf\color{red}0.67}                 & \underline{\color{blue}0.625} & 0.795                      & 1.126 & {\bf\color{red}0.818}                & 0.640  & 0.967                         & 0.727 & 0.739      & 1.081 & 0.662 & 0.722  \\ \bottomrule
    
    Univariate        & \multicolumn{6}{c|}{M4 Hourly}                                                                                               & \multicolumn{6}{c|}{Us\_births}                                                                                       & \multicolumn{6}{c}{Saugeenday}                                                                                       \\ \toprule
    DeepBooTS  & \textbf{\color{red}0.223} & \textbf{\color{red}0.287} & \textbf{\color{red}0.230} & \textbf{\color{red}0.509} & \textbf{\color{red}0.276} & \textbf{\color{red}0.311} & \textbf{\color{red}0.364} & \textbf{\color{red}0.431} & \textbf{\color{red}0.248} & \underline{\color{blue}0.445}          & \textbf{\color{red}0.363} & \textbf{\color{red}0.361} & 1.194          & 0.544          & \textbf{\color{red}0.808} & \textbf{\color{red}1.053} & 1.232          & 0.674        \\
      iTransformer &  \underline{\color{blue}0.310}          & \underline{\color{blue}0.390}          & \underline{\color{blue}0.376}             & \underline{\color{blue}0.656}          & \underline{\color{blue}0.366}         & \underline{\color{blue}0.410}          & \underline{\color{blue}0.378}          & \underline{\color{blue}0.439}          & \underline{\color{blue}0.251}        & \textbf{\color{red}0.438} & \underline{\color{blue}0.365}          & \underline{\color{blue}0.388}          & 1.206          & 0.552          & \underline{\color{blue}0.832}          & \underline{\color{blue}1.066}          & 1.258          & 0.687   \\
      N-Beats      & 0.337          & 0.423          & 0.392          & 0.700          & 0.390          & 0.493          & 0.670          & 0.638          & 0.409        & 0.625          & 0.546          & 0.834          & \underline{\color{blue}1.028}          & \underline{\color{blue}0.543}          & 0.962       & 1.387          & \underline{\color{blue}0.943}          & \underline{\color{blue}0.572}     \\
      N-Hits       & 0.346          & 0.418          & 0.402          & 0.687          & 0.396          & 0.442        & 0.538          & 0.578          & 0.364          & 0.540          & 0.464          & 0.802          & \textbf{\color{red}0.982} & \textbf{\color{red}0.532} & 0.865       & 1.189          & \textbf{\color{red}0.858} & \textbf{\color{red}0.562}     \\
      Autoformer   & 0.635          & 0.611          & 0.614          & 0.881          & 0.518          & 0.741          & 0.901          & 0.757          & 0.447        & 0.789           & 0.737          & 0.788          & 1.252          & 0.640          & 1.058               & 1.370          & 1.057          & 0.674          \\ \bottomrule
  \end{tabular}
    \begin{tablenotes}
      \item[*] All results are averaged across all prediction lengths. The results for all prediction lengths and the experimental settings are provided in Appendix R. The definitions of all metrics are provided in Appendix Q.
    \end{tablenotes}
    \end{threeparttable}
    \caption{Mulvariate and univariate forecasting results with diverse metrics on Monash TS datasets.}
    \label{tb_metric_1}
\end{table*}


\begin{figure*}[t]
  \centerline{\includegraphics[width=\textwidth]{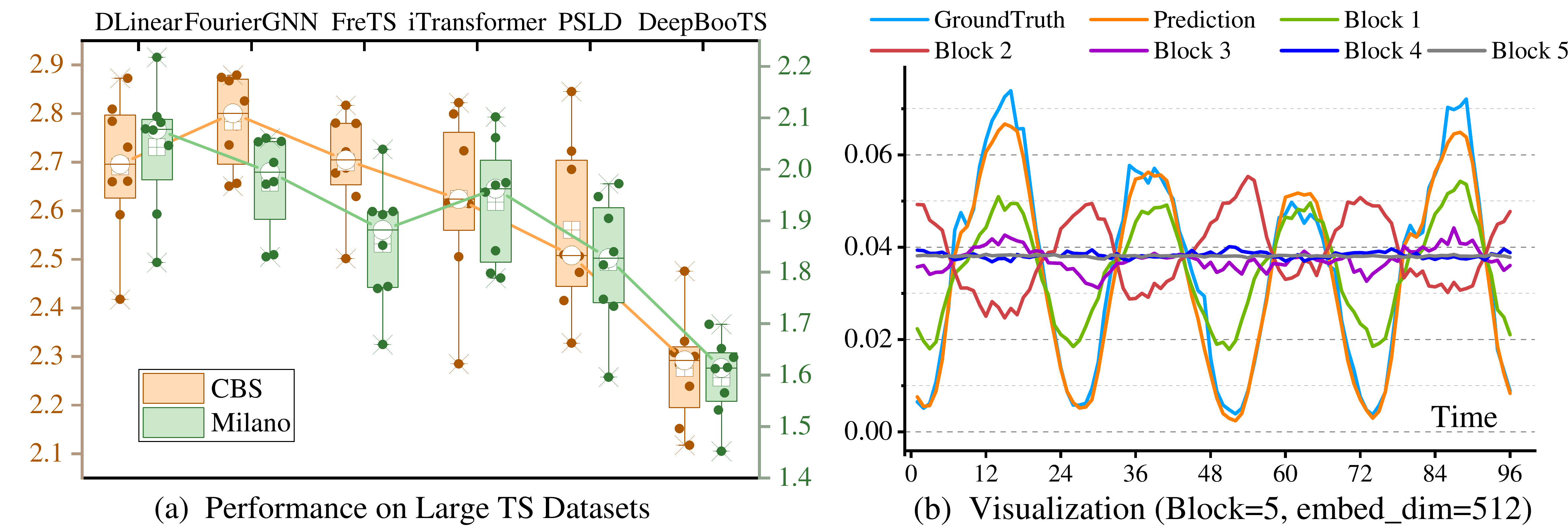}}
  \caption{(a) Comparisons on large-scale TS datasets. (b) Visualization depicting the output of each block in DeepBooTS.} %
  \label{fig_visblock}
\end{figure*}

\subsection{Performance on Large-Scale Datasets}
\label{ssec_large_ts}

The performance comparisons for large-scale TS datasets are summarized in Fig. \ref{fig_visblock}(a), including the CBS dataset with 4,454 nodes (17GB) and the Milano dataset with 10,000 nodes (19GB). 
On these datasets, it is not feasible to place all nodes on a single GPU for long-term TS forecasting. Therefore, we adopted the random partitioning strategy \cite{liang2024act} for our experiments. Please refer to Appendices C.3. 
 and D for  details.
Compared to the latest advanced PSLD, the proposed DeepBooTS yields an overall MSE reduction of {\bf 8.9\%} and {\bf 6.2\%} on the CBS and Milano datasets, respectively.
In addition, we are pleased to report that our method has been successfully applied to real-world scenarios, e.g., energy and power, showing substantial practical value.
For more experiments, please refer to Appendix F.

\begin{figure*}[t]
  \centerline{\includegraphics[width=\textwidth]{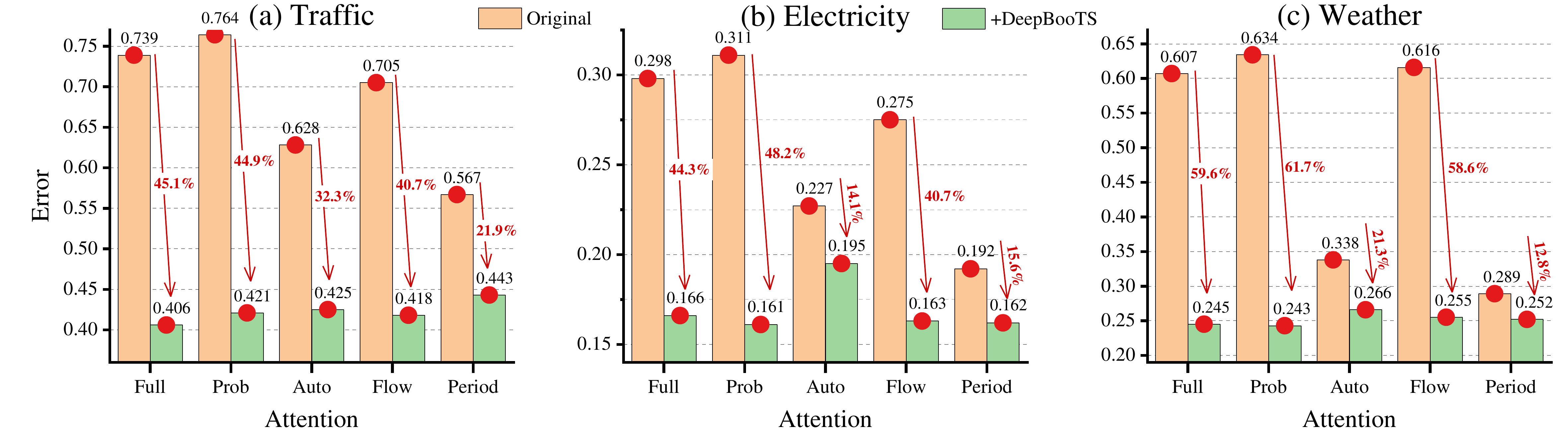}}
  \caption{Ablation studies of DeepBooTS using various Attention. All results are averaged across all prediction lengths. The tick labels of the X-axis are the abbreviation of Attention types. The detailed setup and results are provided in Appendix Q.}
  \label{fig_attns}
\end{figure*}

\subsection{Generality}

To investigate DeepBooTS's generality as a universal architecture, we substituted its original Attention with other novel Attention mechanisms to observe the resulting changes in model performance.
As shown in Fig. \ref{fig_attns}, after harnessing the newly invented Attention within DeepBooTS, its performance exhibited considerable variation. 
E.g., the average MSE of Prob-Attention \cite{Zhou2021Informer} on the Electricity and Weather datasets witnessed a reduction of {\bf 48.2\%} (0.311$\rightarrow$0.161) and {\bf 61.7\%} (0.634$\rightarrow$0.243), respectively, surpassing Full-Attention and achieving new SOTA performance.
Furthermore, Period-Attention \cite{liang2023does}, Auto-Correlation \cite{wu2021autoformer}, and Flow-Attention \cite{wu2022flowformer} all exhibit commendable performance on the aforementioned datasets, with their average performance surpassing that of Full-Attention.
The conducted experiments suggest that DeepBooTS can serve as a versatile architecture, amenable to the integration of novel modules, thereby facilitating the enhancement of performance in TS forecasting.

\subsection{Effectiveness}
\label{ssec_ablation}
 
\begin{figure*}[t]
  \centerline{\includegraphics[width=\textwidth]{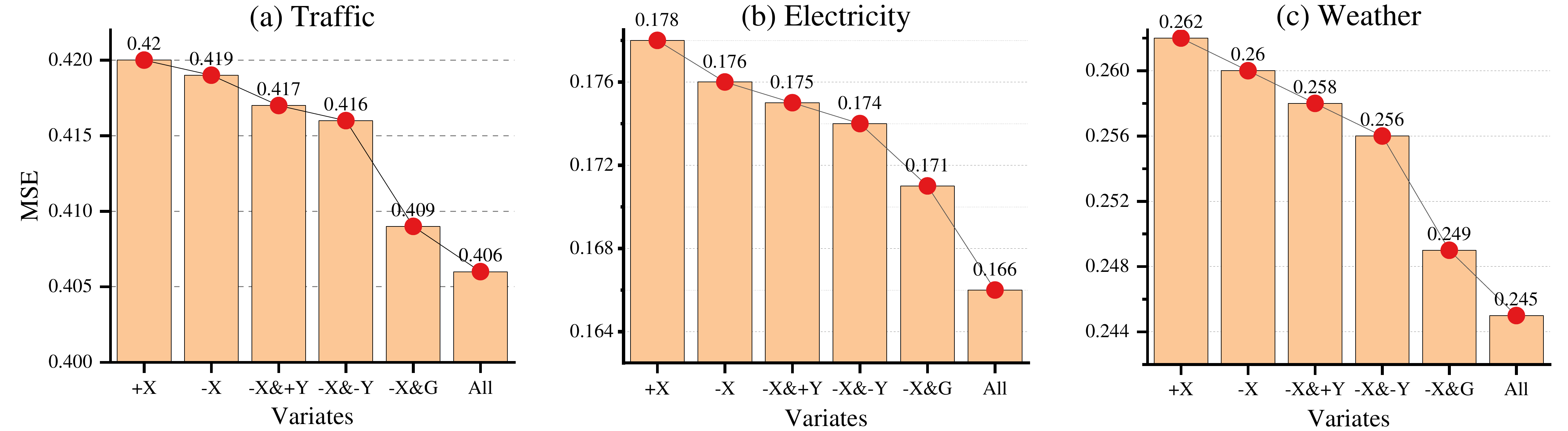}}
  \caption{Ablation studies on various components of DeepBooTS. All results are averaged across all prediction lengths. The variables $X$ and $Y$ represent the input and output streams, while the signs `+' and `-' denote the addition or subtraction operations used when the streams' aggregation. The letter `G' denotes adding a learnable gating (Coefficient) to the output of each block.}
  \label{fig_variates}
\end{figure*}

To validate the effectiveness of DeepBooTS components, we conduct comprehensive ablation studies encompassing both component replacement and component removal experiments, as shown in Fig. \ref{fig_variates}.
We utilize signs `+' and `-' to denote the utilization of addition or subtraction operations during the aggregation process of the input or output streams. 
In cases involving only input streams, it becomes evident that the model's average performance is superior when employing subtraction (-X) compared to when employing addition (+X).
E.g., on the Weather dataset, forecast error is reduced by {\bf 2.3\%} (0.262$\rightarrow$0.256).
Moreover, with the introduction of a high-speed output stream to the model, shifting the aggregation method of the output stream from addition (+Y) to subtraction (-Y) is poised to further enhance the model's performance.
Afterward, incorporating gating mechanisms (G) into the model holds the potential to improve predictive performance again, e.g., forecast error is reduced by {\bf 4.9\%} (0.262$\rightarrow$0.249) on the Weather dataset.
In summary, integrating the advantages of the aforementioned components has the potential to significantly boost the model's performance across the board.

\subsection{Interpretability}
\label{ssec_intpbt}

The intrinsic characteristic of DeepBooTS lies in the alignment of the output from each block with the final output. This alignment, in turn, expedites the interpretability of the model's learning process.
As depicted in Fig. \ref{fig_visblock}(b), the output of each block in DeepBooTS is visualized.
It becomes evident that each block discerns and assimilates meaningful patterns.
Specifically, when the embedding dimension is low, each block must learn salient patterns. Further, when the depth of the model is increased, it becomes evident that the amplitude of the shallow block decreases, and numerous components are transferred to the deep block.
For a more detailed analysis, please refer to Appendix P.

\subsection{Reduce Variance and Go Deeper}

\begin{figure*}[!ht]
      \begin{center}
        \includegraphics[width=0.7\textwidth]{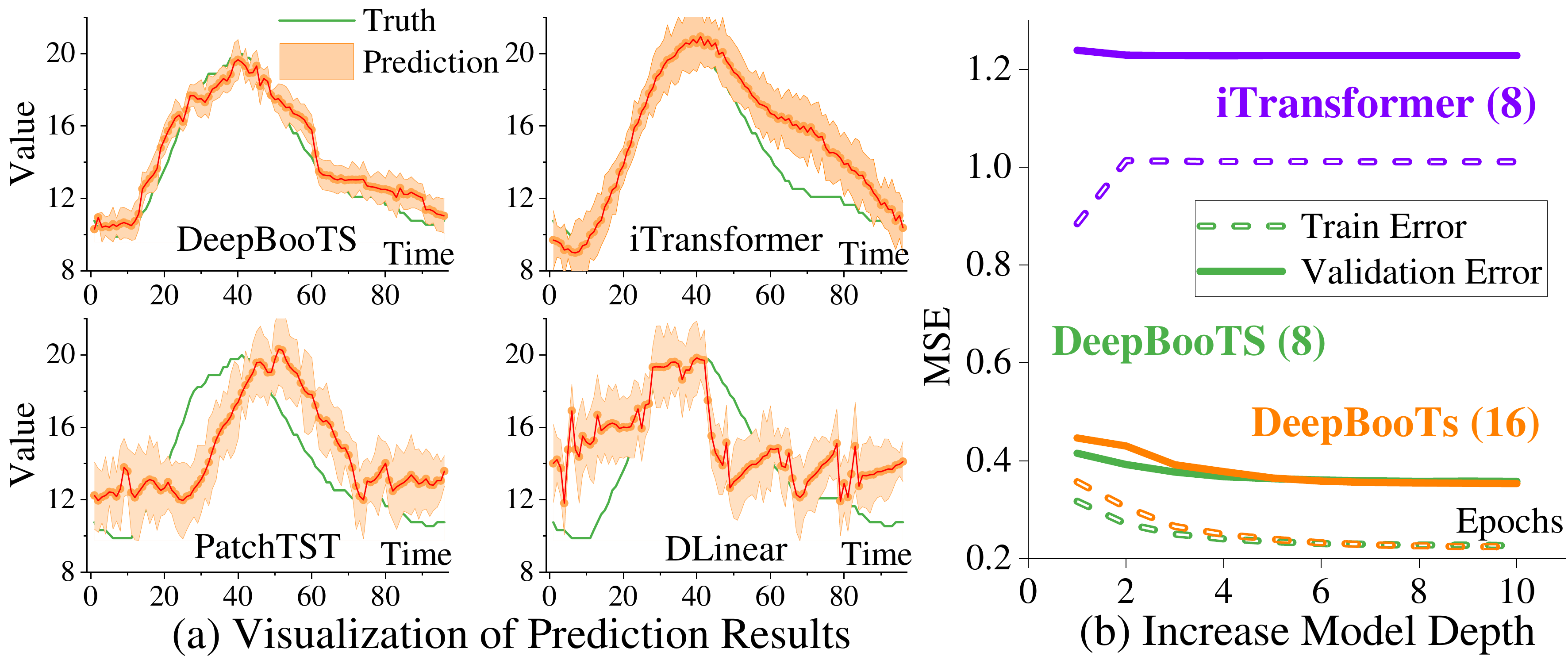}
      \end{center}
      \caption{Comparisons of model's variance and depth.}
      \label{fig_godeeper}
      \vspace{-1em}
\end{figure*}

By adjusting model hyperparameters and injecting noise, we generated a diverse set of predictions from latest advanced models. Then, these predictions are statistically analyzed to compute the prediction mean and variance. As shown in Fig. \ref{fig_godeeper}(a), DeepBooTS achieves superior performance and the smallest prediction variance, while other models exhibit weaker alignment with ground truth and higher variance. This validates the effectiveness of our theory and model design. Meanwhile, these results also confirm that DeepBooTS is less sensitive to hyperparameters, with details presented in Appendix I.
Furthermore, given DeepBooTS's robustness against high-variance, it can be designed with considerable depth.
Fig. \ref{fig_godeeper}(b) shows the scenarios when models go deeper.
Serious overfitting happens when the number of iTransformer blocks is increased from 4 to 8.
However, even with the DeepBooTS blocks deepened to 16, it continues to exhibit excellent performance.

\section{Conclusion}
\label{sec_con}

In this paper, we theoretically analyze the impact of concept drift in TS forecasting from a bias-variance perspective. We demonstrate that the deep ensemble is capable of reducing this issue. 
Based on this, we propose a novel dual-stream residual-decreasing Boosting ensemble approach, termed DeepBooTS, and demonstrate its efficacy in variance reduction.
DeepBooTS facilitates the learning-driven and progressive decomposition of both the input and output streams, thereby empowering the model resilience against concept drift.
Extensive experiments show that DeepBooTS achieves SOTA performance, and it is less sensitive to hyperparameters and can be designed very deep with enhanced interpretability.
Furthermore, DeepBooTS can serve as a versatile framework to enhance other models' performance.

\section*{Acknowledgments}
This work was supported by Project of Key R\&D Program of Shandong Province, China (2025CXGC010107, 2025CXPT095), Taishan Scholars Program: NO.tspd20240814, the Pilot Project for the Integration of Science, Education and Industry of Qilu University of Technology (Shandong Academy of Sciences)(2025ZDZX01), National Key R\&D Program of China (2024YFB3312302), and the Qilu Institute of Technology for Talent Project (QIT24TP027), and the Young Talent of Lifting engineering for Science and Technology in Shandong (SDAST2025QTB008).

\bibliography{aaai2026}

\clearpage
\onecolumn
\appendix

\section{Related Work}
\label{sec_rework}

\subsection{Classical Models for TS Forecasting}

TS forecasting is a classic research field where numerous methods have been invented to utilize historical series to predict future missing values. 
Early classical methods \cite{piccolo1990distance,gardner1985exponential} are widely applied because of their well-defined theoretical guarantee and interpretability. For example, ARIMA \cite{piccolo1990distance} initially transforms a non-stationary TS into a stationary one via differencing, and subsequently approximates it using a linear model with several parameters. 
Exponential smoothing \cite{gardner1985exponential} predicts outcomes at future horizons by computing a weighted average across historical data. In addition, some regression-based methods, e.g., random forest regression (RFR) \cite{liaw2002classification} and support vector regression (SVR) \cite{castro2009online}, etc., are also applied to TS  forecasting. 
These methods are straightforward and have fewer parameters to tune, making them a reliable workhorse for TS forecasting. 
However, their shortcoming is insufficient data fitting ability, especially for high-dimensional series, resulting in limited performance.

\subsection{Deep Models for TS Forecasting}

The advancement of deep learning has greatly boosted the progress of TS forecasting. 
Specifically, convolutional neural networks (CNNs) \cite{lecun1998gradient} and recurrent neural networks (RNNs) \cite{connor1994recurrent} have been adopted by many works to model nonlinear dependencies of TS, 
e.g., LSTNet \cite{lai2018modeling} improve CNNs by adding recursive skip connections to capture long- and short-term temporal patterns;
DeepAR \cite{salinas2020deepar} predicts the probability distribution by combining autoregressive methods and RNNs.
Several works have improved the series aggregation forms of Attention mechanism, such as operations of exponential intervals adopted in LogTrans \cite{li2019enhancing}, ProbSparse activations in Informer \cite{Zhou2021Informer}, frequency sampling in FEDformer \cite{zhou2022fedformer} and iterative refinement in Scaleformer \cite{shabani2022scaleformer}. 
Besides, GNNs and Temporal convolutional networks (TCNs) \cite{lea2016temporal} have been utilized in some methods \cite{wu2019graph,li2023dynamic,liu2022scinet,wu2022timesnet} for TS forecasting on graph data.
The aforementioned methods solely concentrate on the forms of aggregating input series, overlooking the challenges posed by the concept drift problem.

Since TS exhibit a variety of patterns, it is meaningful and beneficial to decompose them into several components, each representing an underlying category of patterns that evolving over time \cite{1976_timeseries}.
Several methods, e.g., STL \cite{cleveland1990stl}, Prophet \cite{taylor2018forecasting} and N-BEATS \cite{oreshkin2019n}, commonly utilize decomposition as a preprocessing phase on historical series.
There are also some methods, e.g., Autoformer \cite{wu2021autoformer}, FEDformer \cite{zhou2022fedformer}, Non-stationary Transformers \cite{liu2022non}, Minusformer \cite{liang2024minusformer} and DistPred \cite{liang2025distpred}, that harness decomposition into the Attention module. 
The aforementioned methods attempt to apply decomposition to input series to enhance predictability, reduce computational complexity, or ameliorate the adverse effects of non-stationarity.
Nevertheless, these prevalent methods are susceptible to significant concept drift when applied to non-stationary TS.

Furthermore, there are four themes that use deep learning to predict time series: (1) smarter transformers \cite{NIPS2017_Transformer}, such as PatchTST \cite{nie2022time}, iTransformer \cite{liu2023itransformer}  BasisFormer \cite{ni2023basisformer}, and TimeXer \cite{wang2024timexer} which restructure attention or add learnable bases to extend context length, cut computation and boost accuracy; (2) competitive non-transformer backbones, including N-HiTS (hierarchical MLP) \cite{challu2023nhits}, DLinear \cite{zeng2023transformers}, PGN \cite{jia2024pgn} and state-space models like TSMamba \cite{ma2024mamba}, TimeMachine \cite{ahamed2024timemachine} and FLDMamba \cite{zhang2025fldmamba}, which deliver linear-time inference and rival or surpass transformers on long horizons; (3) foundation-model initiatives, TimeGPT \cite{garza2023timegpt}, OneFitAll \cite{zhou2023one}, TimeLLM \cite{jin2024timellm}, UniTime \cite{liu2024unitime} and DAM \cite{darlowdam2024DAM} that pre-train on massive heterogeneous corpora and achieve impressive zero-shot or few-shot performance across domains; and (4) training and interpretability advances, such as frequency-adaptive normalization (FAN) \cite{ye2024frequency}, e.g., FreTS \cite{yi2024frequency}, FilterNet \cite{yi2024filternet}, and decomposition-aware architectures, e.g.,  which tackle non-stationarity, quantify uncertainty and make forecasts more transparent. 




\subsection{Concept Drift in TS forecasting}

Concept drift in TS forecasting has long been managed with adaptive retraining and change detection techniques. 

\textbf{Classical methods} include using sliding windows of recent data to regularly update models, thereby ``forgetting'' outdated patterns. Traditional statistical forecasters like ARIMA \cite{piccolo1990distance} can be incrementally retrained on rolling windows, e.g., an incremental ARIMA scheme fits a new model whenever the current model's validation error exceeds a threshold, signaling a drift. Another staple is change-point detection, where statistical tests (e.g. the Page-Hinkley test) monitor the data or error distribution and flag abrupt shifts once differences exceed a predefined threshold. For instance, the Error Intersection Approach (EIA) \cite{baier2020handling} employs two models, one ``stable'' (more complex) and one ``reactive'' (simpler), and compares their performance.
Furthermore, the Drift Detection Method (DDM) \cite{korycki2023adversarial} monitors the classifier's error rate distribution and raises an alarm when the error's increase is statistically significant. 
These classical and statistical methods form the foundation of drift handling in forecasting, using either explicit retraining policies or specialized detectors to maintain model validity over time.

\textbf{Online learning} strategies embed concept drift adaptation within the forecasting models themselves. One example is FSNet \cite{pham2023learning}, which leverages complementary learning systems theory to pair a slow-learning base forecaster with fast-adapting components. Another line of work is OneNet \cite{wen2023onenet}, an online ensembling approach that dynamically combines two neural models: one specializes in capturing temporal dependencies within each series, and the other focuses on cross-series (covariate) relationships. 
Each of these deep learning techniques illustrates how integrating drift-awareness (through dual-model architectures, ensembling, or proactive adjustment) can improve TS forecasting performance in online.

\section{Theoretical Rationale}
\label{app_overfit_theory}

\subsection{Proof of Equation \ref{eq_est_err}}
\label{seca_proof_eq1}

In this subsection, we provide a theoretical analysis of how DeepBooTS alleviates concept drift.
In practice, observations $Y$ often contain additive noise $\varepsilon$: 
$Y  = \mathcal{Y} + \varepsilon $, 
where $\varepsilon \sim \mathcal{N}(0, \sigma )$. 
Then, the estimation error (MSE) of the final model is:
\begin{align}
  \mathbb{E}[(\hat{Y} - Y )^2] & = \mathbb{E}[( \hat{Y} - \mathcal{Y} - \varepsilon )^2] \nonumber \\ 
   & = \mathbb{E}[(\hat{Y} - \mathcal{Y} )^2] + \mathbb{E}(\varepsilon^2) 
   - 2\mathbb{E}(\varepsilon(\hat{Y} - \mathcal{Y} )) \nonumber \\
   & = \mathbb{E}\left( (\hat{Y} - \mathbb{E}(\hat{Y}) + \mathbb{E}(\hat{Y}) - \mathcal{Y} )^2 \right) + \sigma^2   - 2\mathbb{E}(\varepsilon(\hat{Y} - \mathcal{Y} )) \nonumber \\
   & = \mathbb{E}\left( (\hat{Y} - \mathbb{E}(\hat{Y}))^2  +(\mathbb{E}(\hat{Y}) - \mathcal{Y}  )^2 + 2(\hat{Y} - \mathbb{E}(\hat{Y}) ) (\mathbb{E}(\hat{Y}) - \mathcal{Y} )  \right)  + \sigma^2   - 2\mathbb{E}(\varepsilon(\hat{Y} - \mathcal{Y} )) \nonumber \\
   & = \mathbb{E}\left( (\hat{Y} - \mathbb{E}(\hat{Y}))^2 \right)  +(\mathbb{E}(\hat{Y}) - \mathcal{Y}  )^2 
   + \sigma^2   - 2\mathbb{E}(\varepsilon(\hat{Y} - \mathcal{Y} )) \nonumber \\
   & = \text{Var}(\hat{Y}) + (\text{Bias}(\hat{Y}))^2 + \sigma^2   - 2\mathbb{E}(\varepsilon(\hat{Y} - \mathcal{Y} )). \label{eq_a3_0}
\end{align}
Based on this, Equation \ref{eq_a3_0} can be reformulated as
\begin{align}
  \underbrace{\text{Var}(\hat{Y}) + (\text{Bias}(\hat{Y}))^2 + \sigma^2}_{\text{Test Error}}  =  \underbrace{\mathbb{E}[(\hat{Y} - Y )^2] + 2\mathbb{E}(\varepsilon(\hat{Y} - \mathcal{Y} ))}_{\text{Training Error}} .
  \label{eq_a3}
\end{align}

Equation \ref{eq_a3} demonstrates a triangular relationship between the loss, bias, and variance of the model when the noise term is not taken into account, as shown in Fig. \ref{fig_boost}.
For modern complex deep learning models, their biases are typically very low \cite{goodfellow2016deep,zhang2021understanding}.
Here, we focus on how to reduce their variance to alleviate concept drift.
In Equation \ref{eq_a2}, the final model is a weighted sum of $L$ simple models, with the weights determined by the performance of the previous model.

\subsection{Proof of Theorem \ref{th1}}
\label{seca_proof_th1}

Before we prove Theorem \ref{th1}, we first give several definitions and prove a lemma.

{\bf Mean Squared Error (MSE)}, MSE is denoted as
\begin{equation}
  \mathcal{L}(Y,\hat{Y}) = \frac{1}{2}(Y-\hat{Y})^2.
\end{equation}

{\bf Dual Variable}: For differentiable $g$, the dual variable is defined as $\hat{Y}^* = \nabla g(\hat{Y})$, and we have $\hat{Y}=(\hat{Y}^*)^*$.

{\bf Convex Conjugate}: The convex conjugate of a function $g(X)$ is defined as
\begin{equation}
  g^*(Z) = \mathop{\text{sup}} \limits_{Z} \{ <Z,\hat{Y}> - g(\hat{Y}) \}.
\end{equation}

{\bf Central Prediction}: The central prediction is the unique minimizer, i.e.,
\begin{equation}
  \mathbb{C}_Z = \mathop{\text{argmin}} \limits_{Z} \mathbb{E}[\mathcal{L}(\hat{Y}, Z)], 
  \label{eqa_cpred}
\end{equation}

According to the above definition, we have
\begin{lemma}\label{lemma1}
  Let $Y_1, \cdots, Y_N$ be $N$ i.i.d random variables drawn from some unknown distribution, for average ensemble $\bar{Y}=\frac{1}{N}\sum_t \hat{Y}^t$ with MSE metric, we have $\mathbb{C}_{\bar{Y}} = \mathbb{C}_{\hat{Y}}$.
\end{lemma}

\begin{proof}
  First, the conjugate function of MSE is
  \begin{align}
  g^*(Z)  & = \mathop{\text{sup}} \limits_{Z} \left( ZX - \mathcal{L}(Y,\hat{Y}) \right) \notag \\
  & = \mathop{\text{sup}} \limits_{Z} \left( ZX -\frac{1}{2} (Y-\hat{Y})^2 \right) \notag \\
  & = \mathop{\text{sup}} \limits_{Z} \left( ZX - \frac{1}{2} Y^2 + Y\hat{Y} - \frac{1}{2}\hat{Y}^2  \right) \notag \\
  & = \mathop{\text{sup}} \limits_{Z} \left( - \frac{1}{2} Y^2 + (Z+\hat{Y})Y - \frac{1}{2} \hat{Y}^2  \right) 
  \label{eqa_th1_2}
  \end{align}
The Eq. \ref{eqa_th1_2} is a quadratic function in $Y$, which opens downward (coefficient of $Y^2$ is negative). The maximum occurs at the vertex of the parabola, given by $Y = Z+\hat{Y}$. Thus, Substitute it back into the quadratic expression in Eq. \ref{eqa_th1_2}, we have
\begin{align}
  g^*(Z)  & = -\frac{1}{2} ( Z+\hat{Y})^2 + (Z+2\hat{Y})(Z+2\hat{Y}) -\frac{1}{2}\hat{Y}^2 \notag \\
  & = \frac{1}{2} ( (Z+\hat{Y})^2 -  \hat{Y}^2 ) \notag \\
  & = \frac{1}{2} Z^2 -  \hat{Y}Z
\end{align}

On the one hand, for $\mathbb{C}_{\hat{Y}}$, we have
\begin{align}
  \mathbb{C}_{\hat{Y}} & = \mathop{\text{argmin}} \limits_{Z} \mathbb{E}(\hat{Y} - Z)^2 \notag \\
  & = \mathop{\text{argmin}} \limits_{Z} \left[ Z^2 -2Z \mathbb{E}[\hat{Y}] + \mathbb{E}[\hat{Y}^2]  \right]
\end{align}
This is a quadratic function of $\hat{Y}$, and its minimum occurs at the vertex, we have
\begin{align}
  Z = \mathbb{E}[\hat{Y}], \qquad \text{and} \qquad \mathbb{C}_{\hat{Y}} = \mathbb{E}[\hat{Y}].
  \label{eqa_exp_z}
\end{align}

On the other hand, Substituting $\bar{Y}=\frac{1}{N}\sum_t \hat{Y}^t $ into Eq. \ref{eqa_cpred}, we compute $\mathbb{E}[\bar{Y}]$ using the same derivation in Eq. \ref{eqa_exp_z}, we have $\mathbb{C}_{\bar{Y}} =\mathbb{E}[\bar{Y}] = \frac{1}{N}\sum_t \hat{Y}^t = \mathbb{E}[\hat{Y}] $.
Therefore,  $\mathbb{C}_{\bar{Y}} = \mathbb{C}_{\hat{Y}}$ is held.
\end{proof}

Next, we prove Theorem \ref{th1}.
Theorem \ref{th1} states that:
Let $\hat{Y}_1, \cdots, \hat{Y}_N$ be $N$ i.i.d random variables drawn from some unknown distribution, and the ensemble method is $\bar{Y}=\frac{1}{N}\sum_t \hat{Y}^t)$. Using MSE as the metric, we have, 
\begin{align}
  & \text{Bias}(\bar{Y}) = \text{Bias}(\hat{Y}), \label{eqa_th0_a} \\
  & \text{Var}(\bar{Y}) \le \text{Var}(\hat{Y}).  \label{eqa_th0_b}
\end{align}

\begin{proof}
  According Lemma \ref{lemma1}, we have $\mathbb{C}_{\hat{Y}} =\mathbb{E}[\hat{Y}] = \mathbb{C}_{\bar{Y}} =\mathbb{E}[\bar{Y}]$. Therefore, $\text{Bias}(\bar{Y})=\text{Bias}(\hat{Y})$.
  Then, for the variance item, we have
  \begin{align}
    \text{Var}(\bar{Y}) &  = \mathbb{E}[\mathcal{L}(\mathbb{C}_{\bar{Y}}, \bar{Y})] = \mathbb{E}[\mathcal{L}(\bar{Y}^*, \mathbb{C}_{\hat{Y}}^* )]  \qquad (\text{by Lemma \ref{lemma1}}) \notag \\ 
    & = \mathbb{E}[\mathcal{L}(\frac{1}{N}\sum_i\hat{Y_i}^* , \mathbb{C}_{\hat{Y}}^* )]  \notag \\ 
    & \le \frac{1}{N}\sum_i \mathbb{E}  [\mathcal{L}(\hat{Y}_i^*, \mathbb{C}_{\hat{Y}}^* )] \qquad (\text{by convexity of} \ \mathcal{L})  \notag \\
    & \le \frac{1}{N}\sum_i \mathbb{E}  [\mathcal{L}(\mathbb{C}_{\hat{Y}}, \hat{Y}_i )] \notag \\
    & \le \mathbb{E}  [\mathcal{L}(\mathbb{C}_{\hat{Y}}, \hat{Y}_i )] = \text{Var}(\hat{Y})
  \end{align}
  Thus, the theorem is proved.
\end{proof}

\subsection{Proof of Theorem \ref{th1_5}}
\label{seca_proof_th1_5}

\begin{proof}

Since the model normalizes (RevIN) by the instance's mean and standard deviation (std) before prediction, then rescales the prediction by the instance's original mean and std. Formally, if $Z = (X - \mu_X)/\sigma_X$ is the normalized input (with $\mu_X, \sigma_X$ the instance mean and std), and $h(Z)$ is the model's prediction in normalized space, then the final prediction is:
\begin{equation}
  \hat{Y} = \mu_X^Y + \sigma_X^Y \, h(Z),
\end{equation}
where $\mu_X^Y, \sigma_X^Y$ are the output rescaling. If an input's deviation from the mean is scaled by factor $c$, the model's output deviations scale by approximately the same factor $c$ due to RevIN. Thus the prediction variance changes with input variance.

\noindent
\textbf{Single-Model Bias and Variance Under Covariate Shift}

Since the true function $f_0(X)$ is unchanged and models are not updated after drift, the model was approximately unbiased for $f_0$ on $P_0$, it remains so on $P_t$. Thus, for each model $\hat{f}$, we assume $\text{Bias}(X) \approx 0$ or small for all $X$. Therefore, we focus on variance differences, noting $E_{P_t}[\text{Bias}(X)^2] \approx E_{P_0}[\text{Bias}(X)^2]$.
Now let's observe the change of variance. The predictive variance of model $\hat{f}$ changes  under $P_t$. Consider the variance of $\hat{f}(X)$ when $X$ is drawn from $P_0$.  Formally,
\begin{equation}
  \text{Var}_{P_t}[\hat{f}(X)] = E_{X\sim P_t}\big[(\hat{f}(X) - E_{P_t}[\hat{f}(X)])^2\big].
\end{equation}
Thus, one can expect $\text{Var}_{P_t}[\hat{f}(X)] \approx c^2, \text{Var}_{P_0}[\hat{f}(X)]$. In other words, if $\sigma_t^2 = c^2 \sigma^2$, then the model's output variance scales by about $c^2$. Thus the expected variance term is
\begin{equation}
  \text{Var}_{P_t}(\hat{f}) = E_{X\sim P_t}[\text{Var}_{P_0}(\hat{f}(X))] > E_{X\sim P_0}[\text{Var}_{P_0}(\hat{f}(X))] = \text{Var}_{P_0}(\hat{f}).
\end{equation}

Therefore, for single model, $\text{Bias}^2_{P_t} \approx \text{Bias}^2_{P_0}$, $\text{Var}_{P_t} \approx c^2 \text{Var}_{P_0}= c^2\sigma^2$, but the noise item is the same.

\noindent
\textbf{Ensemble Variance Reduction under $P_t$}

Now consider the ensemble predictor $\hat{f}_{\text{ens}}(X) = \frac{1}{L}\sum_{l=1}^L g_i(X)$. The ensemble bias is $\text{Bias}_{\text{ens}}(X) = E_{P_t}[\hat{f}_{\text{ens}}(X)] - \hat{f}_0(X) = \frac{1}{L}\sum_{i=1}^L E_{P_t}[g_l(X)] - \hat{f}_0(X) = \frac{1}{L}\sum_{i=1}^L (g_l(X) + \text{Bias}(X)) - \hat{f}_0(X)$. Since each model is unbiased ($\text{Bias}_{\text{ens}}(X)=0$), then clearly $\text{Bias}_{\text{ens}}(X)=0$. In the worst case that all models have the same bias, the ensemble's bias is the same $\text{Bias}_{\text{ens}}(X)$. 
Further, the variance of the average is $\frac{1}{L^2}$ times the sum of their variances
\begin{equation}
  \text{Var}_{{P_t}} \big(\hat{f}_{\text{ens}}(X)\big) = \text{Var}_{P_t}\Big(\frac{1}{L}\sum_{i=1}^L g_l(X)\Big) = \frac{1}{L^2}\sum_{i,j}\text{Cov}(g_i(X), g_j(X)).
\end{equation}
If model errors are uncorrelated, $\text{Cov}(g_l,g_j)=0$ for $l\neq j$, and $\text{Var}(g_l)$ is the same $\forall l$. Then this simplifies to
\begin{equation}
  \text{Var}_{P_t} \big(\hat{f}_{\text{ens}}(X)\big) = \frac{1}{L^2}\Big(L\text{Var}_{P_t}(g_l(X))\Big) = \frac{1}{L}\text{Var}_{P_t}(g_l(X)).
\end{equation}
In the general case, let $\alpha_l \in (0,1)$ is the average pairwise correlation between model prediction errors, one can show
\begin{equation}
  \text{Var}_{P_t}(\hat{f}_{\text{ens}}(X)) = \frac{1}{L^2}\Big(L c^2 \sigma^2 + L(L-1)\alpha_l \sigma^2\Big) = \frac{1 + (L-1)\alpha_l}{L} c^2 \sigma^2,
\end{equation}
Thus, ensembling $L$ predictors cuts down the variance roughly by a factor of $L$, while the ensemble bias term $\text{Bias}_{\text{ens}}(X)$ is not worse than the single-model bias. 

Finally, we compare the overall prediction error of a single model to that of the ensemble on the shifted distribution $P_t$. Using the bias–variance decomposition as shown in Eq. \ref{eq_f0_mse}, we have
\begin{equation}
  \text{MSE}_{P_t}(\hat{f}_0) - \text{MSE}_{P_t}(\hat{f}_{\text{ens}}) \approx c^2 \sigma^2 - \frac{1 + (L-1)\alpha_l}{L} c^2 \sigma^2 = \frac{(L-1)(1-\alpha_l)}{L}c^2\sigma^2_t \ge 0. 
\end{equation}
\noindent
Therefore, under $P_t$, the ensemble's total prediction error remains lower than that of any single model.
\end{proof}

\subsection{Proof of Equation \ref{eq_min_loss}}
\label{seca_proof_eq6}
Here, we want to prove that Eq. \ref{eq_min_loss} yields
\begin{equation}
  (Y - \hat{f}_{\text{ens}}(X))^2 = \sum_{l=1}^L \alpha_l (Y-g_l(X))^2 - \sum_{l=1}^L \alpha_l (\hat{f}_{\text{ens}}(X) - g_l(X))^2
  \label{eq_app_6}.
\end{equation}

\begin{proof}
\begin{align}
  (Y - \hat{f}_{\text{ens}}(X))^2 & = Y^2 -2\hat{f}_{\text{ens}}(X) Y + \hat{f}_{\text{ens}}(X)^2  \notag \\
  & = \sum_{l=1}^{L} \alpha_l Y^2 -2 \sum_{l=1}^{L} \alpha_l g_l(X) Y + \hat{f}_{\text{ens}}(X)^2 \notag \\ 
  & = \sum_{l=1}^{L} \alpha_l Y^2 -2 \sum_{l=1}^{L} \alpha_l g_l(X) Y  + \sum_{l=1}^{L} \alpha_l g_l(X)^2 - \sum_{l=1}^{L} \alpha_l g_l(X)^2  + \hat{f}_{\text{ens}}(X)^2  \notag \\
  & = \sum_{l=1}^{L} \alpha_l \left( Y^2 -2g_l(X)Y + g_l(X)^2 \right)  - \left(\sum_{l=1}^{L} \alpha_l g_l(X)^2 - \hat{f}_{\text{ens}}(X)^2 \right) \notag \\
  & = \sum_{l=1}^{L} \alpha_l \left( Y - g_l(X) \right)^2  - \left(\sum_{l=1}^{L} \alpha_l g_l(X)^2   - 2\hat{f}_{\text{ens}}(X)^2  + \hat{f}_{\text{ens}}(X)^2  \right) \notag \\
  & = \sum_{l=1}^{L} \alpha_l \left( Y - g_l(X) \right)^2  - \left(\sum_{l=1}^{L} \alpha_l g_l(X)^2   - 2 \sum_{l=1}^{L} \alpha_l g_l(X) \hat{f}_{\text{ens}}(X)  + \hat{f}_{\text{ens}}(X)^2  \right) \notag \\
  & = \sum_{l=1}^{L} \alpha_l \left( Y - g_l(X) \right)^2  - \sum_{l=1}^{L} \alpha_l \left(g_l(X)^2   - 2 g_l(X) \hat{f}_{\text{ens}}(X)  + \hat{f}_{\text{ens}}(X)^2\right) \notag \\
  & = \sum_{l=1}^{L} \alpha_l \left( Y - g_l(X) \right)^2  - \sum_{l=1}^{L} \alpha_l \left( \hat{f}_{\text{ens}}(X)- g_l(X) \right)^2
\end{align}
\end{proof}

\subsection{Proof of Theorem \ref{th2}}
\label{seca_proof_th2}

Now, we proof how subtraction in DeepBooTS mitigates concept drift, i.e., how Equation \ref{eq_a2} diminishes the variance term in Equation \ref{eq_a3}.
According to the estimation error $e_l \overset{i.i.d}{\sim} \mathcal{N}(0, \nu)$, we have $\text{Var}(g_l(X))=\nu$. Then, utilizing $\text{Cov}(g_l, g_{k\neq l})=\mu$ in Theorem \ref{th1}, we have the proof as follows:
\begin{proof}
\begin{align}
  \text{Var}(\hat{Y}) & = \text{Var}(\hat{f}_{\text{ens}}(X)) = \frac{1}{\hbar^2} \text{Var}\left( i \sum_{l=0}^{\hbar} \alpha g_{2l+1}(X) - i \sum_{l=0}^{\hbar} \alpha g_{2l}(X) \right) \nonumber \\
  & = \frac{1}{\hbar^2} \text{Var}\left( \sum_{l=0}^{\hbar} \alpha g_{2l+1}(X)\right) +   \frac{1}{\hbar^2} \text{Var}\left( \sum_{l=0}^{\hbar} \alpha g_{2l}(X) \right) 
  -  \frac{1}{\hbar^2} \text{Cov}\left( \sum_{l=0}^{\hbar} \alpha g_{2l+1}(X) \sum_{l=0}^{\hbar} \alpha g_{2l}(X)  \right) \nonumber \\
  & = \frac{1}{\hbar^2} \sum_{l=0}^{\hbar} \alpha^2 \text{Var}\left( g_{2l+1}(X)\right) 
  +  \frac{1}{\hbar^2} \sum_{l=0}^{\hbar} \sum_{k=1,k\neq j}^{\hbar} \alpha^2 \text{Cov}\left( g_{2l+1}(X)  g_{2k+1}(X)  \right) \nonumber \\ 
  & \ \ \ \ \ \ + \frac{1}{\hbar^2} \sum_{l=0}^{\hbar} \alpha^2 \text{Var}\left(  g_{2l}(X) \right)
  +  \frac{1}{\hbar^2} \sum_{l=0}^{\hbar} \sum_{k=1,k\neq j}^{\hbar} \alpha^2 \text{Cov}\left( g_{2l}(X)  g_{2k}(X)  \right)  \nonumber \\
  & \ \ \ \ \ \ -  \frac{1}{\hbar^2} \sum_{l=0}^{\hbar} \sum_{k=0}^{\hbar} \alpha^2 \text{Cov}\left(  g_{2l+1}(X)  g_{2k}(X)  \right) \nonumber \\
  & = \frac{1}{\hbar^2} \sum_{l=0}^{\hbar} \alpha^2 \nu 
  +  \frac{1}{\hbar^2} \sum_{l=0}^{\hbar} \sum_{k=1,k\neq j}^{\hbar} \alpha^2 
  + \frac{1}{\hbar^2} \sum_{l=0}^{\hbar} \alpha^2 \nu  
    +  \frac{1}{\hbar^2} \sum_{l=0}^{\hbar} \sum_{k=1,k\neq j}^{\hbar} \alpha^2 \mu 
  -  \frac{1}{\hbar^2} \sum_{l=0}^{\hbar} \sum_{k=0}^{\hbar} \alpha^2 \mu \nonumber \\
  & = \frac{1}{\hbar} \alpha^2 \nu +  \frac{\hbar-1}{\hbar}  \alpha^2 \mu 
  + \frac{1}{\hbar} \alpha^2 \nu +  \frac{\hbar-1}{\hbar} \alpha^2 \mu -  \alpha^2 \mu \nonumber \\
  & = \frac{2}{\hbar} \alpha^2 \nu +  2\frac{\hbar-1}{\hbar}  \alpha^2 \mu -  \alpha^2 \mu \nonumber \\
  & < \frac{2}{\hbar} \alpha^2 (\nu + \mu) \nonumber \\
  & \le \frac{4}{L} \alpha^2 (\nu + \mu). \label{eq_a4}
\end{align}
\end{proof}

Here, we focus on ways to reduce the variance of the model $\text{Var}(\hat{Y})$, not the noise $\varepsilon$ inherent in the TS data.
Equation \ref{eq_a4} indicates that if the models trained in ensemble are highly correlated, using multiple models offers no significant advantage over using a single model. However, when the trained models are diverse, the variance of the estimation can be significantly reduced by a factor of 
$L$. This also implies that ensemble is never detrimental; it either has no effect or enhances the estimation by reducing variance.
It is important to note that in DeepBooTS, the modules are interdependent, so Theorem \ref{th1} only offers an ideal upper bound on variance estimation. Nonetheless, it provides valuable insights into designing deep models that can mitigate concept drift.

\section{Dataset}
\label{app_dataset}

\subsection{Commonly used TS datasets}

The information of the experiment datasets used in this paper are summarized as follows: (1) Electricity Transformer Temperature (ETT) dataset \cite{Zhou2021Informer}, which contains the data collected from two electricity transformers in two separated counties in China, including the load and the oil temperature recorded every 15 minutes (ETTm) or 1 hour (ETTh) between July 2016 and July 2018. (2) Electricity (ECL) dataset \footnote[1]{https://archive.ics.uci.edu/ml/datasets/ElectricityLoadDiagrams20112014} collects the hourly electricity consumption of 321 clients (each column) from 2012 to 2014. (3) Exchange \cite{lai2018modeling} records the current exchange of 8 different countries from 1990 to 2016. (4) Traffic dataset \footnote[2]{http://pems.dot.ca.gov} records the occupation rate of freeway system across State of California measured by 861 sensors. (5) Weather dataset \footnote[3]{https://www.bgc-jena.mpg.de/wetter} records every 10 minutes for 21 meteorological indicators in Germany throughout 2020. (6) Solar-Energy \cite{lai2018modeling} documents the solar power generation of 137 photovoltaic (PV) facilities in the year 2006, with data collected at 10-minute intervals. (7) The PEMS dataset \cite{liu2022scinet} comprises publicly available traffic network data from California, collected within 5-minute intervals and encompassing 358 attributes.
(8) Illness (ILI) dataset \footnote[4]{https://gis.cdc.gov/grasp/fluview/fluportaldashboard.html} describes the influenza-like illness patients in the United States between 2002 and 2021, which records the ratio of patients seen with illness and the total number of the patients. 
The detailed statistics information of the datasets is shown in Table \ref{tb1}.

\begin{table}[!ht]
  \centering
  \caption{Details of the seven TS datasets. }
  \label{tb1}
  \fontsize{9pt}{10pt}\selectfont
  \setlength{\tabcolsep}{1mm}
    \begin{tabular}{cccc}
      \toprule
       Dataset    & length  & features & frequency \\
      \midrule
      ETTh1       & 17,420  & 7       & 1h \\
      ETTh2       & 17,420  & 7       & 1h \\
      ETTm1       & 69,680  & 7       & 15m \\
      ETTm2       & 69,680  & 7       & 15m \\
      Electricity & 26,304  & 321     & 1h  \\
      Exchange    & 7,588   & 8       & 1d  \\
      Traffic     & 17,544  & 862     & 1h  \\
      Weather     & 52,696  & 21      & 10m \\
      Solar       & 52,560  & 137     & 10m  \\
      PEMS        & 26,208  & 358     & 5m  \\
      Illness     & 966     & 7       & 7d  \\
      \bottomrule
    \end{tabular}
  \vspace{-1em}
\end{table}

\subsection{Monash TS Forecasting Datasets}
(1) Saugeenday dataset \footnote[5]{https://zenodo.org/records/4656058} contains a single very long time series representing the daily mean flow of the Saugeen River at Walkerton in cubic meters per second and the length of this time series is 23741.
(2) Sunspot dataset \footnote[6]{https://www.kaggle.com/datasets/robervalt/sunspots} contains monthly numbers of sunspots, as from the World Data Center, aka SIDC, between 1749 and 2018 with a total observation of 3240 months.
(3) M4 dataset \cite{MAKRIDAKIS202054} is a collection of 100,000 time series used in the fourth Makridakis Prediction Contest. The dataset consists of a time series of annual, quarterly, monthly, and other weekly, daily, and hourly data. 
In this paper, we utilize the hourly version of the M4 dataset and standardize its length to 768.
(4) NN5 dataset \footnote[7]{http://www.neural-forecasting-competition.com/downloads/NN5/datasets/download.htm} contains daily time series originated from the observation of daily withdrawals at 111 randomly selected different cash machines at different locations within England.
(5) Oikolab Weather dataset \footnote[8]{https://docs.oikolab.com} contains hundreds of terabytes of weather data, all of them are post-processed data from national weather agencies.
(6) US Births dataset \footnote[9]{https://cran.r-project.org/web/packages/mosaicData} provides birth rates and related data across the 50 states and DC from 2016 to 2021, with a total observation of 7305.
(7) Rideshare dataset \cite{godahewa2021monash} comprises diverse hourly time series representations of attributes pertinent to Uber and Lyft rideshare services across multiple locations in New York. The data spans from November 26th, 2018, to December 18th, 2018.
All the above data sets are collected in the Monash library\footnote[10]{https://forecastingdata.org}.

\subsection{Large-Scale Time Series Datasets}
\label{app_large_ts}

Now we introduce three large-scale TS datasets:

(1) Milano \cite{barlacchi2015multi} is provided by SKIL\footnote[11]{http://jol.telecomitalia.com/jolskil} of Telecom Italia. The dataset was collected from November 1, 2013, to January 1, 2014, and the data is aggregated into 10-minute intervals over the whole city of Milan (62 days, 300 million records, about 19 GB).
Milano records the time of the interaction and the specific radio base station that managed it. 
The dataset contains 10,000 nodes, and the length of each node is 1498. The ratio of training, validation and test datasets is 10:2:3.

(2) CBS is collected from all base station in a certain city in China, including up-link and
down-link communication trafﬁc data. The dataset is collected from June 8 to July 26, 2019, with a temporal interval of 60 minutes over the whole city (94 days, 51.9 million records, about 17.1 GB). 
CBS contains 4454 nodes, each with a length of 4032. The ratio of training, validation and test datasets is 3:1:1.

(3) C2TM \cite{Modeling15Chen} makes an analysis on week-long traffic generated by a large population of people in a median-size city of China.
C2TM analyses make use of request-response records extracted from traffic at the city scale, consisting of individuals' activities during a continuous week (actually eight days from Aug. 19 to Aug. 26, 2012), with accurate timestamp and location information indicated by connected cellular base stations. 
C2TM contains 13269 nodes, each with a length of 192. The ratio of training, validation and test datasets is 5:1:2.

\section{Implementation Details}
\label{app_exp_imp}

The model undergoes training utilizing the ADAM optimizer \cite{kingma2014adam} and minimizing the Mean Squared Error (MSE) loss function.
The training process is halted prematurely, typically within 10 epochs.
The DeepBooTS architecture solely comprises the embedding layer and backbone architecture, devoid of any additional introduced hyperparameters.
Refer to Appendix \ref{app_hpyer} for the hyperparameter sensitivity analysis.
During model validation, two evaluation metrics are employed: MSE and Mean Absolute Error (MAE).

For the large-scale datasets, we initially employ the RSS algorithm in \cite{liang2024act} to perform data sampling. In this algorithm, we random decompose the entire large graph into 17 and 25 sub-graphs for the CBS and Milano datasets, respectively.
As in previous works \cite{wang2024timexer, wang2024timemixer, liu2023itransformer, nie2022time, wu2021autoformer, Zhou2021Informer, liang2023does}, the model is trained using the ADAM optimizer \cite{kingma2014adam} under input-36-output-24 setting.
The learning rate is set to 1e$^{-4}$ with MSE loss. 
The batch size is set to 64 for training and 1 for testing, and the maximum number of epochs is set to 10.

\section{Performance Gains of DeepBooTS with Different Attention}
\label{app_otherattn}

Figure \ref{fig_other_attn} breaks down how the choice of attention inside DeepBooTS affects forecast accuracy, reporting both MSE and MAE for five attention variants (Full-Attention, Prob-Attention, Autoformer, Flow-Attention, Period-Attention) on three TS datasets.
On the highly stochastic Traffic series, vanilla dense attention still wins. Every sparse or structured variant increases error, with Periodic attention showing the largest drop‑off.
On the Traffic dataset, ProbSparse and Periodic attention each cut MSE by about 3\%, whereas Autoformer's decomposition hurts performance heavily.
On the Weather dataset, ProbSparse delivers the best MSE, Flow attention excels at MAE (–5.9\%), and Autoformer again degrades accuracy. Periodic attention sits in the middle-better than Autoformer but slightly behind Full and ProbSparse in MSE.

\begin{figure*}[!ht]
  \centerline{ \includegraphics[width=\textwidth]{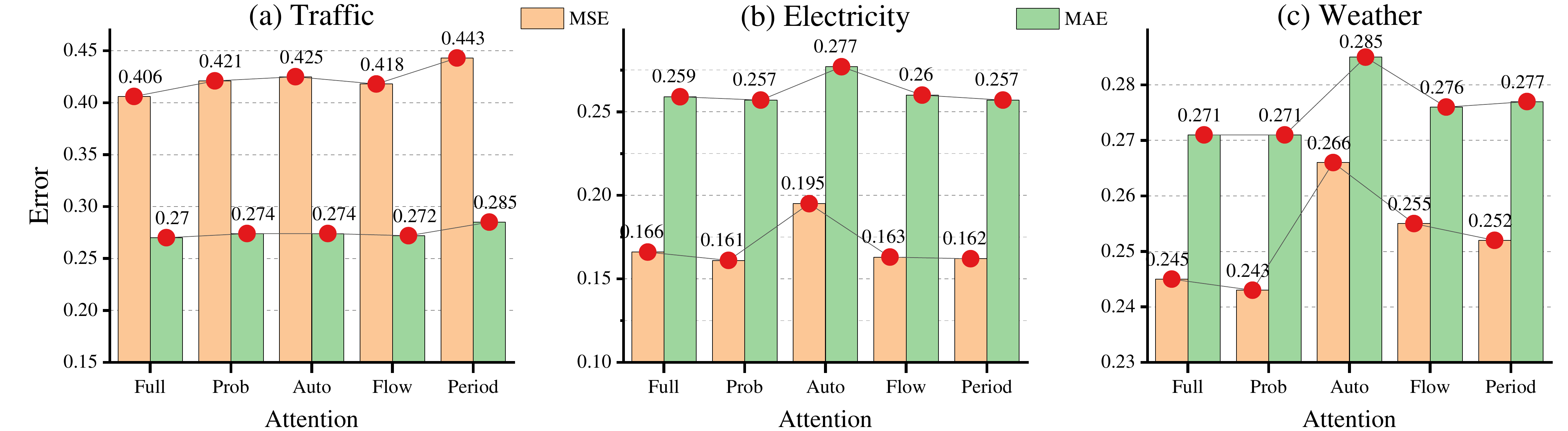}}
  \caption{Performance gains of DeepBooTS with different attention mechanisms.
  }
  \label{fig_other_attn}
\end{figure*}

\section{Performance Gains on Large-Scale TS Datasets}
\label{app_large_ts_per}

We also evaluated the performance gain when DeepBooTS is configured with other models, including PatchTST \cite{nie2022time}, STID \cite{shao2022spatial}, GWNet \cite{wu2019graph}, Autoformer \cite{wu2021autoformer} and Informer \cite{Zhou2021Informer}.
As shown in Fig. \ref{fig_large_ts}, DeepBooTS brings significant performance gains to other models in large-scale TS datasets. 
DeepBooTS consistently lowers forecasting error across all five backbones and three datasets, with relative gains ranging from small (2\% on C2TM) to massive (up to 94\% on CBS), highlighting its ability to stabilize and sharpen predictions even for models struggling under heavy distribution shifts.

\begin{figure*}[!ht]
  \centerline{\includegraphics[width=\textwidth]{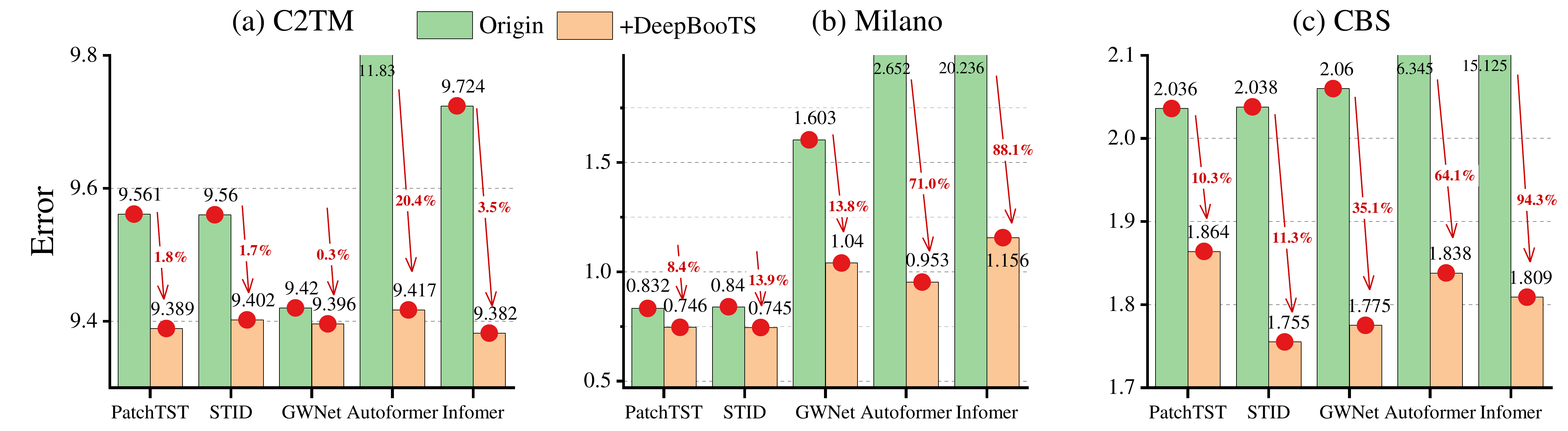}}
  \caption{Performance gains that DeepBooTS contributes to other models in large-scale TS datasets. 
  }
  \label{fig_large_ts}
\end{figure*}

\section{Model Complexity and Computation Cost}
\label{seca_cost}

DeepBooTS contains gate mechanism and attention module. Among them, the gate mechanism is composed of linear layer and sigmoid activation function, i.e., $O_{l+1}=$Sigmoid(Linear($X_l$))$\cdot$Linear($X_l)$, which still maintains linear complexity. Therefore, like other Transformer models, the main complexity of our model is mainly in the attention module ($O(N^2)$).

Further, we evaluate the running time, memory usage, parameters, and FLOPs of our model in comparison to other Transformer-based models using identical settings. The computation costs are shown in the Table \ref{tb_cost}.
It reveals that the proposed model has a smaller computational cost and faster running speed compared to others. It is noteworthy that the introduction of the auxiliary output branch only marginally increases the computational cost, i.e., the differences between the indicators of our model without gate and those of the original model are negligible.

\begin{table}[!ht]
  \centering
  \caption{Model complexity, training efficiency and computation cost. }
  \label{tb_cost}
  \fontsize{9pt}{12pt}\selectfont
  \setlength{\tabcolsep}{1mm}
    \begin{tabular}{ccccc}
      \toprule
    Models                 & Seconds/Epoch & GPU Memory Usage (GB) & Parameters (MB) & FLOPs (GB) \\ \toprule
    DeepBooTS            & 28.7          & 1.2              & 8.8       & 3.1   \\
    DeepBooTS w/o Gate & 28.0          & 1.1              & 6.5       & 2.5   \\
    PatchTST               & 30.1          & 2.3              & 8.4       & 17.4  \\
    FEDformer              & 396           & 7.1              & 14.7      & 558.1 \\
    Autoformer             & 84.5          & 4.8              & 10.5      & 66.2  \\
    Informer               & 81.8          & 2.3              & 11.3      & 63.3   \\ \bottomrule
    \end{tabular}
\end{table}

\section{Concept Drift in TS Forecasting}
\label{app_overfit_exp}

We recorded the training process and verification errors of various models, including Informer, Autoformer, FEDformer, DLinear, iTransformer, and DeepBooTS (ours).
The experimental setup is the same as in the original paper of these methods.
As shown in Fig. \ref{fig_app_overift} (zoom in for better viewing), the model exhibits a significant concept drift phenomenon, particularly in the early stages of training.
We analyze that the main cause of concept drift in this scenario is the high model complexity: Using highly complex models with many parameters can lead to the model ``memorizing'' the noise and specific details in the training data. This results in excellent performance on the training data but poor generalization to new data.
In this case, using a linear model (e.g., DLinear in Fig. \ref{fig_app_overift}(e)) will mitigate concept drift but will also fail to capture temporal correlations.
The proposed DeepBooTS method adopts the idea of ensemble learning to provide flexible complexity for the model, thereby alleviating the poor generalization esulting from concept drift the model to the training data.
This can be verified in Fig. \ref{fig_app_overift}(f), where DeepBooTS has high training error but low test error.

\begin{figure*}[!ht]
  \centering
  \centerline{\includegraphics[width=0.98\textwidth]{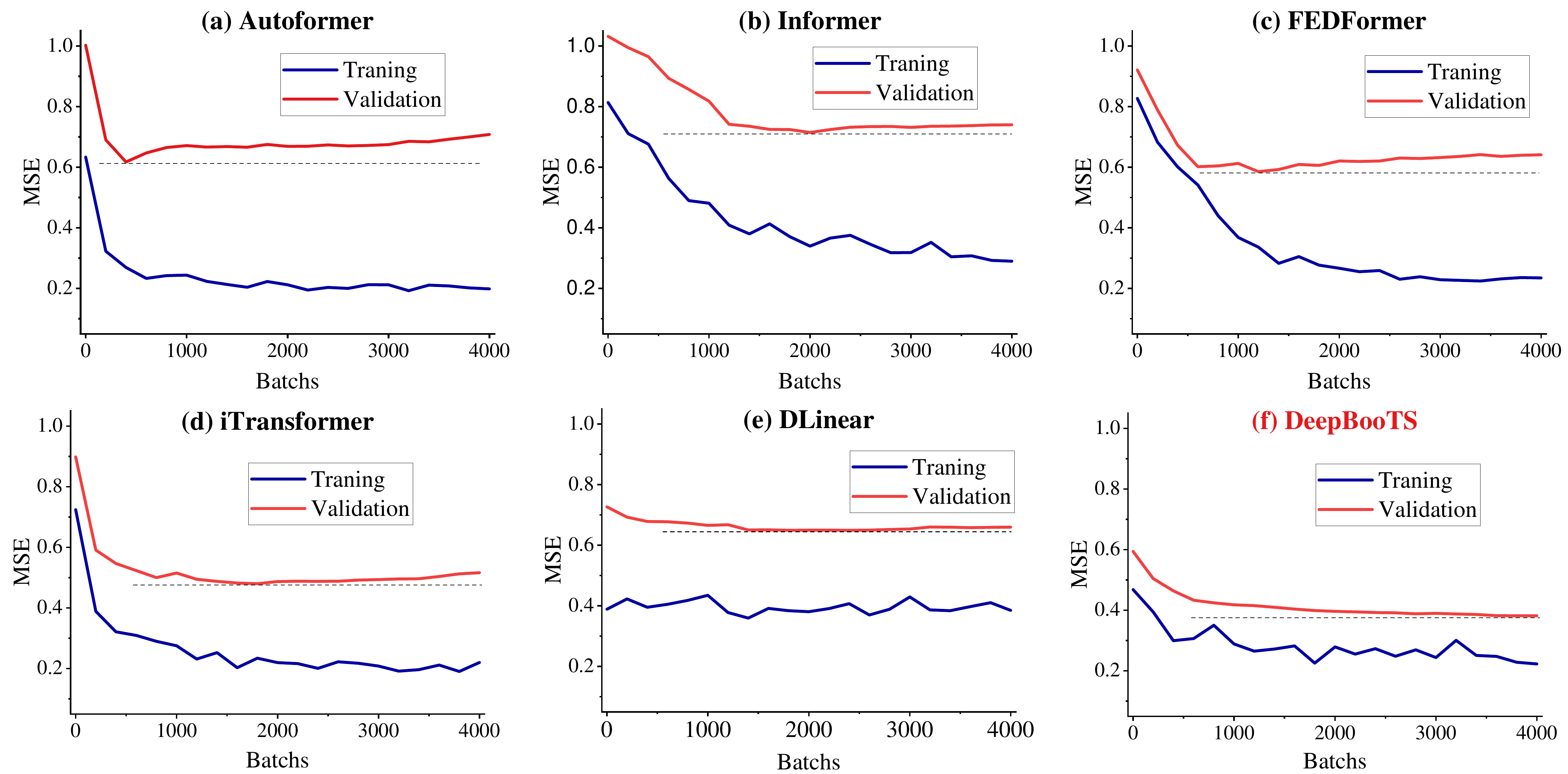}}
  \caption{Concept drift in TS Forecasting. Experiments are conducted utilizing various models on the Traffic dataset. Similar phenomena can be observed on other datasets.}
  \label{fig_app_overift} 
\end{figure*}

\section{Hyperparameter Sensitivity}
\label{app_hpyer}

We evaluate the hyperparameter sensitivity of DeepBooTS with respect to the learning rate, the number of the block, the batch size and the embedding dimension.
As shown in Fig. \ref{fig_hyper}, the performance of DeepBooTS fluctuates under different hyperparameter settings.
In most cases, increasing the number of blocks tends to enhance model performance.
Once again, this confirms that DeepBooTS exhibits resilience against concept drift across diverse datasets.
Notably, we observe that the learning rate, being the most prevalent influencing factor, especially in scenarios involving numerous attributes.
Meanwhile, modifying the batch size induces minor fluctuations in the model's performance, albeit with a limited impact.
Furthermore, we conducted extensive analysis on hyperparameters, as shown in Tables \ref{tb_hpyer1}, \ref{tb_hpyer2}, and \ref{tb_hpyer3}. It is evident that the proposed model is less insensitive to numerous hyperparameters.
In summary, DeepBooTS exhibits low sensitivity to these hyperparameters, thereby enhancing its resilience against concept drift.

\begin{figure*}[!ht]
  \centering
  \centerline{\includegraphics[width=\textwidth]{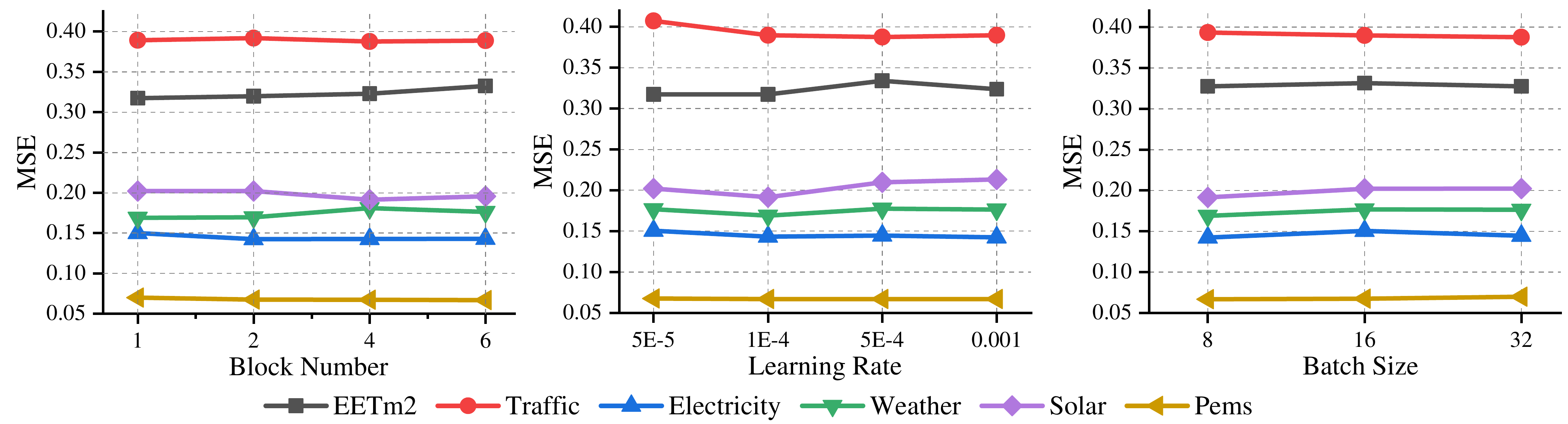}}
  \caption{Hyperparameter sensitivity with respect to the number of block, the learning rate and the number of batch size. The results are recorded with the input length $I = 96$ and the prediction length $O = 12$ for PEMS and $O = 96$ for others.}
  \label{fig_hyper} 
\end{figure*}

\begin{table}[!ht]
    \centering
    \caption{Ablation studies of DeepBooTS's hyperparameters on the Electricity dataset.}
    \label{tb_hpyer1}
    \fontsize{9pt}{10pt}\selectfont
    \setlength{\tabcolsep}{2mm}
    \begin{tabular}{c|ccccc| cccc|cccc}
    \toprule
      & \multicolumn{5}{c}{Block Number}            & \multicolumn{4}{c}{Embedding Dimension} & \multicolumn{4}{c}{Learning Rate} \\ \toprule
    Metrics     & 2     & 3     & 4     & 5     & 6     & 92     & 384   & 512   & 768   & 1E-3 & 5E-4 & 1E-4 & 5E-5 \\  \toprule
    MSE         & 0.137 & 0.137 & 0.137 & 0.138 & 0.137 & 0.140   & 0.137 & 0.143 & 0.136 & 0.143 & 0.143  & 0.144  & 0.145   \\
    MAE         & 0.231 & 0.231 & 0.231 & 0.230  & 0.231 & 0.233  & 0.231 & 0.235 & 0.231 & 0.231 & 0.231  & 0.230   & 0.232   \\
    RMSP        & 0.214 & 0.214 & 0.256 & 0.215 & 0.213 & 0.219  & 0.215 & 0.214 & 0.212 & 0.214 & 0.214  & 0.212  & 0.226   \\
    MAPE        & 2.295 & 2.294 & 2.34  & 2.32  & 2.32  & 2.397  & 2.318 & 2.295 & 2.305 & 2.29  & 2.303  & 2.302  & 2.35    \\
    sMAPE       & 0.466 & 0.464 & 0.467 & 0.466 & 0.464 & 0.475  & 0.467 & 0.466 & 0.468 & 0.466 & 0.466  & 0.465  & 0.469   \\
    MASE        & 0.257 & 0.256 & 0.259 & 0.257 & 0.257 & 0.266  & 0.258 & 0.257 & 0.261 & 0.257 & 0.257  & 0.258  & 0.275   \\
    Q25         & 0.232 & 0.229 & 0.232 & 0.232 & 0.231 & 0.234  & 0.232 & 0.232 & 0.233 & 0.232 & 0.229  & 0.229  & 0.238   \\
    Q75         & 0.23  & 0.232 & 0.23  & 0.231 & 0.23  & 0.24   & 0.231 & 0.23  & 0.231 & 0.23  & 0.232  & 0.233  & 0.243  \\ \bottomrule
    \end{tabular}
\end{table}

\begin{table}[!ht]
    \centering
    \caption{Ablation studies of DeepBooTS's hyperparameters on the Traffic dataset.}
    \label{tb_hpyer2}
    \fontsize{9pt}{10pt}\selectfont
    \setlength{\tabcolsep}{2mm} 
    \begin{tabular}{c|ccccc| cccc|cccc}
        \toprule
         & \multicolumn{5}{c|}{Block Number}            & \multicolumn{4}{c|}{Embedding Dimension}       & \multicolumn{4}{c}{Learning Rate} \\ \toprule
        Metrics & 2     & 3     & 4     & 5     & 6     & 92     & 384     &  512    & 768     & 1E-3   & 5E-4   & 1E-4   & 5E-5 \\ \toprule
        MSE     & 0.400 & 0.401 & 0.400 & 0.400 & 0.401 & 0.393  & 0.387   & 0.387   & 0.389   & 0.389  & 0.391  & 0.399  & 0.404  \\
        MAE     & 0.266 & 0.267 & 0.267 & 0.267 & 0.267 & 0.265  & 0.258   & 0.259   & 0.257   & 0.257  & 0.256  & 0.261  & 0.271  \\
        RMSP    & 0.208 & 0.209 & 0.208 & 0.209 & 0.208 & 0.207  & 0.198   & 0.199   & 0.200   & 0.207  & 0.203  & 0.208  & 0.214  \\
        MAPE    & 2.730 & 2.770 & 2.765 & 2.750 & 2.750 & 2.760  & 2.670   & 2.690   & 2.700   & 2.720  & 2.660  & 2.720  & 2.820  \\
        sMAPE   & 0.485 & 0.486 & 0.486 & 0.486 & 0.486 & 0.481  & 0.470   & 0.480   & 0.480   & 0.484  & 0.482  & 0.491  & 0.492  \\
        MASE    & 0.257 & 0.258 & 0.258 & 0.258 & 0.258 & 0.255  & 0.248   & 0.250   & 0.252   & 0.249  & 0.250  & 0.274  & 0.262  \\
        Q25     & 0.262 & 0.265 & 0.264 & 0.264 & 0.264 & 0.263  & 0.256   & 0.256   & 0.254   & 0.253  & 0.255  & 0.259  & 0.268  \\
        Q75     & 0.271 & 0.270 & 0.270 & 0.270 & 0.271 & 0.267  & 0.260   & 0.262   & 0.261   & 0.261  & 0.258  & 0.264  & 0.275  \\ \bottomrule
    \end{tabular}
\end{table}

\begin{table}[!ht]
    \centering
    \caption{Ablation studies of DeepBooTS's hyperparameters on the Weather dataset.}
    \label{tb_hpyer3}
    \fontsize{9pt}{10pt}\selectfont
    \setlength{\tabcolsep}{2mm}
        \begin{tabular}{c|ccccc| cccc|cccc}
            \toprule
             & \multicolumn{5}{c|}{Block Number}            & \multicolumn{4}{c|}{Embedding Dimension}       & \multicolumn{4}{c}{Learning Rate} \\ \toprule
            Metrics & 2     & 3     & 4     & 5     & 6     & 92     & 384     &  512     & 768     & 1E-3   & 5E-4   & 1E-4   & 5E-5 \\ \toprule
            MSE     & 0.160 & 0.156 & 0.160 & 0.159 & 0.157 & 0.161  & 0.156   & 0.158   & 0.159   & 0.215   & 0.168  & 0.157  & 0.159 \\
            MAE     & 0.203 & 0.201 & 0.206 & 0.203 & 0.199 & 0.205  & 0.201   & 0.202   & 0.202   & 0.271   & 0.211  & 0.202  & 0.202 \\
            RMSP    & 0.213 & 0.213 & 0.220 & 0.215 & 0.206 & 0.213  & 0.211   & 0.210   & 0.212   & 0.324   & 0.218  & 0.212  & 0.210 \\
            MAPE    & 8.990 & 8.860 & 9.220 & 9.430 & 9.130 & 10.600 & 9.151   & 8.387   & 8.450   & 17.680  & 8.530  & 7.900  & 7.860 \\
            sMAPE   & 0.529 & 0.524 & 0.534 & 0.528 & 0.519 & 0.531  & 0.526   & 0.525   & 0.524   & 0.637   & 0.542  & 0.529  & 0.523 \\
            MASE    & 1.162 & 1.110 & 1.172 & 1.161 & 1.090 & 1.263  & 1.171   & 1.151   & 1.094   & 1.230   & 1.202  & 1.170  & 1.200 \\
            Q25     & 0.198 & 0.196 & 0.203 & 0.197 & 0.188 & 0.197  & 0.195   & 0.197   & 0.189   & 0.272   & 0.206  & 0.200  & 0.195 \\
            Q75     & 0.209 & 0.207 & 0.209 & 0.209 & 0.210 & 0.212  & 0.208   & 0.206   & 0.215   & 0.269   & 0.215  & 0.205  & 0.209 \\ \bottomrule
        \end{tabular}
  \end{table}

\section{Different Attention Promotion}
\label{app_improve_attn}

\begin{figure*}[!ht]
  \centering
  \centerline{\includegraphics[width=\textwidth]{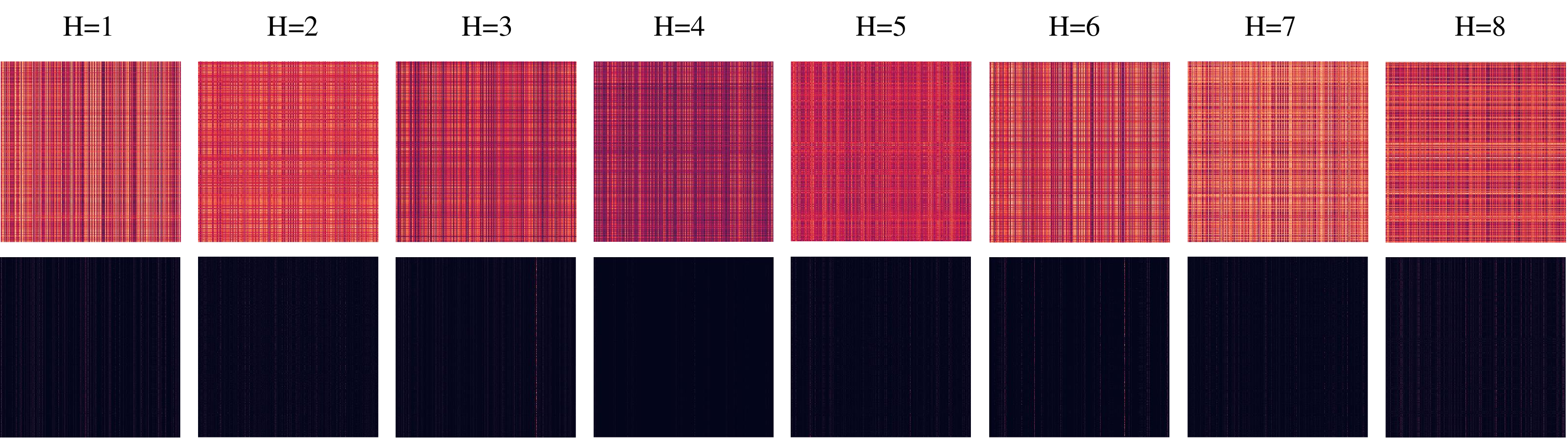}}
  \caption{Visualization of Attention utilized in DeepBooTS. The model is trained on Traffic dataset with 862 attributes under the setting of Input-96-Predict-96. Both the Attention score (top) and the post-softmax score (below) of the 8 heads (H) are from the first block.}
  \label{fig_vis_attn_all} 
\end{figure*}


We visualized the Attention maps for all Attention heads in the initial layer of DeepBooTS, and the results are presented in Fig. \ref{fig_vis_attn_all}.
It is evident that the Attention score graph exhibits numerous bar structures, which is especially prominent on the post-softmax Attention map.
This implies that there exists a row in Query that bears a striking resemblance to all the columns in Key. 
This scenario arises when there are numerous attributes, many of which are homogeneously represented.
Our speculation suggests that this capability of the Attention may explain its ability to capture subtle patterns in TS without succumbing to concept drift.
Furthermore, we substituted its original Attention with other novel Attention mechanisms to compare the resulting changes in model performance. 
The full results for all prediction lengths are provided in Table \ref{tb_attns}.

\begin{table*}[h!]
  \centering
  \caption{Ablation of different Attention layers.}
  \label{tb_attns}
    \fontsize{9pt}{12pt}\selectfont
    \setlength{\tabcolsep}{3mm}
    \begin{threeparttable}
      \begin{tabular}{ccc >{\columncolor{gray!15}} cc >{\columncolor{gray!15}} cc >{\columncolor{gray!15}} cc >{\columncolor{gray!15}} cc >{\columncolor{gray!15}} c}
        \toprule
        \multicolumn{2}{c}{Attention Layer} & \multicolumn{2}{c}{FullAttention} & \multicolumn{2}{c}{ProbAttention} & \multicolumn{2}{c}{AutoCorrelation} & \multicolumn{2}{c}{FlowAttention} & \multicolumn{2}{c}{PeriodAttention} \\
        \toprule
        Dataset                      & Length  & MSE            & MAE             & MSE             & MAE             & MSE              & MAE              & MSE             & MAE             & MSE              & MAE              \\
        \toprule
        \multirow{5}{*}{Traffic}     & 96     & \underline{\color{blue}0.386}           & \underline{\color{blue}0.258}           & 0.390           & 0.259           & 0.396            & 0.261            & {\bf\color{red}0.385}           & {\bf\color{red}0.256}           & 0.414            & 0.273            \\
                                     & 192    & {\bf\color{red}0.398}           & {\bf\color{red}0.263}           & 0.410           & 0.268           & 0.413            & 0.268            & \underline{\color{blue}0.406}           & \underline{\color{blue}0.265}           & 0.431            & 0.279            \\
                                     & 336    & {\bf\color{red}0.409}           & {\bf\color{red}0.270}           & 0.425           & 0.274           & 0.429            & 0.275            & \underline{\color{blue}0.424}           & \underline{\color{blue}0.274}           & 0.447            & 0.286            \\
                                     & 720    & {\bf\color{red}0.431}           & {\bf\color{red}0.287}           & \underline{\color{blue}0.459}           & \underline{\color{blue}0.293}           & 0.461            & 0.293            & \underline{\color{blue}0.459}           & \underline{\color{blue}0.293}           & 0.478            & 0.303            \\ \cline{2-12}
        \rowcolor{cyan!15}  \cellcolor{white} & Avg    & {\bf\color{red}0.406}           & {\bf\color{red}0.270}           & 0.421           & 0.274           & 0.425            & 0.274            & \underline{\color{blue}0.418}           & \underline{\color{blue}0.272}           & 0.443            & 0.285            \\ \toprule
        \multirow{5}{*}{Electricity} & 96     & 0.135           & 0.229           & {\bf\color{red}0.133}           & {\bf\color{red}0.226}           & 0.174            & 0.256            & \underline{\color{blue}0.134}           & \underline{\color{blue}0.228}           & 0.142            & 0.233            \\
                                     & 192    & 0.151          & 0.246           & {\bf\color{red}0.148}           & {\bf\color{red}0.241}           & 0.179            & 0.263            & 0.150           & 0.251           & \underline{\color{blue}0.148}            & \underline{\color{blue}0.243}            \\
                                     & 336    & 0.172           & 0.268           & \underline{\color{blue}0.167}           & \underline{\color{blue}0.263}           & 0.195            & 0.279            & 0.174           & 0.267           & {\bf\color{red}0.166}            & {\bf\color{red}0.262}            \\
                                     & 720    & 0.204           & 0.294           & 0.195           & \underline{\color{blue}0.290}           & 0.232            & 0.311            &  {\bf\color{red}0.193}           & 0.294           & \underline{\color{blue}0.194}            &  {\bf\color{red}0.289}            \\ \cline{2-12}
        \rowcolor{cyan!15}  \cellcolor{white} & Avg    & 0.166           & 0.259           &  {\bf\color{red}0.161}           & {\bf \color{red}0.255}           & 0.195            & 0.277            & 0.163           & 0.260           & \underline{\color{blue}0.162}            & \underline{\color{blue}0.258}            \\ \toprule
        \multirow{5}{*}{Weather}     & 96     & \underline{\color{blue}0.160}           & \underline{\color{blue}0.206}           & {\bf\color{red}0.159}           & {\bf\color{red}0.204}           & 0.179            & 0.220            & 0.169           &  0.209           & 0.170            & 0.212            \\
                                     & 192    & \underline{\color{blue}0.208}           & \underline{\color{blue}0.250}           & {\bf\color{red}0.207}           & {\bf\color{red}0.248}           & 0.228            & 0.261            & 0.220           & 0.255           & 0.214            & 0.253          \\
                                     & 336    & \underline{\color{blue}0.269}           & \underline{\color{blue}0.296}           & {\bf\color{red}0.265}           & {\bf\color{red}0.291}           & 0.288            & 0.304            & 0.276           & 0.294           & 0.273            & \underline{\color{blue}0.296}            \\
                                     & 720    & \underline{\color{blue}0.343}           & {\bf\color{red}0.334}           & {\bf\color{red}0.340}           & \underline{\color{blue}0.342}           & 0.367            & 0.356            & 0.354           & 0.348           & 0.352            & 0.347            \\ \cline{2-12}
        \rowcolor{cyan!15}  \cellcolor{white} & Avg    & \underline{\color{blue}0.245}           & \underline{\color{blue}0.271}           & {\bf\color{red}0.243 }           & {\bf\color{red}0.271}           & 0.266            & 0.285            & 0.255           & 0.276           & 0.252            & 0.277            \\ 
        \bottomrule
        \end{tabular}
    \begin{tablenotes}
      \item[*] The input length $I$ is set as 96, while the prediction lengths $O \in $ \{96, 192, 336, 720\}. 
    \end{tablenotes}
    \end{threeparttable}
\end{table*}

\section{Full Results on ETT Datasets}
\label{app_ett}
The ETT dataset records electricity data of four different granularities and types. 
We offer an in-depth comparison of DeepBooTS utilizing the complete ETT dataset to facilitate future research endeavors.
Detailed results are provided in Table \ref{tb_ettm}.
It is evident that DeepBooTS demonstrates excellent performance on the complete ETT dataset.

\begin{table*}[!ht]
    \centering
    \caption{Full Multivariate Forecasting Results on ETT dataset.}
    \label{tb_ettm}
    \fontsize{9pt}{12pt}\selectfont
    \setlength{\tabcolsep}{1.3mm}
        \begin{tabular}{ccc >{\columncolor{gray!15}} cc >{\columncolor{gray!15}} cc >{\columncolor{gray!15}} cc >{\columncolor{gray!15}} cc >{\columncolor{gray!15}} cc >{\columncolor{gray!15}} cc >{\columncolor{gray!15}} c}
        \toprule
        \multicolumn{2}{c}{Models}      & \multicolumn{2}{c}{DeepBooTS} & \multicolumn{2}{c}{Periodformer} & \multicolumn{2}{c}{FEDformer} & \multicolumn{2}{c}{Autoformer} & \multicolumn{2}{c}{Informer} & \multicolumn{2}{c}{LogTrans} & \multicolumn{2}{c}{Reformer} \\
        \toprule
                            & Length & MSE              & MAE             & MSE             & MAE            & MSE           & MAE           & MSE            & MAE           & MSE           & MAE          & MSE           & MAE          & MSE           & MAE          \\
        \toprule
        \multirow{5}{*}{\rotatebox{90}{ETTh1}} & 96     & {\bf\color{red}0.370}             & {\bf\color{red}0.394}           & \underline{\color{blue} 0.375}           & \underline{\color{blue} 0.395}          & 0.395         & 0.424         & 0.449          & 0.459         & 0.865         & 0.713        & 0.878         & 0.74         & 0.837         & 0.728        \\
                            & 192    & \underline{\color{blue} 0.423}            & \underline{\color{blue} 0.427}           & {\bf\color{red}0.413}           & {\bf\color{red}0.421}          & 0.469         & 0.47          & 0.5            & 0.482         & 1.008         & 0.792        & 1.037         & 0.824        & 0.923         & 0.766        \\
                            & 336    & \underline{\color{blue} 0.465}            & \underline{\color{blue} 0.446}           & {\bf\color{red}0.443}           & {\bf\color{red}0.441}          & 0.530          & 0.499         & 0.521          & 0.496         & 1.107         & 0.809        & 1.238         & 0.932        & 1.097         & 0.835        \\
                            & 720    & {\bf\color{red}0.465}            & {\bf\color{red}0.464}           & \underline{\color{blue} 0.467}           & \underline{\color{blue} 0.469}          & 0.598         & 0.544         & 0.514          & 0.512         & 1.181         & 0.865        & 1.135         & 0.852        & 1.257         & 0.889        \\ \cline{2-16}
        \rowcolor{cyan!15}  \cellcolor{white} & Avg    & \underline{\color{blue} 0.431}            & \underline{\color{blue} 0.433}           & {\bf\color{red}0.425}           & {\bf\color{red}0.432}          & 0.498         & 0.484         & 0.496          & 0.487         & 1.040         & 0.795        & 1.072         & 0.837        & 1.029         & 0.805        \\ 
        \toprule
        \multirow{5}{*}{\rotatebox{90}{ETTh2}} & 96     & {\bf\color{red}0.285}            & {\bf\color{red}0.333}           & \underline{\color{blue} 0.313}           & \underline{\color{blue} 0.356}          & 0.394         & 0.414         & 0.358          & 0.397         & 3.755         & 1.525        & 2.116         & 1.197        & 2.626         & 1.317        \\
                            & 192    & {\bf\color{red}0.354}            & {\bf\color{red}0.377}           & \underline{\color{blue} 0.389}           & \underline{\color{blue} 0.405}          & 0.439         & 0.445         & 0.456          & 0.452         & 5.602         & 1.931        & 4.315         & 1.635        & 11.12         & 2.979        \\
                            & 336    & {\bf\color{red}0.405}            & {\bf\color{red}0.420}           & \underline{\color{blue} 0.418}           & \underline{\color{blue} 0.432}          & 0.482         & 0.48          & 0.482          & 0.486         & 4.721         & 1.835        & 1.124         & 1.604        & 9.323         & 2.769        \\
                            & 720    & {\bf\color{red}0.418}            & {\bf\color{red}0.438}           & \underline{\color{blue} 0.427}           & \underline{\color{blue} 0.444}          & 0.5           & 0.509         & 0.515          & 0.511         & 3.647         & 1.625        & 3.188         & 1.54         & 3.874         & 1.697        \\ \cline{2-16}
        \rowcolor{cyan!15}  \cellcolor{white} & Avg    & {\bf\color{red}0.366}            & {\bf\color{red}0.392}           & \underline{\color{blue} 0.387}           & \underline{\color{blue} 0.409}          & 0.454         & 0.462         & 0.453          & 0.462         & 4.431         & 1.729        & 2.686         & 1.494        & 6.736         & 2.191        \\
        \toprule
        \multirow{5}{*}{\rotatebox{90}{ETTm1}} & 96     & {\bf\color{red}0.311}            & {\bf\color{red}0.340}           & \underline{\color{blue} 0.337}           & \underline{\color{blue} 0.378}          & 0.378         & 0.418         & 0.505         & 0.475         & 0.672         & 0.571        & 0.600           & 0.546        & 0.538         & 0.528        \\
                            & 192    & {\bf\color{red}0.363}            & {\bf\color{red}0.379}           & \underline{\color{blue} 0.413}           & \underline{\color{blue} 0.431}          & 0.464         & 0.463         & 0.553          & 0.496         & 0.795         & 0.669        & 0.837         & 0.7          & 0.658         & 0.592        \\
                            & 336    & {\bf\color{red}0.393}            & {\bf\color{red}0.391}           & \underline{\color{blue} 0.428}           & \underline{\color{blue} 0.441}          & 0.508         & 0.487         & 0.621          & 0.537         & 1.212         & 0.871        & 1.124         & 0.832        & 0.898         & 0.721        \\
                            & 720    & {\bf\color{red}0.454}            & {\bf\color{red}0.442}           & \underline{\color{blue} 0.483}           & \underline{\color{blue} 0.483}          & 0.561         & 0.515         & 0.671          & 0.561         & 1.166         & 0.823        & 1.153         & 0.82         & 1.102         & 0.841        \\ \cline{2-16}
        \rowcolor{cyan!15}  \cellcolor{white} & Avg    & {\bf\color{red}0.380}            & {\bf\color{red}0.388}           & \underline{\color{blue} 0.415}           & \underline{\color{blue} 0.433}          & 0.478         & 0.471         & 0.588          & 0.517         & 0.961         & 0.734        & 0.929         & 0.725        & 0.799         & 0.671        \\
        \toprule
        \multirow{5}{*}{\rotatebox{90}{ETTm2}} & 96     & {\bf\color{red}0.171}            & {\bf\color{red}0.250}           & \underline{\color{blue} 0.186}           & \underline{\color{blue} 0.274}          & 0.204         & 0.288         & 0.255          & 0.339         & 0.365         & 0.453        & 0.768         & 0.642        & 0.658         & 0.619        \\
                            & 192    & {\bf\color{red}0.235}            & {\bf\color{red}0.292}           & \underline{\color{blue} 0.252}           & \underline{\color{blue} 0.317}          & 0.316         & 0.363         & 0.281          & 0.34          & 0.533         & 0.563        & 0.989         & 0.757        & 1.078         & 0.827        \\
                            & 336    & {\bf\color{red}0.293}            & {\bf\color{red}0.330}            & \underline{\color{blue} 0.311}           & \underline{\color{blue} 0.355}          & 0.359         & 0.387         & 0.339          & 0.372         & 1.363         & 0.887        & 1.334         & 0.872        & 1.549         & 0.972        \\
                            & 720    & {\bf\color{red}0.393}            & {\bf\color{red}0.389}           & \underline{\color{blue} 0.402}           & \underline{\color{blue} 0.405}          & 0.433         & 0.432         & 0.422          & 0.419         & 3.379         & 1.338        & 3.048         & 1.328        & 2.631         & 1.242        \\ \cline{2-16}
        \rowcolor{cyan!15}  \cellcolor{white} & Avg    & {\bf\color{red}0.273}            & {\bf\color{red}0.315}           & \underline{\color{blue} 0.288}           & \underline{\color{blue} 0.338}          & 0.328         & 0.368         & 0.324          & 0.368         & 1.410         & 0.810        & 1.535         & 0.900        & 1.479         & 0.915        \\
        \toprule
        \multicolumn{2}{c}{ $1^{\text{st}}$ Count} & {\bf\color{red} 17}           & {\bf\color{red}17}           & \underline{\color{blue}3}                   & \underline{\color{blue}3}       & 0               & 0               & 0             & 0               & 0               & 0               & 0             & 0             & 0               & 0             \\
        \bottomrule
        \end{tabular}
\end{table*}

\newpage
\section{Pseudocode of DeepBooTS}
\label{sec_algo}

To facilitate a comprehensive understanding of DeepBooTS's working principle, we offer detailed pseudocode outlining its implementation, as shown in Algorithm \ref{algo}.
The implementation presented here outlines the core ideas of DeepBooTS. 
It is evident that the deployment procedure of DeepBooTS exhibits a relative simplicity, characterized by the inclusion of several iteratively applied blocks.
This property renders it highly versatile for integrating newly devised Attention mechanisms or modules.
As demonstrated in Appendix \ref{app_improve_attn}, the substitution of Attention mechanisms in DeepBooTS with novel alternatives yields superior generalization.

\begin{algorithm}[htbp]
  \caption{DeepBooTS Architecture.}\label{algo}
  \begin{algorithmic}[1]
  \REQUIRE 
  Batch size $B$, input lookback time series $\mathbf{X}\in\mathbb{R}^{B \times I\times D}$; input length $I$; predicted length $O$; embedding dimension $E$; the number of the block $L$; the output length $H$ in each block.

    \STATE $\mathbf{X}_0=\texttt{StandardScaler}(\mathbf{X}^T)$ \COMMENT{$\mathbf{X}_0\in\mathbb{R}^{B \times D \times I}$}

    \STATE $\triangleright \ $ Apply a linear transformation to the temporal dimension of X to align it with the embedding dimension.

    \STATE $\mathbf{X}_{1,1}=\texttt{Linear}(\mathbf{X}_0)$ \  \COMMENT{$\mathbf{X}_1 \in\mathbb{R}^{B \times D \times E}$} \  \COMMENT{The embedded X enters the backbone as an input stream.}

    \STATE $O_0 = 0$ \ \COMMENT{Set the initial value of the output stream to 0.}

    \STATE $\textbf{for}\ l\ \textbf{in}\ \{1,\cdots,L\}\textbf{:}$ \ \COMMENT{Run through the backbone.}

    \STATE $\textbf{\textcolor{white}{for}}$  $\triangleright \ $ Apply attention to the input stream. It can be achieved by using frequency Attention.
    \STATE $\textbf{\textcolor{white}{for}} \  \mathbf{\hat{X}}_{l,1} = \texttt{Attention}(\mathbf{X}_{l,1})$ \ \COMMENT{$\mathbf{\hat{X}}_{l,1}\in\mathbb{R}^{B \times D \times E}$} 
    
    \STATE $\textbf{\textcolor{white}{for}}$  $\triangleright \ $ Subtracting the attention output from the input.
    \STATE $\textbf{\textcolor{white}{for}}\ \mathbf{R}_{l, 1} = \mathbf{X}_{l,1} - \delta \ \texttt{Dropout}(\mathbf{\hat{X}}_{l,1})$ \ 
    \COMMENT{$\mathbf{R}_{l, 1}\in\mathbb{R}^{B \times D \times E}$}
    
    \STATE $\textbf{\textcolor{white}{for}}$ $\triangleright \ $ LayerNorm is adopted to reduce attributes discrepancies.
    \STATE $\textbf{\textcolor{white}{for}} \mathbf{X}_{l,2} = \mathbf{R}_{l, 1} = \texttt{LayerNorm}(\mathbf{R}_{l,1}) $ 

    \STATE $\textbf{\textcolor{white}{for}}$ $\triangleright \ $The feedforward exclusively performs nonlinear transformations on the temporal aspect.
    \STATE $\textbf{\textcolor{white}{for}}\ \mathbf{\hat{X}}_{l,2} = \texttt{FeedForward}(\mathbf{R}_{l, 1})$ \ \COMMENT{$\mathbf{\hat{X}}_{l,2} \in\mathbb{R}^{B \times D \times E}$}
    
    \STATE $\textbf{\textcolor{white}{for}}$  $\triangleright \ $ Subtracting the feedforward output from the input.
    \STATE $\textbf{\textcolor{white}{for}}\ \mathbf{R}_{l, 2} = \mathbf{X}_{l,2} - \mathbf{\hat{X}}_{l,2}$ \ \COMMENT{$\mathbf{R}_{l, 2}\in\mathbb{R}^{B \times D \times E}$}
    
    \STATE $\textbf{\textcolor{white}{for}}$  $\triangleright \ $ Add gate mechanism to input stream.
    \STATE $\textbf{\textcolor{white}{for}}\ \mathbf{X}_{l+1} = \texttt{Sigmoid}\big(\texttt{Linear} (\mathbf{R}_{l,2})\big) \cdot \texttt{Linear} (\mathbf{R}_{l,2})$ \ \COMMENT{$\mathbf{X}_{l+1}\in\mathbb{R}^{B \times D \times E}$}

    \STATE $\textbf{\textcolor{white}{for}}$  $\triangleright \ $ Add gate mechanism to output stream.
    \STATE $\textbf{\textcolor{white}{for}}\ \mathbf{\hat{O}}_{l+1} = \texttt{Sigmoid}\big(\texttt{Linear} ([\mathbf{\hat{X}}_{l,1}, \mathbf{\hat{X}}_{l,2}] )\big) \cdot \texttt{Linear} ([\mathbf{\hat{X}}_{l,1}, \mathbf{\hat{X}}_{l,2}])$ \ \COMMENT{$\mathbf{\hat{O}}_{l+1}\in\mathbb{R}^{B \times D \times H}$} 

    \STATE $\textbf{\textcolor{white}{for}}$  $\triangleright \ $ Subtract the previously learned output.
    \STATE $\textbf{\textcolor{white}{for}}\ \mathbf{O}_{l+1} = \mathbf{\hat{O}}_{l+1} - \mathbf{O}_{l}$ \ \COMMENT{$\mathbf{O}_{l+1}\in\mathbb{R}^{B \times D \times H}$}
    
    \STATE $\textbf{End for}$

    \STATE  $\triangleright \ $ Align the final output with the predicted length.
    \STATE $\textbf{if}\  H \neq O \ \ \textbf{then}\ $:
    \STATE $\textbf{\textcolor{white}{for}}$ $ \mathbf{O}_L = \texttt{Linear} (\mathbf{O}_{L-1})$ \ \COMMENT{$\mathbf{O}_{L}\in\mathbb{R}^{B \times D \times O}$}
    \STATE $ \textbf{Output}\ \texttt{InvertedScaler}(\mathbf{O}_L^T)$ \ \COMMENT{Output the final prediction results $\mathbf{O}_{L}^T\in\mathbb{R}^{B \times O \times D}$}
  \end{algorithmic} 
\end{algorithm}

\section{Full Multivariate TS Forecasting Results}
\label{app_mts}

The full results for multivariate TS forecasting are presented in Table \ref{tb_mts}.
It is evident that DeepBooTS demonstrates excellent performance on the complete multivariate TS forecasting.
Notably, several pioneering models have also achieved competitive performance on certain datasets under particular settings.
For instance, Informer, considered a groundbreaking model in long-term TS forecasting, demonstrates advanced performance on the Solar-Energy dataset with input-96-predict-192 and -720 settings.
This is due to the substantial presence of zero values on each column attribute of the Solor-Energy dataset.
This renders the KL-divergence based ProbSparse Attention, as adopted in Informer, highly effective on this sparse dataset.
Additionally, linear-based methods (e.g., DLinear) have demonstrated promising results on the Weather dataset with input-336-predict-720 setting, while convolution-based methods (e.g., SCINet) have yielded favorable results on the PEMS dataset with input-168-predict-192 setting. 
This phenomenon can be ascribed to a twofold interplay of factors. Previously, the diversity of input settings exerts a direct influence on model generalization. Secondarily, other models exhibit a propensity to overfit non-stationary TS characterized by aperiodic fluctuations.
Remarkably, DeepBooTS adeptly mitigates both concept drift and underfitting challenges in multivariate TS forecasting, thereby enhancing its overall performance.
Particularly on datasets with numerous attributes, e.g., Traffic and Solor-Energy, DeepBooTS achieves superior performance by feeding the learned meaningful patterns to the output layer at each block.

\begin{table*}[!ht]
    \centering
    \caption{Multivariate TS forecasting results on six benchmark datasets.}
    \label{tb_mts}
      \fontsize{9pt}{12pt}\selectfont
      \setlength{\tabcolsep}{0.8mm}
        \begin{tabular}{ccc >{\columncolor{gray!15}} cc >{\columncolor{gray!15}} cc >{\columncolor{gray!15}} cc >{\columncolor{gray!15}} cc >{\columncolor{gray!15}} cc >{\columncolor{gray!15}} cc >{\columncolor{gray!15}} cc >{\columncolor{gray!15}} cc >{\columncolor{gray!15}} c}
        \toprule
        \multicolumn{2}{c}{Model}          & \multicolumn{2}{c}{DeepBooTS*} & \multicolumn{2}{c}{DeepBooTS} & \multicolumn{2}{c}{TimeXer} & \multicolumn{2}{c}{TimeMixer} & \multicolumn{2}{c}{iTransformer} & \multicolumn{2}{c}{Crossformer} & \multicolumn{2}{c}{PatchTST} & \multicolumn{2}{c}{DLinear} & \multicolumn{2}{c}{FEDformer} \\ 
                                  & Metric  & MSE       & MAE               & MSE       & MAE            & MSE              & MAE       & MSE              & MAE         & MSE             & MAE           & MSE             & MAE                  & MSE            & MAE          & MSE            & MAE            & MSE               & MAE                   \\ \toprule
        \multirow{5}{*}{\rotatebox{90}{ETT}}         & 96  & {\bf\color{red} 0.280}      & {\bf\color{red}0.336}   & \underline{\color{blue} 0.284}     & \underline{\color{blue}0.329}    &  {\bf\color{red} 0.280}  & 0.338   & 0.290  &  0.339   & 0.297            & 0.349            & 0.745             & 0.584            &   0.299        & 0.347          & 0.333          & 0.387          & 0.358             & 0.397                \\
                                      & 192 & {\bf\color{red} 0.332}     & {\bf\color{red}0.371}          & \underline{\color{blue}0.344}     & \underline{\color{blue}0.369}              & 0.348  & 0.377   & 0.350  & 0.373    & 0.380            & 0.400           & 0.877             & 0.656            & 0.365         & 0.404       & 0.477          & 0.476          & 0.429             & 0.439               \\
                                      & 336 & {\bf\color{red} 0.362}     & {\bf\color{red}0.392}        & \underline{\color{blue}0.389}     & \underline{\color{blue}0.397}                & 0.393  & 0.404   & 0.390  & 0.404    & 0.428            & 0.432            & 1.043             & 0.731            & 0.413       & 0.415          & 0.594          & 0.541          & 0.496             & 0.487            \\
                                      & 720 & {\bf\color{red} 0.410}     & {\bf\color{red}0.429}        & 0.433     & 0.433                                       & 0.430  & \underline{\color{blue}0.432}  & 0.439   & 0.438            & \underline{\color{blue}0.427} & 0.445            & 1.104             & 0.763            & 0.458        & 0.451        & 0.831          & 0.657          & 0.463             & 0.474             \\ \cline{2-20}
       \rowcolor{cyan!15}  \cellcolor{white}  & Avg & {\bf\color{red}0.346}     & {\bf\color{red}0.382}        & \underline{\color{blue}0.362}     & \underline{\color{blue}0.382}         & 0.365  & 0.388   & 0.367  & 0.389   & 0.383            & 0.407    & 0.942             & 0.684            & 0.384            & 0.404               & 0.559          & 0.515          & 0.437             & 0.449               \\ \toprule
        \multirow{5}{*}{\rotatebox{90}{Traffic}}     & 96  & {\bf\color{red}0.350}     & {\bf\color{red}0.250}          & \underline{\color{blue}0.386}     & \underline{\color{blue}0.258}  & 0.428  & 0.271   & 0.462  & 0.285       & 0.395            & 0.268           & 0.522             & 0.290            & 0.526           & 0.347          & 0.650          & 0.396          & 0.587             & 0.366              \\
                                      & 192 & {\bf\color{red}0.375}     & {\bf\color{red}0.261}          & \underline{\color{blue}0.398}     & \underline{\color{blue}0.263}                 & 0.448  & 0.282   & 0.473  & 0.296  & 0.417            & 0.276          & 0.530             & 0.293            & 0.522           & 0.332        & 0.598          & 0.370          & 0.604             & 0.373              \\
                                      & 336 & {\bf\color{red}0.381}     & {\bf\color{red}0.264}         & \underline{\color{blue}0.409}     & \underline{\color{blue}0.270}                & 0.473  & 0.289   & 0.498  & 0.296      & 0.433            & 0.283            & 0.558             & 0.305            & 0.517         & 0.334         & 0.605          & 0.373          & 0.621             & 0.383              \\
                                      & 720 & {\bf\color{red}0.386}     & {\bf\color{red}0.268}         & \underline{\color{blue}0.431}     & \underline{\color{blue}0.287}                & 0.516  & 0.307   & 0.506  & 0.313  & 0.467            & 0.302            & 0.589             & 0.328            & 0.552           & 0.352          & 0.645          & 0.394          & 0.626             & 0.382              \\ \cline{2-20}
        \rowcolor{cyan!15}  \cellcolor{white} & Avg & {\bf\color{red}0.373}     & {\bf\color{red}0.261}          & \underline{\color{blue}0.406}     & \underline{\color{blue}0.270}      & 0.466  &  0.287   & 0.484  & 0.297    & 0.428            & 0.282   & 0.550             & 0.304            &     0.529       &   0.341   & 0.625          & 0.383          & 0.610             & 0.376               \\ \toprule
        \multirow{5}{*}{\rotatebox{90}{ELC}} & 96  & {\bf\color{red}0.128}     & {\bf\color{red}0.223}        & \underline{\color{blue}0.135}     & \underline{\color{blue}0.229}        & 0.140  &  0.242   & 0.153  & 0.247  & 0.148            & 0.240         & 0.219             & 0.314            &   0.190       &   0.296        & 0.197          & 0.282          & 0.193             & 0.308                 \\
                                      & 192 & {\bf\color{red}0.148}     & {\bf\color{red}0.24}           & \underline{\color{blue}0.151}     & \underline{\color{blue}0.246}         & 0.157  &   0.256    & 0.166  & 0.256   & 0.162            & 0.253            & 0.231             & 0.322            &    0.199      & 0.304        & 0.196          & 0.285          & 0.201             & 0.315                \\
                                      & 336 & {\bf\color{red}0.164}     & {\bf\color{red}0.26}          & 0.172     & 0.268                                                      & 0.176  &  0.275  & 0.185  & 0.277   &  \underline{\color{blue}0.178}  &   \underline{\color{blue}0.269}            & 0.246             & 0.337            &  0.217         & 0.319       & 0.209          & 0.301          & 0.214             & 0.329            \\
                                      & 720 & {\bf\color{red}0.192}     & {\bf\color{red}0.284}        & \underline{\color{blue}0.204}     & \underline{\color{blue}0.294}      & 0.211  &  0.306   & 0.225  & 0.310       & 0.225            & 0.317           & 0.280             & 0.363            & 0.258          & 0.352        & 0.245          & 0.333          & 0.246             & 0.355            \\ \cline{2-20}
        \rowcolor{cyan!15}  \cellcolor{white}   & Avg & {\bf\color{red}0.158}     & {\bf\color{red}0.252}          & \underline{\color{blue}0.166}     & \underline{\color{blue}0.259}   & 0.171   &  0.270   & 0.182  & 0.272        & 0.178            & 0.270           & 0.244             & 0.334            & 0.216        & 0.318        & 0.212          & 0.300          & 0.214             & 0.327              \\ \toprule
        \multirow{5}{*}{\rotatebox{90}{Weather}}     & 96  & {\bf\color{red}0.150}     & {\bf\color{red}0.201}         & 0.160     & 0.206         & \underline{\color{blue}0.157}  &  \underline{\color{blue}0.205}   & 0.163   & 0.209   & 0.174            & 0.214            & 0.158             & 0.230            & 0.186         & 0.227        & 0.196          & 0.255          & 0.217             & 0.296              \\
                                      & 192 & {\bf\color{red}0.194}     & {\bf\color{red}0.244}         & 0.208     & 0.250                 & \underline{\color{blue}0.204}   & \underline{\color{blue}0.247}   & 0.208  & 0.250     & 0.221  & 0.254            & 0.206             & 0.277            & 0.234          & 0.265      & 0.237          & 0.296          & 0.276             & 0.336                \\
                                      & 336 & {\bf\color{red}0.245}     & {\bf\color{red}0.282}          & 0.269     & 0.296                & 0.261  & 0.290   & \underline{\color{blue}0.251}  & \underline{\color{blue}0.287}   & 0.278 & 0.296            & 0.272             & 0.335            & 0.284          & 0.301       & 0.283          & 0.335          & 0.339             & 0.380             \\
                                      & 720 & {\bf\color{red}0.320}     & {\bf\color{red}0.336}         & 0.343     & \underline{\color{blue}0.334}                & 0.340  & 0.341   &  \underline{\color{blue}0.339}  &  0.341      & 0.358    & 0.349            & 0.398             & 0.418            & 0.356           & 0.349        & 0.345          & 0.381          & 0.403             & 0.428              \\ \cline{2-20}
        \rowcolor{cyan!15}  \cellcolor{white}  & Avg & {\bf\color{red}0.227}     & {\bf\color{red}0.266}         & 0.245     & \underline{\color{blue}0.271}   & 0.241  & \underline{\color{blue}0.271}    & \underline{\color{blue}0.240}  & \underline{\color{blue}0.271}         & 0.258            & 0.279            & 0.259             & 0.315            & 0.265          & 0.285       & 0.265          & 0.317          & 0.309             & 0.360                \\ \toprule
        \multirow{5}{*}{\rotatebox{90}{Solar}}       & 96  & {\bf\color{red}0.181}     & {\bf\color{red}0.222}           & 0.192     & \underline{\color{blue}0.222}    & 0.229  & 0.273   & \underline{\color{blue}0.189}   & 0.259       & 0.203            & 0.237        & 0.310             & 0.331            & 0.265        & 0.323       & 0.290          & 0.378          & 0.242             & 0.342              \\
                                      & 192 & {\bf\color{red}0.198}     & {\bf\color{red}0.239}        & 0.230      & 0.251                        & 0.265  & 0.297   & \underline{\color{blue}0.222}  & 0.283             & 0.233            & 0.261             & 0.734             & 0.725            & 0.288        & 0.332         & 0.320          & 0.398          & 0.285             & 0.380                 \\
                                      & 336 & {\bf\color{red}0.202}     & {\bf\color{red}0.245}        & 0.243     & \underline{\color{blue}0.263}     & 0.293  & 0.312   & \underline{\color{blue}0.231}  & 0.292    & 0.248            & 0.273        & 0.750             & 0.735            & 0.301        & 0.339    & 0.353          & 0.415          & 0.282             & 0.376         \\
                                      & 720 & {\bf\color{red}0.206}    & {\bf\color{red}0.252}         & 0.243     & \underline{\color{blue}0.265}                  & 0.302  & 0.318   & \underline{\color{blue}0.223}  & 0.285    & 0.249    & 0.275          & 0.769             & 0.765            & 0.295         & 0.336        & 0.356          & 0.413          & 0.357             & 0.427           \\ \cline{2-20}
        \rowcolor{cyan!15}  \cellcolor{white} & Avg & {\bf\color{red}0.197}     & {\bf\color{red}0.240}        & 0.227     & \underline{\color{blue}0.250}    & 0.272  & 0.300   &  \underline{\color{blue}0.216}   &  0.280      & 0.233            & 0.262            & 0.641             & 0.639            & 0.287          & 0.333         & 0.330          & 0.401          & 0.291             & 0.381         \\ \toprule
        \multirow{5}{*}{\rotatebox{90}{PEMS}}        & 12  & {\bf\color{red}0.057}     & {\bf\color{red}0.157}       & 0.067     & \underline{\color{blue}0.171}       & 0.092  & 0.193  & 0.078  & 0.187     & 0.071            & 0.174       & 0.090             & 0.203            & 0.099        & 0.216       & 0.122          & 0.243          & 0.126             & 0.251      \\
                                      & 24  & {\bf\color{red}0.070}     & {\bf\color{red}0.173}         & 0.093     & 0.203         & 0.111  & 0.218   & 0.108  & 0.216  & 0.093            & 0.201            & 0.121             & 0.240            & 0.142         & 0.259       & 0.201          & 0.317          & 0.149             & 0.275         \\
                                      & 36  & {\bf\color{red}0.083}     & {\bf\color{red}0.186}         & 0.125     & 0.237       & 0.135  & 0.257  & 0.151  & 0.254    & \underline{\color{blue}0.125}           & \underline{\color{blue}0.236}         & 0.202             & 0.317            & 0.211          & 0.319        & 0.333          & 0.425          & 0.227             & 0.348              \\
                                      & 48  & {\bf\color{red}0.091}     & {\bf\color{red}0.195}          & \underline{\color{blue}0.151}     & \underline{\color{blue}0.262}   & 0.226  & 0.232   & 0.213  & 0.311          & 0.160            & 0.270      & 0.262             & 0.367            & 0.269         &  0.370    & 0.457          & 0.515          & 0.348             & 0.434             \\  \cline{2-20}
        \rowcolor{cyan!15}  \cellcolor{white}  & Avg & {\bf\color{red}0.075}     & {\bf\color{red}0.178}     & \underline{\color{blue}0.109}     & \underline{\color{blue}0.218}     & 0.141  & 0.225   & 0.138  & 0.242     & 0.113            & 0.221    & 0.169             & 0.281            & 0.180         & 0.291    & 0.278          & 0.375          & 0.213             & 0.327           \\ \toprule
        \multicolumn{2}{c}{{\bf \color{red}$1^{st}$} or \underline{\color{blue}$2^{st}$}} & {\bf\color{red} 30}  & {\bf\color{red}30}   & \underline{\color{blue}15}   & \underline{\color{blue}22}   &  \underline{\color{blue}3}  &  \underline{\color{blue}4}   & \underline{\color{blue}8}  & \underline{\color{blue}2}      & \underline{\color{blue}3}   & \underline{\color{blue}2}  & \underline{\color{blue}3}    & 0    & 0  & 0  & \underline{\color{blue}1}  & 0    & 0       & 0        \\  \bottomrule
        \end{tabular}
  \end{table*}

\section{Full Univariate TS Forecasting Results}
\label{app_uts}

The full results for univariate TS forecasting are presented in Table \ref{tb4}. 
As other models, e.g., iTransformer \cite{liu2023itransformer}, PatchTST \cite{nie2022time} and Crossformer \cite{zhang2022crossformer} do not offer performance information for all prediction lengths, we compare our method with those that provide comprehensive performance analysis, including Periodformer \cite{liang2023does}, FEDformer \cite{zhou2022fedformer}, Autoformer \cite{wu2021autoformer}, Informer \cite{Zhou2021Informer}, LogTrans \cite{li2019enhancing} and Reformer \cite{Kitaev2020Reformer}.
Despite Periodformer being a model that determines optimal hyperparameters through search, the proposed method outperforms benchmark approaches by achieving the highest count of leading terms across various prediction lengths.
This reaffirms the effectiveness of DeepBooTS.

\begin{table*}[!ht]
  \centering
  \caption{Univariate TS forecasting results on benchmark datasets.}
  \label{tb4}
  \fontsize{9pt}{9pt}\selectfont
  \setlength{\tabcolsep}{1mm}
  \begin{threeparttable}
  \begin{tabular}{ccc >{\columncolor{gray!15}} cc >{\columncolor{gray!15}} cc >{\columncolor{gray!15}} cc >{\columncolor{gray!15}} cc >{\columncolor{gray!15}} cc >{\columncolor{gray!15}} cc >{\columncolor{gray!15}} c}
    \toprule
    \multicolumn{2}{c}{Model}          & \multicolumn{2}{c}{DeepBooTS} & \multicolumn{2}{c}{Periodformer} & \multicolumn{2}{c}{FEDformer} & \multicolumn{2}{c}{Autoformer} & \multicolumn{2}{c}{Informer} & \multicolumn{2}{c}{LogTrans} & \multicolumn{2}{c}{Reformer} \\
    \toprule
                              & Length & MSE              & MAE             & MSE               & MAE               & MSE             & MAE            & MSE             & MAE             & MSE            & MAE            & MSE            & MAE            & MSE            & MAE            \\ \toprule
    \multirow{5}{*}{\rotatebox{90}{ETTh1}}       & 96  & {\bf\color{red}0.055}            & {\bf\color{red}0.177}           & \underline{\color{blue} 0.068}             & \underline{\color{blue} 0.203}             & 0.079           & 0.215          & 0.071           & 0.206           & 0.193          & 0.377          & 0.283          & 0.468          & 0.532          & 0.569          \\
                                & 192 & {\bf\color{red}0.072}            & {\bf\color{red}0.204}           & \underline{\color{blue} 0.088}             & \underline{\color{blue} 0.228}             & 0.104           & 0.245          & 0.114           & 0.262           & 0.217          & 0.395          & 0.234          & 0.409          & 0.568          & 0.575          \\
                                & 336 & {\bf\color{red}0.08}             & {\bf\color{red}0.219}           & \underline{\color{blue} 0.105}             & \underline{\color{blue} 0.256}             & 0.119           & 0.270           & 0.107           & 0.258           & 0.202          & 0.381          & 0.386          & 0.546          & 0.635          & 0.589          \\
                                & 720 & {\bf\color{red}0.079}            & {\bf\color{red}0.224}           & \underline{\color{blue} 0.109}             & \underline{\color{blue} 0.262}             & 0.142           & 0.299          & 0.126           & 0.283           & 0.183          & 0.355          & 0.475          & 0.628          & 0.762          & 0.666          \\ \cline{2-16}
    \rowcolor{cyan!15}  \cellcolor{white} & Avg & {\bf\color{red}0.072}            & {\bf\color{red}0.206}           & \underline{\color{blue} 0.093}             & \underline{\color{blue} 0.237}             & 0.111           & 0.257          & 0.105           & 0.252           & 0.199          & 0.377          & 0.345          & 0.513          & 0.624          & 0.600          \\ \toprule
    \multirow{5}{*}{\rotatebox{90}{ETTh2}}       & 96  & 0.129            & 0.275           & {\bf\color{red} 0.125}             & \underline{\color{blue}0.272}             & \underline{\color{blue} 0.128}           & {\bf\color{red} 0.271}          & 0.153           & 0.306           & 0.213          & 0.373          & 0.217          & 0.379          & 1.411          & 0.838          \\ 
                                & 192 & \underline{\color{blue} 0.178}            & \underline{\color{blue} 0.329}           & {\bf\color{red} 0.175}             & {\bf\color{red}0.329}             & 0.185           & 0.33           & 0.204           & 0.351           & 0.227          & 0.387          & 0.281          & 0.429          & 5.658          & 1.671          \\
                                & 336 & {\bf\color{red}0.211}            & {\bf\color{red}0.365}           & \underline{\color{blue} 0.219}             & \underline{\color{blue} 0.372}             & 0.231           & 0.378          & 0.246           & 0.389           & 0.242          & 0.401          & 0.293          & 0.437          & 4.777          & 1.582          \\
                                & 720 & \underline{\color{blue}0.220}            & {\bf\color{red}0.377}           &  0.249             &  0.400              & 0.278           & 0.42           & 0.268           & 0.409           & 0.291          & 0.439          & {\bf\color{red}0.218}          & \underline{\color{blue} 0.387}          & 2.042          & 1.039          \\ \cline{2-16}
   \rowcolor{cyan!15}  \cellcolor{white} & Avg & {\bf\color{red}0.185}            & {\bf\color{red}0.337}           & \underline{\color{blue} 0.192}             & \underline{\color{blue} 0.343}             & 0.206           & 0.350          & 0.218           & 0.364           & 0.243          & 0.400          & 0.252          & 0.408          & 3.472          & 1.283          \\ \toprule
    \multirow{5}{*}{\rotatebox{90}{ETTm1}}       & 96  & {\bf\color{red}0.029}            & {\bf\color{red}0.126}           & \underline{\color{blue} 0.033}             & \underline{\color{blue} 0.139}             & 0.033           & 0.140           & 0.056           & 0.183           & 0.109          & 0.277          & 0.049          & 0.171          & 0.296          & 0.355          \\ 
                                & 192 & {\bf\color{red}0.044}            & {\bf\color{red}0.158}           & \underline{\color{blue}0.052}             & \underline{\color{blue} 0.177}             &  0.058           & 0.186          & 0.081           & 0.216           & 0.151          & 0.310           & 0.157          & 0.317          & 0.429          & 0.474          \\
                                & 336 & {\bf\color{red}0.057}            & {\bf\color{red}0.185}           & \underline{\color{blue} 0.070}              & 0.267             & 0.084           &  0.231          & 0.076           & \underline{\color{blue}0.218}           & 0.427          & 0.591          & 0.289          & 0.459          & 0.585          & 0.583          \\
                                & 720 & {\bf\color{red}0.080}            & {\bf\color{red}0.218}           & \underline{\color{blue} 0.081}             & \underline{\color{blue} 0.221}             & 0.102           & 0.250           & 0.110            & 0.267           & 0.438          & 0.586          & 0.430           & 0.579          & 0.782          & 0.73           \\ \cline{2-16}
     \rowcolor{cyan!15}  \cellcolor{white} & Avg & {\bf\color{red}0.052}            & {\bf\color{red}0.172}           & \underline{\color{blue} 0.059}             & \underline{\color{blue} 0.201}             & 0.069          & 0.202          & 0.081           & 0.221           & 0.281          & 0.441          & 0.231          & 0.382          & 0.523          & 0.536          \\ \toprule
    \multirow{5}{*}{\rotatebox{90}{ETTm2}}       & 96  & 0.064            & {\bf\color{red} 0.180}          & {\bf\color{red} 0.060}             & \underline{\color{blue}0.182}             & \underline{\color{blue}0.063}           & 0.189          & 0.065           & 0.189           & 0.08           & 0.217          & 0.075          & 0.208          & 0.077          & 0.214          \\ 
                                & 192 & {\bf\color{red}0.099}            & {\bf\color{red}0.233}           & \underline{\color{blue}0.099}             & \underline{\color{blue}0.236}             & 0.110           & 0.252          & 0.118           & 0.256           & 0.112          & 0.259          & 0.129          & 0.275          & 0.138          & 0.290           \\
                                & 336 & {\bf\color{red}0.129}            & {\bf\color{red}0.273}           & \underline{\color{blue}0.129}             & \underline{\color{blue}0.275}             & 0.147           & 0.301          & 0.154           & 0.305           & 0.166          & 0.314          & 0.154          & 0.302          & 0.160          & 0.313          \\
                                & 720 & 0.180    & 0.329           & \underline{\color{blue}0.170}             & {\bf\color{red}0.317}             & 0.219           & 0.368          & 0.182           & 0.335           & 0.228      & 0.380           & {\bf\color{red}0.160}          & \underline{\color{blue}0.322}         & 0.168          & 0.334          \\ \cline{2-16}
    \rowcolor{cyan!15}  \cellcolor{white}& Avg & \underline{\color{blue}0.118}            & \underline{\color{blue}0.254}           & {\bf\color{red}0.115}             & {\bf\color{red}0.253}             & 0.135           & 0.278          & 0.130           & 0.271           & 0.147          & 0.293          & 0.130          & 0.277          & 0.136          & 0.288          \\ \toprule
    \multirow{5}{*}{\rotatebox{90}{Traffic}}     & 96  & {\bf\color{red}0.127}            & {\bf\color{red}0.202}           & \underline{\color{blue}0.143}             & \underline{\color{blue}0.222}             & 0.170           & 0.263          & 0.246           & 0.346           & 0.257          & 0.353          & 0.226          & 0.317          & 0.313          & 0.383          \\ 
                                & 192 & {\bf\color{red}0.135}            & {\bf\color{red}0.211}           & \underline{\color{blue}0.146}             & \underline{\color{blue}0.227}             & 0.173           & 0.265          & 0.266           & 0.37            & 0.299          & 0.376          & 0.314          & 0.408          & 0.386          & 0.453          \\
                                & 336 & {\bf\color{red}0.130}             & {\bf\color{red}0.215}           & \underline{\color{blue}0.147}             & \underline{\color{blue}0.231}             & 0.178           & 0.266          & 0.263           & 0.371           & 0.312          & 0.387          & 0.387          & 0.453          & 0.423          & 0.468          \\
                                & 720 & {\bf\color{red}0.135}            & {\bf\color{red}0.218}           & \underline{\color{blue}0.164}             & \underline{\color{blue}0.252}             & 0.187           & 0.286          & 0.269           & 0.372           & 0.366          & 0.436          & 0.437          & 0.491          & 0.378          & 0.433          \\ \cline{2-16}
     \rowcolor{cyan!15}  \cellcolor{white} & Avg & {\bf\color{red}0.132}            & {\bf\color{red}0.212}           & \underline{\color{blue}0.150}             & \underline{\color{blue}0.233}             & 0.177           & 0.270          & 0.261           & 0.365           & 0.309          & 0.388          & 0.341          & 0.417          & 0.375          & 0.434          \\ \toprule
    \multirow{5}{*}{\rotatebox{90}{ELC}} & 96  & \underline{\color{blue}0.249}            & \underline{\color{blue}0.358}           & {\bf\color{red}0.236}             & {\bf\color{red}0.349}             & 0.262           & 0.378          & 0.341           & 0.438           & 0.258          & 0.367          & 0.288          & 0.393          & 0.275          & 0.379          \\
                                & 192 & 0.286            & \underline{\color{blue}0.379}           & {\bf\color{red}0.277}             & {\bf\color{red}0.369}             & 0.316           & 0.410          & 0.345           & 0.428           & \underline{\color{blue}0.285}          & 0.388          & 0.432          & 0.483          & 0.304          & 0.402          \\
                                & 336 & 0.337           & \underline{\color{blue}0.413}           & {\bf\color{red}0.324}             & {\bf\color{red}0.400}               & 0.361           & 0.445          & 0.406           & 0.470          & \underline{\color{blue}0.336}          & 0.423          & 0.430          & 0.483          & 0.37           & 0.448          \\
                                & 720 & \underline{\color{blue}0.385}            & \underline{\color{blue}0.454}           & {\bf\color{red}0.353}             & {\bf\color{red}0.437}             & 0.448           & 0.501          & 0.565           & 0.581           & 0.607          & 0.599          & 0.491          & 0.531          & 0.46           & 0.511          \\ \cline{2-16}
    \rowcolor{cyan!15}  \cellcolor{white}& Avg & \underline{\color{blue}0.314}            & \underline{\color{blue}0.401}           & {\bf\color{red}0.298}             & {\bf\color{red}0.389}             & 0.347           & 0.434          & 0.414           & 0.479           & 0.372          & 0.444          & 0.410          & 0.473          & 0.352          & 0.435          \\ \toprule
    \multirow{5}{*}{\rotatebox{90}{Weather}}     & 96  & {\bf\color{red}0.0012}           & {\bf\color{red}0.0263}          & \underline{\color{blue}0.0015}            & \underline{\color{blue}0.0300}              & 0.0035          & 0.046          & 0.0110          & 0.081           & 0.004          & 0.044          & 0.0046         & 0.052          & 0.012          & 0.087          \\
                                & 192 & {\bf\color{red}0.0014}           & {\bf\color{red}0.0283}          & \underline{\color{blue}0.0015}            & \underline{\color{blue}0.0307}            & 0.0054          & 0.059          & 0.0075          & 0.067           & 0.002          & 0.040          & 0.006          & 0.060          & 0.0098         & 0.044          \\
                                & 336 & {\bf\color{red}0.0015}           & {\bf\color{red}0.0294}          & \underline{\color{blue}0.0017}            & \underline{\color{blue}0.0313}            & 0.008           & 0.072          & 0.0063          & 0.062           & 0.004          & 0.049          & 0.006          & 0.054          & 0.013          & 0.100          \\
                                & 720 & {\bf\color{red}0.002}            & {\bf\color{red}0.0333}         & \underline{\color{blue}0.0020}             & \underline{\color{blue}0.0348}            & 0.015           & 0.091          & 0.0085          & 0.070           & 0.003          & 0.042          & 0.007          & 0.059          & 0.011          & 0.083          \\ \cline{2-16}
     \rowcolor{cyan!15}  \cellcolor{white}& Avg & {\bf\color{red}0.0015}           & {\bf\color{red}0.0293}          & \underline{\color{blue}0.0017}            & \underline{\color{blue}0.0317}            & 0.008           & 0.067          & 0.0083          & 0.0700          & 0.0033         & 0.0438         & 0.0059         & 0.0563         & 0.0115         & 0.0785         \\ \toprule
    \multirow{5}{*}{\rotatebox{90}{\ \ Exchange}}    & 96  & \underline{\color{blue}0.096}            & \underline{\color{blue}0.226}           & {\bf\color{red}0.092}             & {\bf\color{red}0.226}             & 0.131           & 0.284          & 0.241           & 0.387           & 1.327          & 0.944          & 0.237          & 0.377          & 0.298          & 0.444          \\
                                & 192 & \underline{\color{blue}0.200}              & {\bf\color{red}0.332}           & {\bf\color{red}0.198}             &  \underline{\color{blue}0.341}             &  0.277           & 0.420          & 0.300           & 0.369           & 1.258          & 0.924          & 0.738          & 0.619          & 0.777          & 0.719          \\
                                & 336 & \underline{\color{blue}0.400}              & \underline{\color{blue}0.473}           & {\bf\color{red}0.370}             & {\bf\color{red}0.471}             & 0.426           & 0.511          & 0.509           & 0.524           & 2.179          & 1.296          & 2.018          & 1.0700         & 1.833          & 1.128          \\
                                & 720 & \underline{\color{blue}1.020}             & \underline{\color{blue}0.779}           & {\bf\color{red}0.753}             & {\bf\color{red}0.696}             & 1.162           & 0.832          & 1.260           & 0.867           & 1.28           & 0.953          & 2.405          & 1.175          & 1.203          & 0.956          \\ \cline{2-16}
    \rowcolor{cyan!15}  \cellcolor{white} & Avg & \underline{\color{blue}0.429}            & \underline{\color{blue}0.453}           & {\bf\color{red}0.353}             & {\bf\color{red}0.434}             & 0.499           & 0.512          & 0.578           & 0.537           & 1.511          & 1.029          & 1.350          & 0.810          & 1.028          & 0.812          \\ \toprule
    \multicolumn{2}{c}{ $1^{\text{st}}$ Count} & {\bf\color{red} 24}           & {\bf\color{red}27}           & \underline{\color{blue}14}                   & \underline{\color{blue}12}       & 0               & 1               & 0             & 0               & 0               & 0               & 2             & 0             & 0               & 0             \\
    \bottomrule
    \end{tabular}
  \begin{tablenotes}
    \item[*] ${\diamondsuit}$ denotes the maximum search range of the input length.
  \end{tablenotes}
  \end{threeparttable}
\end{table*}

\section{Comparative Analysis}
\label{app_cmp_analysis}

Notably, several pioneering models have also achieved competitive performance on certain datasets under particular settings.
For instance, Informer, considered a groundbreaking model in long-term TS forecasting, demonstrates advanced performance on the Solar-Energy dataset with input-96-predict-192 and -720 settings.
This is due to the substantial presence of zero values on each column attribute of the Solor-Energy dataset.
This renders the KL-divergence based ProbSparse Attention, as adopted in Informer, highly effective on this sparse dataset.
Additionally, linear-based methods (e.g., DLinear) have demonstrated promising results on the Weather dataset with input-336-predict-720 setting, while convolution-based methods (e.g., SCINet) have yielded favorable results on the PEMS dataset with input-168-predict-192 setting. 
This phenomenon can be ascribed to a twofold interplay of factors. Previously, the diversity of input settings exerts a direct influence on model generalization. Secondarily, other models exhibit a propensity to overfit non-stationary TS characterized by aperiodic fluctuations.
Remarkably, DeepBooTS adeptly mitigates both concept drift and underfitting challenges in multivariate TS forecasting, thereby enhancing its overall performance.
Particularly on datasets with numerous attributes, e.g., Traffic and Solor-Energy, DeepBooTS achieves superior performance by feeding the learned meaningful patterns to the output layer at each block.

\section{Interpretability of DeepBooTS}
\label{app_intp}

Fig. \ref{fig_vis_intp} illustrates that DeepBooTS learns a fully interpretable, hierarchical decomposition of the target series by assigning distinct temporal patterns to each block. In the three‐block, low‐dimensional setting, Block 1 captures the dominant seasonal cycle, Block 2 models the residual phase and amplitude errors, and Block 3 provides a constant bias correction; increasing the embedding dimension sharpens the periodic waveforms without altering this structure. When extended to five blocks at high dimensionality, Blocks 1 and 2 isolate the two strongest periodicities, Blocks 3 and 4 extract mid‐frequency fluctuations, and Block 5 again serves as a global baseline. Summing these block outputs reproduces the final forecast almost exactly, demonstrating DeepBooTS's ability to decompose complex series into interpretable, scale‐specific components and thereby facilitate both efficient learning and transparent model analysis.

\begin{figure*}[t]
  \centerline{\includegraphics[width=0.99\textwidth]{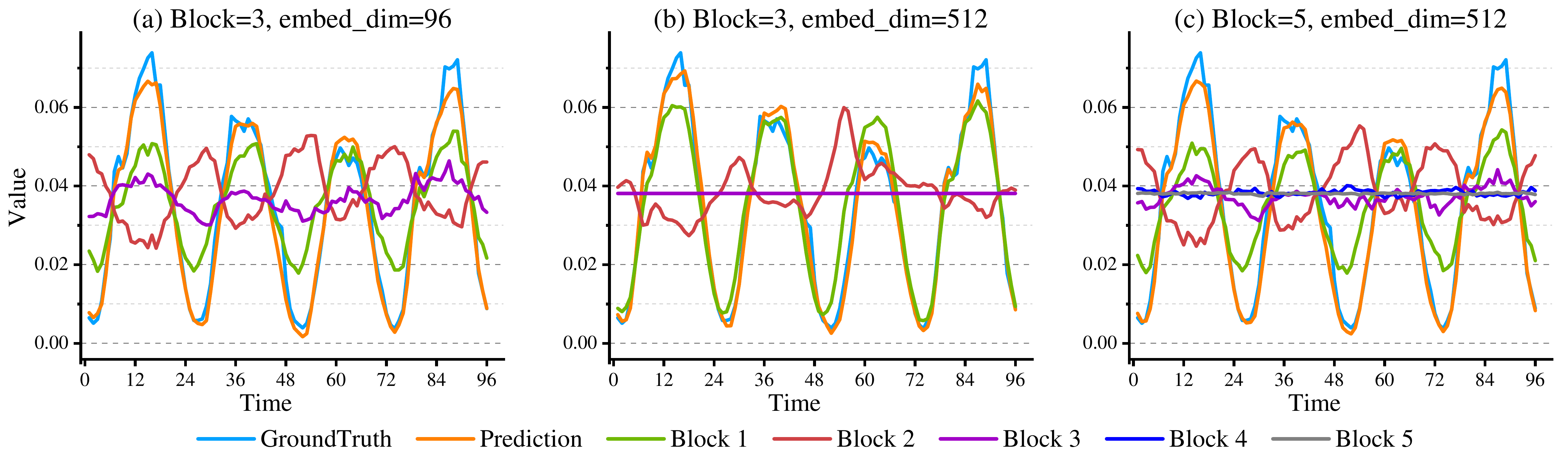}}
  \caption{The visualization of each block in DeepBooTS.} %
  \label{fig_vis_intp}
\end{figure*}

\section{Evaluation Metrics}
\label{secsub_metrics}

This paper uses a variety of metrics, including MSE (Mean Square Error), MAE (Mean Absolute Error), RMSP (Root Median Square Percent), MAPE (Mean Absolute Percentage Error), sMAPE (symmetric MAPE), MASE (Mean Absolute Scaled Error) and Quantile loss. 
All evaluation metrics adopted in this paper are as follows:

\begin{align}
  \text{MSE}(Y, \hat{Y}) &= \frac{1}{N} \sum_{i=1}^{N} (y_i - \hat{y}_i)^2 , \\
  \text{MAE}(Y, \hat{Y}) &= \frac{1}{N} \sum_{i=1}^{N} |y_i - \hat{y}_i| , \\
  \text{RMSP}(Y, \hat{Y}) &= \sqrt{ \text{Median} \left((\frac{y_i - \hat{y}_i}{y_i})^2 \right) } ,\\
  \text{MAPE}(Y, \hat{Y}) &= \frac{1}{N} \sum_{i=1}^{N}  \frac{|y_i - \hat{y}_i|}{|y_i|} ,\\
  \text{sMAPE}(Y, \hat{Y}) &= \frac{2}{N} \sum_{i=1}^{N} \frac{|y_i - \hat{y}_i|}{|y_i| + |\hat{y}_i|} ,\\
  \text{MASE}(Y, \hat{Y}) &= \frac{1}{N} \sum_{i=1}^{N} \frac{|y_i - \hat{y}_i|}{\frac{1}{N-m} \sum_{i=m+1}^{N} |y_i - y_{i-m}|} ,\\
  \text{QuantileLoss}(Y, \hat{Y}, q) &= \frac{1}{N} \sum_{i=1}^{N} \mathbb{I}_{\hat{y}_i \ge y_i} (1- q) |y_i - \hat{y}_i| + \mathbb{I}_{\hat{y}_i < y_i}q |y_i - \hat{y}_i|.
\end{align}

\section{Full TS forecasting on Monash's Repository}
\label{subsec_mts_monash}

The full results for mulvariate and univariate TS forecasting on 4 commonly used datasets and 7 Monash TS datasets are presented in Table \ref{tb_mts_monash_full}. 
Due to different characteristics of the dataset, the output lengths Y1, Y2, Y3 and Y4,  are slightly different. Specifically, for Traffic, Electricity, Solar, OiK\_Weather, Sunspot, Saugeenday and M4, the input length is 96, and the output lengths are 96, 192, 336, and 720, respectively. 
For ILI dataset, the input length is 36, and the output lengths are 24, 36, 48, and 60, respectively. 
For NN5 and Rideshare datasets, the input length is 48, and the output lengths are 12, 24, 36, and 48, respectively. 
It is evident that the proposed method outperforms benchmark approaches by achieving the highest count of leading terms across various prediction lengths.
This reaffirms the effectiveness of DeepBooTS.

\begin{table}[!ht]
  \centering
  \caption{Mulvariate and univariate forecasting results with diverse metrics on Monash TS datasets.}
  \label{tb_mts_monash_full}
  \resizebox{\textwidth}{!}
  {
   \Huge
  \begin{tabular}{c|c|cccc >{\columncolor{cyan!15}} c| >{\columncolor{green!10}} c| cccc >{\columncolor{cyan!15}} c| cccc >{\columncolor{cyan!15}} c| cccc >{\columncolor{cyan!15}} c| cccc >{\columncolor{cyan!15}} c}
    \toprule
  \multicolumn{2}{c|}{Mulvariate}              & \multicolumn{6}{c|}{DeepBooTS}    & \multicolumn{5}{c|}{iTransformer}             & \multicolumn{5}{c|}{DLinear}       & \multicolumn{5}{c|}{Autoformer}      & \multicolumn{5}{c}{Informer}       \\ \toprule
  Dataset                       & Metric & Y1    & Y2    & Y3    & Y4    & Avg   & IMP     & Y1    & Y2    & Y3    & Y4    & Avg   & Y1    & Y2    & Y3    & Y4    & Avg   & Y1     & Y2    & Y3    & Y4    & Avg   & Y1    & Y2    & Y3    & Y4    & Avg   \\ \toprule
  \multirow{8}{*}{\rotatebox{90}{Traffic}}      
                                & MSE    & 0.386 & 0.398 & 0.409 & 0.431 & {\bf\color{red}0.406} & 5.14\%  & 0.395 & 0.417 & 0.433 & 0.467 & \underline{\color{blue}0.428} & 0.650 & 0.598 & 0.605 & 0.645 & 0.625 & 0.613  & 0.616 & 0.622 & 0.660 & 0.628 & 0.719 & 0.696 & 0.777 & 0.864 & 0.764 \\
                                & MAE    & 0.258 & 0.263 & 0.270 & 0.287 & {\bf\color{red}0.270} & 4.26\%  & 0.268 & 0.276 & 0.283 & 0.302 & \underline{\color{blue}0.282} & 0.396 & 0.370 & 0.373 & 0.394 & 0.383 & 0.388  & 0.382 & 0.337 & 0.408 & 0.379 & 0.391 & 0.379 & 0.420 & 0.472 & 0.416 \\
                                & RMSP & 0.205 & 0.200 & 0.207 & 0.231 & {\bf\color{red}0.211} & 6.64\%  & 0.214 & 0.219 & 0.225 & 0.247 & \underline{\color{blue}0.226} & 0.350 & 0.317 & 0.322 & 0.356 & 0.336 & 0.437  & 0.409 & 0.384 & 0.432 & 0.416 & 0.423 & 0.416 & 0.522 & 0.715 & 0.519 \\
                                & MAPE   & 2.657 & 2.694 & 2.715 & 2.797 & {\bf\color{red}2.716} & 8.24\%  & 2.885 & 2.955 & 2.942 & 3.059 & \underline{\color{blue}2.960} & 5.018 & 4.477 & 4.290 & 4.196 & 4.495 & 4.377  & 4.198 & 4.254 & 4.288 & 4.279 & 4.760 & 4.677 & 6.115 & 7.016 & 5.642 \\
                                & sMAPE  & 0.479 & 0.476 & 0.485 & 0.514 & {\bf\color{red}0.489} & 2.59\%  & 0.485 & 0.494 & 0.501 & 0.529 & \underline{\color{blue}0.502} & 0.662 & 0.625 & 0.629 & 0.663 & 0.645 & 0.721  & 0.687 & 0.657 & 0.718 & 0.696 & 0.718 & 0.712 & 0.810 & 0.978 & 0.805 \\
                                & MASE   & 0.256 & 0.257 & 0.264 & 0.285 & {\bf\color{red}0.266} & 1.12\%  & 0.255 & 0.263 & 0.268 & 0.288 & \underline{\color{blue}0.269} & 0.397 & 0.365 & 0.366 & 0.390 & 0.380 & 0.405  & 0.390 & 0.364 & 0.401 & 0.390 & 0.415 & 0.406 & 0.468 & 0.595 & 0.471 \\
                                & Q25    & 0.255 & 0.264 & 0.268 & 0.281 & {\bf\color{red}0.267} & 4.64\%  & 0.271 & 0.279 & 0.279 & 0.290 & \underline{\color{blue}0.280} & 0.380 & 0.357 & 0.357 & 0.373 & 0.367 & 0.440  & 0.414 & 0.395 & 0.419 & 0.417 & 0.397 & 0.372 & 0.461 & 0.614 & 0.461 \\
                                & Q75    & 0.270 & 0.270 & 0.280 & 0.304 & {\bf\color{red}0.281} & 1.40\%  & 0.267 & 0.276 & 0.286 & 0.311 & \underline{\color{blue}0.285} & 0.412 & 0.383 & 0.390 & 0.417 & 0.401 & 0.391  & 0.387 & 0.374 & 0.416 & 0.392 & 0.456 & 0.465 & 0.536 & 0.621 & 0.520 \\ \toprule
  \multirow{8}{*}{\rotatebox{90}{Electricity}}
                                & MSE    & 0.143 & 0.162 & 0.179 & 0.204 & {\bf\color{red}0.172} & 3.37\%  & 0.148 & 0.162 & 0.178 & 0.225 & \underline{\color{blue}0.178} & 0.197 & 0.196 & 0.209 & 0.245 & 0.212 & 0.201  & 0.222 & 0.231 & 0.254 & 0.227 & 0.274 & 0.296 & 0.300 & 0.373 & 0.311 \\
                                & MAE    & 0.235 & 0.253 & 0.271 & 0.294 & {\bf\color{red}0.263} & 2.59\%  & 0.240 & 0.253 & 0.269 & 0.317 & \underline{\color{blue}0.270} & 0.282 & 0.285 & 0.301 & 0.333 & 0.300 & 0.317  & 0.334 & 0.338 & 0.361 & 0.338 & 0.368 & 0.386 & 0.394 & 0.439 & 0.397 \\
                                & RMSP & 0.215 & 0.230 & 0.253 & 0.287   & {\bf\color{red}0.246} & 2.38\%  & 0.223 & 0.239 & 0.256 & 0.290 & \underline{\color{blue}0.252} & 0.310 & 0.316 & 0.334 & 0.370 & 0.333 & 0.324  & 0.346 & 0.367 & 0.374 & 0.353 & 0.429 & 0.796 & 0.487 & 0.680 & 0.598 \\
                                & MAPE   & 2.296 & 2.431 & 2.699 & 3.194 & {\bf\color{red}2.655} & 2.10\%  & 2.493 & 2.732 & 2.713 & 2.910 & \underline{\color{blue}2.712} & 2.695 & 2.715 & 2.667 & 2.760 & 2.709 & 3.515  & 3.309 & 3.626 & 3.401 & 3.463 & 3.789 & 3.354 & 2.354 & 3.826 & 3.331 \\
                                & sMAPE  & 0.466 & 0.487 & 0.514 & 0.558 & {\bf\color{red}0.506} & 1.75\%  & 0.479 & 0.498 & 0.519 & 0.564 & \underline{\color{blue}0.515} & 0.598 & 0.601 & 0.621 & 0.660 & 0.620 & 0.584  & 0.619 & 0.639 & 0.642 & 0.621 & 0.711 & 1.148 & 0.786 & 0.988 & 0.908 \\
                                & MASE   & 0.257 & 0.276 & 0.293 & 0.321 & {\bf\color{red}0.287} & 3.37\%  & 0.267 & 0.284 & 0.298 & 0.340 & \underline{\color{blue}0.297} & 0.357 & 0.359 & 0.374 & 0.405 & 0.374 & 0.329  & 0.359 & 0.380 & 0.379 & 0.362 & 0.414 & 0.832 & 0.445 & 0.614 & 0.576 \\
                                & Q25    & 0.232 & 0.239 & 0.270 & 0.321 & 0.266 & -1.53\% & 0.232 & 0.255 & 0.269 & 0.290 & {\bf\color{red}0.262} & 0.314 & 0.317 & 0.327 & 0.356 & 0.329 & 0.349  & 0.342 & 0.378 & 0.384 & 0.363 & 0.406 & 0.665 & 0.421 & 0.590 & 0.521 \\
                                & Q75    & 0.230 & 0.254 & 0.261 & 0.276 & {\bf\color{red}0.255} & 7.27\%  & 0.250 & 0.258 & 0.272 & 0.318 & \underline{\color{blue}0.275} & 0.319 & 0.323 & 0.339 & 0.370 & 0.338 & 0.285  & 0.334 & 0.341 & 0.338 & 0.325 & 0.411 & 0.663 & 0.420 & 0.566 & 0.515 \\ \toprule
  \multirow{8}{*}{\rotatebox{90}{Solar}}
                                & MSE    & 0.192 & 0.230 & 0.243 & 0.243 & {\bf\color{red}0.227} & 2.58\%  & 0.203 & 0.233 & 0.248 & 0.249 & \underline{\color{blue}0.233} & 0.290 & 0.320 & 0.353 & 0.356 & 0.330 & 0.884  & 0.834 & 0.941 & 0.882 & 0.885 & 0.236 & 0.217 & 0.249 & 0.241 & 0.235 \\ 
                                & MAE    & 0.222 & 0.251 & 0.263 & 0.265 & {\bf\color{red}0.250} & 4.58\%  & 0.237 & 0.261 & 0.273 & 0.275 & \underline{\color{blue}0.262} & 0.378 & 0.398 & 0.415 & 0.413 & 0.401 & 0.711  & 0.692 & 0.723 & 0.717 & 0.711 & 0.259 & 0.269 & 0.283 & 0.317 & 0.280 \\ 
                                & RMSP & 0.061 & 0.060 & 0.072 & 0.090   & {\bf\color{red}0.071} & 19.32\% & 0.080 & 0.089 & 0.090 & 0.093 & \underline{\color{blue}0.088} & 0.363 & 0.374 & 0.389 & 0.385 & 0.378 & 0.449  & 0.563 & 0.694 & 0.868 & 0.644 & 0.059 & 0.075 & 0.105 & 0.150 & 0.097 \\
                                & MAPE   & 1.813 & 1.948 & 1.984 & 2.081 & {\bf\color{red}1.957} & 2.25\%  & 1.842 & 2.052 & 1.986 & 2.129 & \underline{\color{blue}2.002} & 2.225 & 2.378 & 2.558 & 2.633 & 2.449 & 2.349  & 2.403 & 2.509 & 2.042 & 2.326 & 2.200 & 2.471 & 2.775 & 2.531 & 2.494 \\
                                & sMAPE  & 0.381 & 0.426 & 0.437 & 0.443 & {\bf\color{red}0.422} & 32.15\% & 0.394 & 0.416 & 0.422 & 1.256 & \underline{\color{blue}0.622} & 0.688 & 0.714 & 0.727 & 0.714 & 0.711 & 0.767  & 0.861 & 1.062 & 1.260 & 0.988 & 0.380 & 0.425 & 0.465 & 0.517 & 0.447 \\
                                & MASE   & 0.286 & 0.379 & 0.377 & 0.374 & {\bf\color{red}0.354} & 28.19\% & 0.301 & 0.369 & 0.362 & 0.939 & \underline{\color{blue}0.493} & 0.436 & 0.511 & 0.519 & 0.507 & \underline{\color{blue}0.493} & 0.562  & 0.670 & 0.817 & 0.905 & 0.739 & 0.273 & 0.344 & 0.363 & 0.415 & 0.349 \\
                                & Q25    & 0.225 & 0.242 & 0.254 & 0.258 & {\bf\color{red}0.245} & 32.13\% & 0.223 & 0.255 & 0.250 & 0.716 & \underline{\color{blue}0.361} & 0.449 & 0.476 & 0.496 & 0.491 & 0.478 & 0.453  & 0.519 & 0.647 & 0.712 & 0.583 & 0.229 & 0.272 & 0.327 & 0.329 & 0.289 \\
                                & Q75    & 0.230 & 0.283 & 0.294 & 0.301 & {\bf\color{red}0.277} & 25.94\% & 0.250 & 0.253 & 0.277 & 0.715 & \underline{\color{blue}0.374} & 0.308 & 0.321 & 0.336 & 0.336 & 0.325 & 0.504  & 0.544 & 0.636 & 0.682 & 0.592 & 0.229 & 0.262 & 0.280 & 0.333 & 0.276 \\ \toprule
  \multirow{8}{*}{\rotatebox{90}{ILI}}
                                & MSE    & 2.065 & 1.917 & 1.966 & 2.114 & {\bf\color{red}2.016} & 10.88\% & 2.329 & 2.189 & 2.217 & 2.314 & \underline{\color{blue}2.262} & 3.052 & 2.804 & 2.829 & 2.973 & 2.915 & 3.410  & 3.365 & 3.125 & 2.847 & 3.187 & 5.421 & 5.001 & 5.098 & 5.312 & 5.208 \\ 
                                & MAE    & 0.835 & 0.852 & 0.854 & 0.900 & {\bf\color{red}0.860} & 10.23\% & 0.939 & 0.946 & 0.956 & 0.989 & \underline{\color{blue}0.958} & 1.245 & 1.153 & 1.161 & 1.193 & 1.188 & 1.296  & 1.252 & 1.200 & 1.146 & 1.224 & 1.606 & 1.549 & 1.564 & 1.584 & 1.576 \\ 
                                & RMSP & 0.341 & 0.371 & 0.371 & 0.386   & {\bf\color{red}0.367} & 11.57\% & 0.393 & 0.413 & 0.421 & 0.434 & \underline{\color{blue}0.415} & 0.608 & 0.550 & 0.561 & 0.576 & 0.574 & 0.643  & 0.590 & 0.581 & 0.570 & 0.596 & 0.827 & 0.791 & 0.778 & 0.773 & 0.792 \\
                                & MAPE   & 3.701 & 3.193 & 3.054 & 2.874 & 3.206 & -3.05\% & 3.638 & 3.169 & 2.937 & 2.701 & {\bf\color{red}3.111} & 4.334 & 2.953 & 2.662 & 2.504 & \underline{\color{blue}3.113} & 5.282  & 4.880 & 4.463 & 3.753 & 4.595 & 2.545 & 2.625 & 2.545 & 2.260 & 2.494 \\
                                & sMAPE  & 0.633 & 0.657 & 0.649 & 0.660 & {\bf\color{red}0.650} & 7.54\%  & 0.683 & 0.702 & 0.707 & 0.718 & \underline{\color{blue}0.703} & 0.903 & 0.840 & 0.849 & 0.866 & 0.865 & 0.957  & 0.897 & 0.899 & 0.882 & 0.909 & 1.203 & 1.175 & 1.172 & 1.181 & 1.183 \\
                                & MASE   & 0.557 & 0.520 & 0.549 & 0.668 & {\bf\color{red}0.574} & 11.42\% & 0.622 & 0.587 & 0.629 & 0.753 & \underline{\color{blue}0.648} & 0.860 & 0.715 & 0.726 & 0.861 & 0.791 & 0.777  & 0.718 & 0.716 & 0.747 & 0.740 & 0.932 & 0.856 & 0.861 & 0.980 & 0.907 \\
                                & Q25    & 0.821 & 0.761 & 0.742 & 0.740 & {\bf\color{red}0.766} & 5.20\%  & 0.873 & 0.798 & 0.789 & 0.771 & \underline{\color{blue}0.808} & 0.986 & 0.825 & 0.804 & 0.792 & 0.852 & 1.248  & 1.150 & 1.054 & 0.941 & 1.098 & 0.824 & 0.833 & 0.857 & 0.876 & 0.848 \\
                                & Q75    & 0.849 & 0.943 & 0.966 & 1.061 & {\bf\color{red}0.955} & 13.73\% & 1.005 & 1.094 & 1.123 & 1.207 & \underline{\color{blue}1.107} & 1.503 & 1.480 & 1.518 & 1.593 & 1.524 & 1.345  & 1.354 & 1.346 & 1.350 & 1.349 & 2.388 & 2.265 & 2.271 & 2.293 & 2.304 \\ \toprule
  \multirow{8}{*}{\rotatebox{90}{OiK\_Weather}} 
                                & MSE    & 0.631 & 0.686 & 0.706 & 0.733 & \underline{\color{blue}0.689} & 4.31\%  & 0.672 & 0.718 & 0.734 & 0.755 & 0.720 & 0.662 & 0.697 & 0.713 & 0.731 & 0.701 & 0.733  & 0.820 & 0.899 & 0.960 & 0.853 & 0.569 & 0.647 & 0.713 & 0.752 & {\bf\color{red}0.670} \\
                                & MAE    & 0.582 & 0.616 & 0.630 & 0.645 & {\bf\color{red}0.618} & 2.22\%  & 0.601 & 0.630 & 0.642 & 0.654 & 0.632 & 0.613 & 0.635 & 0.645 & 0.656 & 0.637 & 0.642  & 0.696 & 0.739 & 0.769 & 0.712 & 0.560 & 0.617 & 0.660 & 0.662 & \underline{\color{blue}0.625} \\
                                & RMSP & 0.712 & 0.770 & 0.790 & 0.820   & {\bf\color{red}0.773} & 1.02\%  & 0.727 & 0.776 & 0.800 & 0.820 & \underline{\color{blue}0.781} & 0.764 & 0.800 & 0.817 & 0.837 & 0.805 & 0.777  & 0.865 & 0.925 & 0.966 & 0.883 & 0.712 & 0.787 & 0.855 & 0.825 & 0.795 \\
                                & MAPE   & 4.005 & 4.934 & 5.270 & 4.378 & {\bf\color{red}4.647} & 41.81\% & 9.158 & 7.506 & 7.768 & 7.510 & 7.986 & 5.019 & 4.616 & 4.613 & 4.400 & 4.662 & 10.976 & 8.417 & 9.650 & 9.859 & 9.726 & 6.597 & 7.261 & 7.739 & 7.372 & 7.242 \\
                                & sMAPE  & 1.024 & 1.079 & 1.101 & 1.125 & {\bf\color{red}1.082} & 0.73\%  & 1.040 & 1.085 & 1.110 & 1.126 & \underline{\color{blue}1.090} & 1.133 & 1.178 & 1.200 & 1.229 & 1.185 & 1.087  & 1.175 & 1.243 & 1.299 & 1.201 & 1.035 & 1.123 & 1.201 & 1.143 & 1.126 \\
                                & MASE   & 0.897 & 0.888 & 0.889 & 0.876 & \underline{\color{blue}0.888} & 5.03\%  & 0.947 & 0.921 & 0.939 & 0.931 & 0.935 & 0.966 & 0.944 & 0.935 & 0.925 & 0.943 & 0.878  & 0.943 & 1.042 & 1.103 & 0.992 & 0.781 & 0.814 & 0.868 & 0.807 & {\bf\color{red}0.818} \\
                                & Q25    & 0.563 & 0.605 & 0.629 & 0.638 & {\bf\color{red}0.609} & 3.03\%  & 0.602 & 0.624 & 0.633 & 0.652 & 0.628 & 0.609 & 0.631 & 0.641 & 0.652 & 0.633 & 0.668  & 0.702 & 0.739 & 0.768 & 0.719 & 0.547 & 0.608 & 0.642 & 0.645 & \underline{\color{blue}0.611} \\
                                & Q75    & 0.600 & 0.628 & 0.631 & 0.653 & {\bf\color{red}0.628} & 1.26\%  & 0.600 & 0.636 & 0.651 & 0.655 & \underline{\color{blue}0.636} & 0.616 & 0.639 & 0.649 & 0.661 & 0.641 & 0.617  & 0.690 & 0.739 & 0.771 & 0.704 & 0.574 & 0.626 & 0.679 & 0.680 & 0.640 \\ \toprule
  \multirow{8}{*}{\rotatebox{90}{NN5}}
                                & MSE    & 0.798 & 0.737 & 0.680 & 0.659 & {\bf\color{red}0.719} & 0.55\%  & 0.793 & 0.745 & 0.685 & 0.668 &  \underline{\color{blue}0.723} & 1.474 & 1.484 & 1.406 & 1.429 & 1.448 & 0.891  & 0.915 & 0.739 & 0.859 & 0.851 & 0.969 & 0.983 & 0.939 & 0.977 & 0.967 \\
                                & MAE    & 0.590 & 0.582 & 0.570 & 0.568 & {\bf\color{red}0.578} & 0.86\%  & 0.593 & 0.589 & 0.574 & 0.574 &  \underline{\color{blue}0.583} & 0.929 & 0.943 & 0.916 & 0.927 & 0.929 & 0.659  & 0.683 & 0.617 & 0.678 & 0.659 & 0.714 & 0.734 & 0.725 & 0.733 & 0.727 \\
                                & RMSP & 0.526 & 0.539 & 0.545 & 0.554   & {\bf\color{red}0.541} & 2.35\%  & 0.536 & 0.554 & 0.557 & 0.570 &  \underline{\color{blue}0.554} & 0.942 & 0.988 & 0.964 & 0.980 & 0.969 & 0.614  & 0.653 & 0.606 & 0.669 & 0.636 & 0.702 & 0.750 & 0.749 & 0.753 & 0.739 \\
                                & MAPE   & 2.332 & 2.563 & 2.426 & 2.431 & {\bf\color{red}2.438} & 5.28\%  & 2.491 & 2.640 & 2.560 & 2.605 &  \underline{\color{blue}2.574} & 3.845 & 3.785 & 3.848 & 3.763 & 3.810 & 2.700  & 2.936 & 2.642 & 2.760 & 2.760 & 2.698 & 2.893 & 2.823 & 3.006 & 2.855 \\
                                & sMAPE  & 0.848 & 0.862 & 0.863 & 0.871 & {\bf\color{red}0.861} & 1.82\%  & 0.860 & 0.877 & 0.878 & 0.892 &  \underline{\color{blue}0.877} & 1.276 & 1.326 & 1.296 & 1.316 & 1.304 & 0.954  & 0.992 & 0.941 & 1.014 & 0.975 & 1.046 & 1.092 & 1.096 & 1.088 & 1.081 \\
                                & MASE   & 0.528 & 0.541 & 0.530 & 0.519 & {\bf\color{red}0.530} & 2.21\%  & 0.538 & 0.553 & 0.543 & 0.534 &  \underline{\color{blue}0.542} & 0.820 & 0.856 & 0.830 & 0.833 & 0.835 & 0.589  & 0.623 & 0.572 & 0.628 & 0.603 & 0.626 & 0.685 & 0.670 & 0.667 & 0.662 \\
                                & Q25    & 0.598 & 0.607 & 0.594 & 0.587 & {\bf\color{red}0.597} & 0.67\%  & 0.601 & 0.611 & 0.597 & 0.593 &  \underline{\color{blue}0.601} & 0.876 & 0.897 & 0.867 & 0.880 & 0.880 & 0.679  & 0.708 & 0.644 & 0.698 & 0.682 & 0.730 & 0.730 & 0.748 & 0.717 & 0.731 \\
                                & Q75    & 0.582 & 0.558 & 0.546 & 0.550 & {\bf\color{red}0.559} & 1.06\%  & 0.585 & 0.568 & 0.550 & 0.555 &  \underline{\color{blue}0.565} & 0.982 & 0.989 & 0.966 & 0.974 & 0.978 & 0.640  & 0.659 & 0.590 & 0.659 & 0.637 & 0.698 & 0.739 & 0.702 & 0.748 & 0.722 \\ \toprule
  \multirow{7}{*}{\rotatebox{90}{Rideshare}}
                                & MSE    & 0.225 & 0.261 & 0.467 & 0.504 & {\bf\color{red}0.364} & 26.61\% & 0.247 & 0.273 & 0.580 & 0.885 &  \underline{\color{blue}0.496} & 0.933 & 1.148 & 1.063 & 0.967 & 1.028 & 0.339  & 0.514 & 0.791 & 0.793 & 0.609 & 0.846 & 0.435 & 0.537 & 0.628 & 0.612 \\
                                & MAE    & 0.273 & 0.296 & 0.412 & 0.445 & {\bf\color{red}0.357} & 13.77\% & 0.291 & 0.300 & 0.473 & 0.591 &  \underline{\color{blue}0.414} & 0.771 & 0.864 & 0.824 & 0.801 & 0.815 & 0.416  & 0.528 & 0.673 & 0.681 & 0.575 & 0.628 & 0.399 & 0.472 & 0.524 & 0.506 \\
                                & RMSP & 0.177 & 0.189 & 0.256 & 0.310   & {\bf\color{red}0.233} & 11.74\% & 0.187 & 0.186 & 0.312 & 0.369 &  \underline{\color{blue}0.264} & 0.777 & 0.874 & 0.866 & 0.842 & 0.840 & 0.401  & 0.503 & 0.632 & 0.638 & 0.544 & 0.494 & 0.248 & 0.332 & 0.384 & 0.365 \\
                                & sMAPE  & 0.388 & 0.406 & 0.589 & 0.652 & {\bf\color{red}0.509} & 11.48\% & 0.420 & 0.423 & 0.675 & 0.781 &  \underline{\color{blue}0.575} & 1.220 & 1.314 & 1.288 & 1.309 & 1.283 & 0.603  & 0.817 & 1.022 & 1.038 & 0.870 & 0.865 & 0.554 & 0.657 & 0.721 & 0.699 \\
                                & MASE   & 1.401 & 1.294 & 1.218 & 1.275 & 1.297 & -3.10\% & 1.342 & 1.299 & 1.183 & 1.208 & 1.258 & 1.025 & 1.030 & 1.047 & 1.094 & \underline{\color{blue}1.049} & 1.074  & 1.060 & 1.066 & 1.128 & 1.082 & 1.027 & 1.052 & 1.073 & 1.137 & \underline{\color{blue}1.072} \\
                                & Q25    & 0.202 & 0.230 & 0.247 & 0.258 & {\bf\color{red}0.234} & 8.24\%  & 0.217 & 0.211 & 0.272 & 0.320 & \underline{\color{blue}0.255} & 0.417 & 0.459 & 0.437 & 0.421 & 0.434 & 0.330  & 0.300 & 0.355 & 0.351 & 0.334 & 0.355 & 0.257 & 0.285 & 0.318 & 0.304 \\
                                & Q75    & 0.344 & 0.362 & 0.578 & 0.633 & {\bf\color{red}0.479} & 16.40\% & 0.365 & 0.390 & 0.675 & 0.863 & \underline{\color{blue}0.573} & 1.126 & 1.270 & 1.211 & 1.182 & 1.197 & 0.503  & 0.757 & 0.992 & 1.011 & 0.816 & 0.902 & 0.541 & 0.658 & 0.730 & 0.708 \\ \bottomrule
  \multicolumn{2}{c|}{Univariate}              & \multicolumn{6}{c|}{DeepBooTS}    & \multicolumn{5}{c|}{iTransformer}             & \multicolumn{5}{c|}{N-Beats}       & \multicolumn{5}{c|}{N-Hits}      & \multicolumn{5}{c}{Autoformer}       \\ \toprule
  Dataset                       & Metric & Y1    & Y2    & Y3    & Y4    & Avg   & IMP     & Y1    & Y2    & Y3    & Y4    & Avg   & Y1    & Y2    & Y3    & Y4    & Avg   & Y1     & Y2    & Y3    & Y4    & Avg   & Y1    & Y2    & Y3    & Y4    & Avg   \\ \toprule  
  \multirow{8}{*}{\rotatebox{90}{M4 Hourly}}         
                                & MSE    & 0.177 & 0.203 & 0.242 & 0.270  & {\bf\color{red}0.223} & 28.06\% & 0.277 & 0.290 & 0.326 & 0.345  & \underline{\color{blue}0.310} & {0.326}  & 0.335  & 0.342  & 0.345 & 0.337 & {0.300}  & 0.342  & 0.372  & 0.37   & 0.346 & {0.598}  & 0.653  & 0.587  & 0.703   & 0.635 \\ \toprule
                                & MAE    & 0.245 & 0.270 & 0.304 & 0.328  & {\bf\color{red}0.287} & 26.41\% & 0.369 & 0.373 & 0.401 & 0.415  & \underline{\color{blue}0.390} & {0.412}  & 0.423  & 0.427  & 0.431 & 0.423 & {0.39}   & 0.421  & 0.433  & 0.428  & 0.418 & {0.583}  & 0.621  & 0.589  & 0.649   & 0.611 \\
                                & RMDSPE & 0.182 & 0.215 & 0.248 & 0.275  & {\bf\color{red}0.230} & 38.83\% & 0.354 & 0.354 & 0.389 & 0.405  & \underline{\color{blue}0.376} & {0.375}  & 0.394  & 0.396  & 0.403 & 0.392 & {0.367}  & 0.41   & 0.418  & 0.414  & 0.402 & {0.576}  & 0.635  & 0.592  & 0.651   & 0.614 \\
                                & MAPE   & 1.364 & 1.478 & 1.688 & 1.819  & {\bf\color{red}1.587} & 27.43\% & 2.108 & 2.109 & 2.224 & 2.305  & \underline{\color{blue}2.187} & {2.532}  & 2.576  & 2.510  & 2.614 & 2.558 & {2.344}  & 2.523  & 2.569  & 2.531  & 2.492 & {3.830}  & 4.019  & 3.701  & 4.369   & 3.980 \\
                                & sMAPE  & 0.445 & 0.488 & 0.534 & 0.569  & {\bf\color{red}0.509} & 22.41\% & 0.631 & 0.636 & 0.669 & 0.687  & \underline{\color{blue}0.656} & {0.680}  & 0.701  & 0.707  & 0.711 & 0.700 & {0.644}  & 0.689  & 0.708  & 0.705  & 0.687 & {0.853}  & 0.896  & 0.865  & 0.909   & 0.881 \\
                                & MASE   & 0.226 & 0.265 & 0.292 & 0.322  & {\bf\color{red}0.276} & 24.59\% & 0.339 & 0.354 & 0.376 & 0.395  & \underline{\color{blue}0.366} & {0.365}  & 0.396  & 0.399  & 0.398 & 0.390 & {0.359}  & 0.395  & 0.415  & 0.415  & 0.396 & {0.481}  & 0.522  & 0.521  & 0.546   & 0.518 \\
                                & Q25    & 0.230 & 0.249 & 0.276 & 0.294  & {\bf\color{red}0.262} & 29.00\% & 0.353 & 0.354 & 0.379 & 0.389  & 0.369 & {0.344}  & 0.351  & 0.358  & 0.360 & 0.353 & {0.38}   & 0.4    & 0.407  & 0.394  & 0.395 & {0.460}  & 0.451  & 0.449  & 0.563   & 0.481 \\
                                & Q75    & 0.259 & 0.291 & 0.332 & 0.362  & {\bf\color{red}0.311} & 24.15\% & 0.384 & 0.392 & 0.424 & 0.441  & \underline{\color{blue}0.410} & {0.479}  & 0.496  & 0.496  & 0.502 & 0.493 & {0.401}  & 0.442  & 0.46   & 0.463  & 0.442 & {0.707}  & 0.792  & 0.729  & 0.735   & 0.741 \\ \toprule
  \multirow{8}{*}{\rotatebox{90}{Us\_births}} 
                                & MSE    & 0.316 & 0.401 & 0.334 & 0.405  & {\bf\color{red}0.364} & 3.70\%  & 0.353 & 0.437 & 0.354 & 0.367  & \underline{\color{blue}0.378} & {0.562}  & 0.705  & 0.608  & 0.804 & 0.670 & {0.371}  & 0.528  & 0.448  & 0.806  & 0.538 & {0.545}  & 0.812  & 0.984  & 1.264   & 0.901 \\
                                & MAE    & 0.391 & 0.459 & 0.412 & 0.462  & {\bf\color{red}0.431} & 1.82\%  & 0.413 & 0.473 & 0.434 & 0.436  & \underline{\color{blue}0.439} & {0.586}  & 0.653  & 0.599  & 0.715 & 0.638 & {0.467}  & 0.579  & 0.529  & 0.737  & 0.578 & {0.583}  & 0.709  & 0.793  & 0.943   & 0.757 \\
                                & RMDSPE & 0.217 & 0.267 & 0.235 & 0.271  & {\bf\color{red}0.248} & 1.20\%  & 0.228 & 0.268 & 0.252 & 0.256  & \underline{\color{blue}0.251} & {0.376}  & 0.414  & 0.377  & 0.469 & 0.409 & {0.283}  & 0.363  & 0.331  & 0.477  & 0.364 & {0.356}  & 0.417  & 0.458  & 0.558   & 0.447 \\
                                & MAPE   & 0.754 & 0.765 & 0.709 & 0.784  & {\bf\color{red}0.753} & 3.71\%  & 0.776 & 0.817 & 0.813 & 0.720  & \underline{\color{blue}0.782} & {0.974}  & 1.121  & 0.932  & 1.037 & 1.016 & {0.851}  & 1.024  & 0.879  & 1.201  & 0.989 & {0.844}  & 0.968  & 1.235  & 1.447   & 1.124 \\
                                & sMAPE  & 0.403 & 0.462 & 0.416 & 0.497  & \underline{\color{blue}0.445} & -1.60\% & 0.409 & 0.466 & 0.420 & 0.458  & {\bf\color{red}0.438} & {0.585}  & 0.628  & 0.590  & 0.698 & 0.625 & {0.447}  & 0.536  & 0.492  & 0.684  & 0.540 & {0.634}  & 0.762  & 0.819  & 0.940   & 0.789 \\
                                & MASE   & 0.328 & 0.393 & 0.351 & 0.381  & {\bf\color{red}0.363} & 0.55\%  & 0.336 & 0.380 & 0.376 & 0.367  & \underline{\color{blue}0.365} & {0.504}  & 0.560  & 0.513  & 0.605 & 0.546 & {0.390}  & 0.486  & 0.438  & 0.542  & 0.464 & {0.546}  & 0.676  & 0.768  & 0.956   & 0.737 \\
                                & Q25    & 0.462 & 0.517 & 0.438 & 0.590  & 0.502 & -2.45\% & 0.497 & 0.559 & 0.395 & 0.510  & 0.490 & {0.440}  & 0.459  & 0.421  & 0.450 & \underline{\color{blue}0.443} & {0.309}  & 0.356  & 0.332  & 0.419  & {\bf\color{red}0.354} & {0.587}  & 0.676  & 0.753  & 0.891   & 0.727 \\
                                & Q75    & 0.321 & 0.401 & 0.387 & 0.335  & {\bf\color{red}0.361} & 6.96\%  & 0.329 & 0.386 & 0.472 & 0.363  & \underline{\color{blue}0.388} & {0.732}  & 0.846  & 0.777  & 0.979 & 0.834 & {0.624}  & 0.801  & 0.727  & 1.056  & 0.802 & {0.580}  & 0.743  & 0.833  & 0.995   & 0.788 \\ \toprule
  \multirow{8}{*}{\rotatebox{90}{Sunspot}}    
                                & MSE    & 1.185 & 1.265 & 1.170 & 1.155  & 1.194 & 1.00\%  & 1.207 & 1.286 & 1.175 & 1.156  & 1.206 & {1.044}  & 1.049  & 1.009  & 1.010 & \underline{\color{blue}1.028} & {0.990}  & 1.004  & 0.971  & 0.961  & {\bf\color{red}0.982} & {1.220}  & 1.290  & 1.224  & 1.275   & 1.252 \\
                                & MAE    & 0.535 & 0.564 & 0.542 & 0.536  & 0.544 & 1.45\%  & 0.542 & 0.572 & 0.548 & 0.545  & 0.552 & {0.539}  & 0.550  & 0.541  & 0.540 & \underline{\color{blue}0.543} & {0.544}  & 0.532  & 0.526  & 0.525  & {\bf\color{red}0.532} & {0.607}  & 0.660  & 0.640  & 0.652   & 0.640 \\
                                & RMDSPE & 0.805 & 0.838 & 0.798 & 0.789  & {\bf\color{red}0.808} & 2.88\%  & 0.810 & 0.868 & 0.822 & 0.828  & \underline{\color{blue}0.832} & {0.960}  & 0.984  & 0.954  & 0.949 & 0.962 & {0.880}  & 0.887  & 0.823  & 0.870  & 0.865 & {1.024}  & 1.087  & 1.046  & 1.073   & 1.058 \\
                                & MAPE   & 2.725 & 2.716 & 2.818 & 2.844  & 2.776 & 2.87\%  & 2.817 & 2.824 & 2.886 & 2.906  & 2.858 & {1.896}  & 1.941  & 2.090  & 2.142 & {\bf\color{red}2.017} & {2.770}  & 2.315  & 2.729  & 2.575  & \underline{\color{blue}2.597} & {2.487}  & 2.928  & 2.775  & 3.033   & 2.806 \\
                                & sMAPE  & 1.044 & 1.059 & 1.056 & 1.051  & {\bf\color{red}1.053} & 1.22\%  & 1.028 & 1.083 & 1.070 & 1.083  & \underline{\color{blue}1.066} & {1.406}  & 1.419  & 1.369  & 1.352 & 1.387 & {1.186}  & 1.214  & 1.139  & 1.216  & 1.189 & {1.349}  & 1.416  & 1.396  & 1.317   & 1.370 \\
                                & MASE   & 1.413 & 1.315 & 1.134 & 1.065  & 1.232 & 2.07\%  & 1.518 & 1.275 & 1.133 & 1.107  & 1.258 & {1.069}  & 0.925  & 0.891  & 0.885 & \underline{\color{blue}0.943} & {0.926}  & 0.863  & 0.811  & 0.830  & {\bf\color{red}0.858} & {1.022}  & 1.014  & 1.060  & 1.130   & 1.057 \\
                                & Q25    & 0.391 & 0.397 & 0.416 & 0.416  & {\bf\color{red}0.405} & 2.64\%  & 0.391 & 0.406 & 0.426 & 0.442  & \underline{\color{blue}0.416} & {0.509}  & 0.514  & 0.514  & 0.518 & 0.514 & {0.523}  & 0.486  & 0.488  & 0.512  & 0.502 & {0.545}  & 0.616  & 0.619  & 0.641   & 0.605 \\
                                & Q75    & 0.669 & 0.701 & 0.670 & 0.657  & 0.674 & 1.89\%  & 0.692 & 0.738 & 0.670 & 0.648  & 0.687 & {0.569}  & 0.586  & 0.568  & 0.563 & \underline{\color{blue}0.572} & {0.566}  & 0.578  & 0.565  & 0.538  & {\bf\color{red}0.562} & {0.669}  & 0.704  & 0.660  & 0.664   & 0.674 \\ \toprule
  \multirow{8}{*}{\rotatebox{90}{Saugeenday}} 
                                & MSE    & 0.276 & 0.31  & 0.362 & 0.56   & 0.377 & 1.31\%  & 0.279 & 0.312 & 0.366 & 0.569  & 0.382 & {0.276}  & 0.308  & 0.357  & 0.512 & {\bf\color{red}0.363} & {0.282}  & 0.312  & 0.362  & 0.505  & \underline{\color{blue}0.365} & {0.296}  & 0.324  & 0.373  & 0.607   & 0.400 \\
                                & MAE    & 0.371 & 0.397 & 0.431 & 0.543  & {\bf\color{red}0.436} & 0.91\%  & 0.374 & 0.399 & 0.436 & 0.551  & \underline{\color{blue}0.440}  & {0.378}  & 0.405  & 0.444  & 0.552 & 0.445 & {0.384}  & 0.401  & 0.449  & 0.534  & 0.442 & {0.393}  & 0.412  & 0.442  & 0.584   & 0.458 \\
                                & RMDSPE & 0.367 & 0.398 & 0.439 & 0.575  & {\bf\color{red}0.445} & 1.11\%  & 0.369 & 0.4   & 0.444 & 0.587  & \underline{\color{blue}0.450}  & {0.373}  & 0.406  & 0.458  & 0.657 & 0.474 & {0.382}  & 0.404  & 0.456  & 0.608  & 0.463 & {0.392}  & 0.413  & 0.453  & 0.636   & 0.474 \\
                                & MAPE   & 71.15 & 78.47 & 86.28 & 100.96 & 84.21 & 1.36\%  & 71.77 & 78.79 & 87.96 & 102.97 & 85.37 & {67.265} & 71.937 & 77.889 & 74.39 & {\bf\color{red}72.87} & {63.643} & 70.235 & 81.053 & 77.865 & \underline{\color{blue}73.20}  & {72.197} & 78.669 & 87.303 & 107.281 & 86.36 \\
                                & sMAPE  & 0.671 & 0.709 & 0.752 & 0.892  & {\bf\color{red}0.756} & 0.66\%  & 0.674 & 0.711 & 0.758 & 0.901  & \underline{\color{blue}0.761} & {0.687}  & 0.732  & 0.793  & 1.003 & 0.804 & {0.701}  & 0.724  & 0.791  & 0.949  & 0.791 & {0.703}  & 0.73   & 0.768  & 0.944   & 0.786 \\
                                & MASE   & 1.211 & 1.234 & 1.323 & 1.688  & 1.364 & 1.23\%  & 1.229 & 1.293 & 1.34  & 1.66   & 1.381 & {1.199}  & 1.235  & 1.307  & 1.357 & \underline{\color{blue}1.275} & {1.282}  & 1.268  & 1.367  & 1.434  & 1.338 & {1.092}  & 1.143  & 1.229  & 1.556   & {\bf\color{red}1.255} \\
                                & Q25    & 0.369 & 0.397 & 0.425 & 0.536  & {\bf\color{red}0.432} & 2.92\%  & 0.375 & 0.401 & 0.443 & 0.561  & \underline{\color{blue}0.445} & {0.387}  & 0.42   & 0.465  & 0.584 & 0.464 & {0.394}  & 0.393  & 0.49   & 0.537  & 0.454 & {0.406}  & 0.431  & 0.447  & 0.639   & 0.481 \\
                                & Q75    & 0.374 & 0.393 & 0.435 & 0.535  & 0.434 & 0.23\%  & 0.373 & 0.397 & 0.429 & 0.542  & 0.435 & {0.368}  & 0.391  & 0.422  & 0.519 & {\bf\color{red}0.425} & {0.373}  & 0.408  & 0.409  & 0.53   & \underline{\color{blue}0.430} & {0.38}   & 0.394  & 0.437  & 0.53    & 0.435 \\ \bottomrule
  \end{tabular}
  }
\end{table}

\clearpage
\section{Visualization of TS forecasting}
\label{app_vis_tsf}

For clarity and comparison among different models, we present supplementary prediction showcases for three representative datasets in Fig. \ref{fig_traffic}, \ref{fig_elc}, and \ref{fig_weather}. 
Visualization of different models for qualitative comparisons. Prediction cases from the Traffic, Electricity and Weather datasets.
These showcases correspond to predictions made by the following models: iTransformer \cite{liu2023itransformer}  and PatchTST \cite{nie2022time}, DLinear \cite{zeng2023transformers}, Autoformer \cite{wu2021autoformer} and Informer \cite{Zhou2021Informer}.
Among the various models considered, the proposed DeepBooTS stands out for its ability to predict future series variations with exceptional precision, demonstrating superior performance.

\begin{figure*}[!ht]
  \centering
  \vspace{0.2em}
  \centerline{\includegraphics[width=\textwidth]{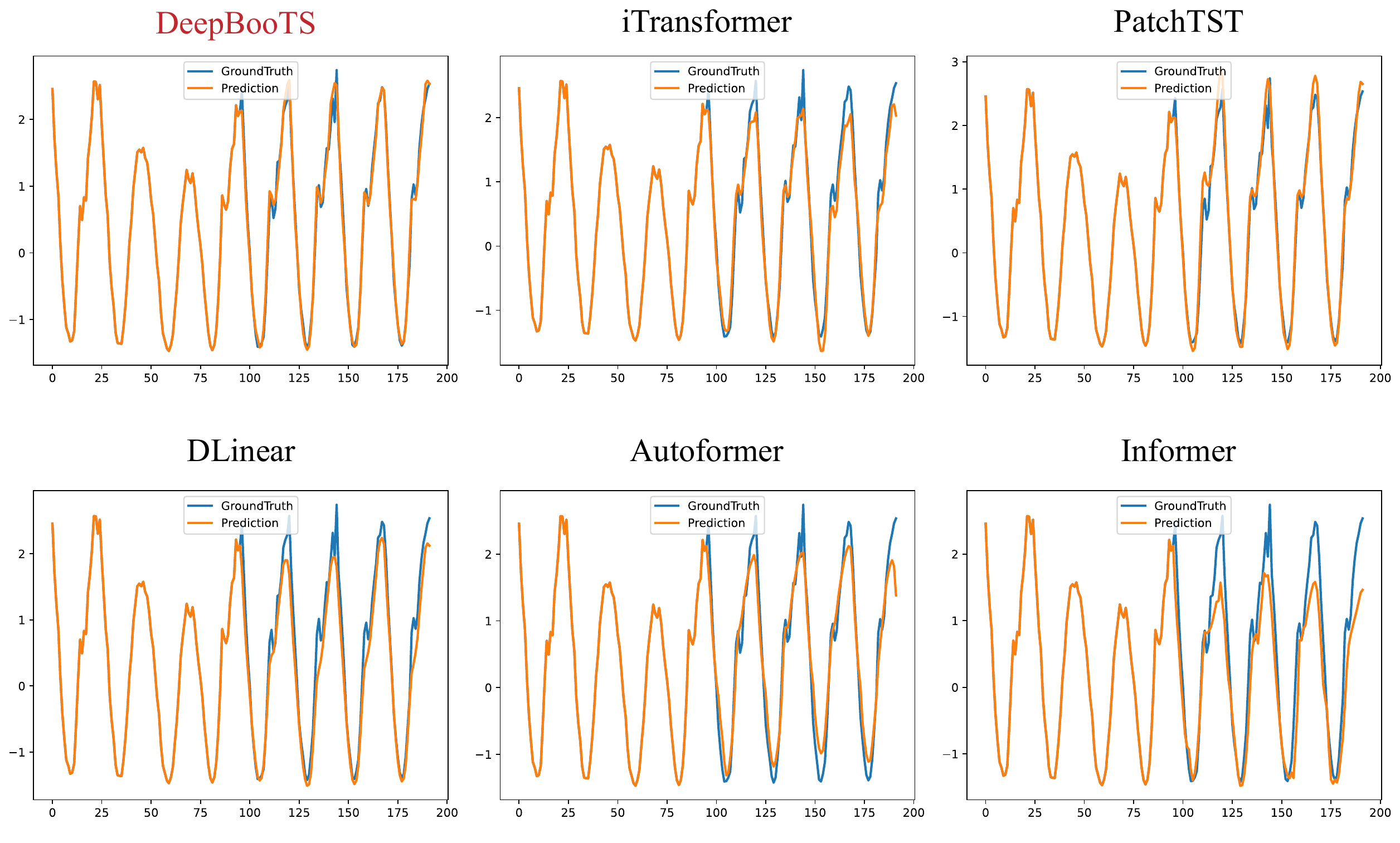}}
  \caption{Prediction cases from the Traffic dataset under the input-96-predict-96 setting.}
  \label{fig_traffic} 
\end{figure*}

\begin{figure*}[!ht]
  \centering
  \centerline{\includegraphics[width=0.95\textwidth]{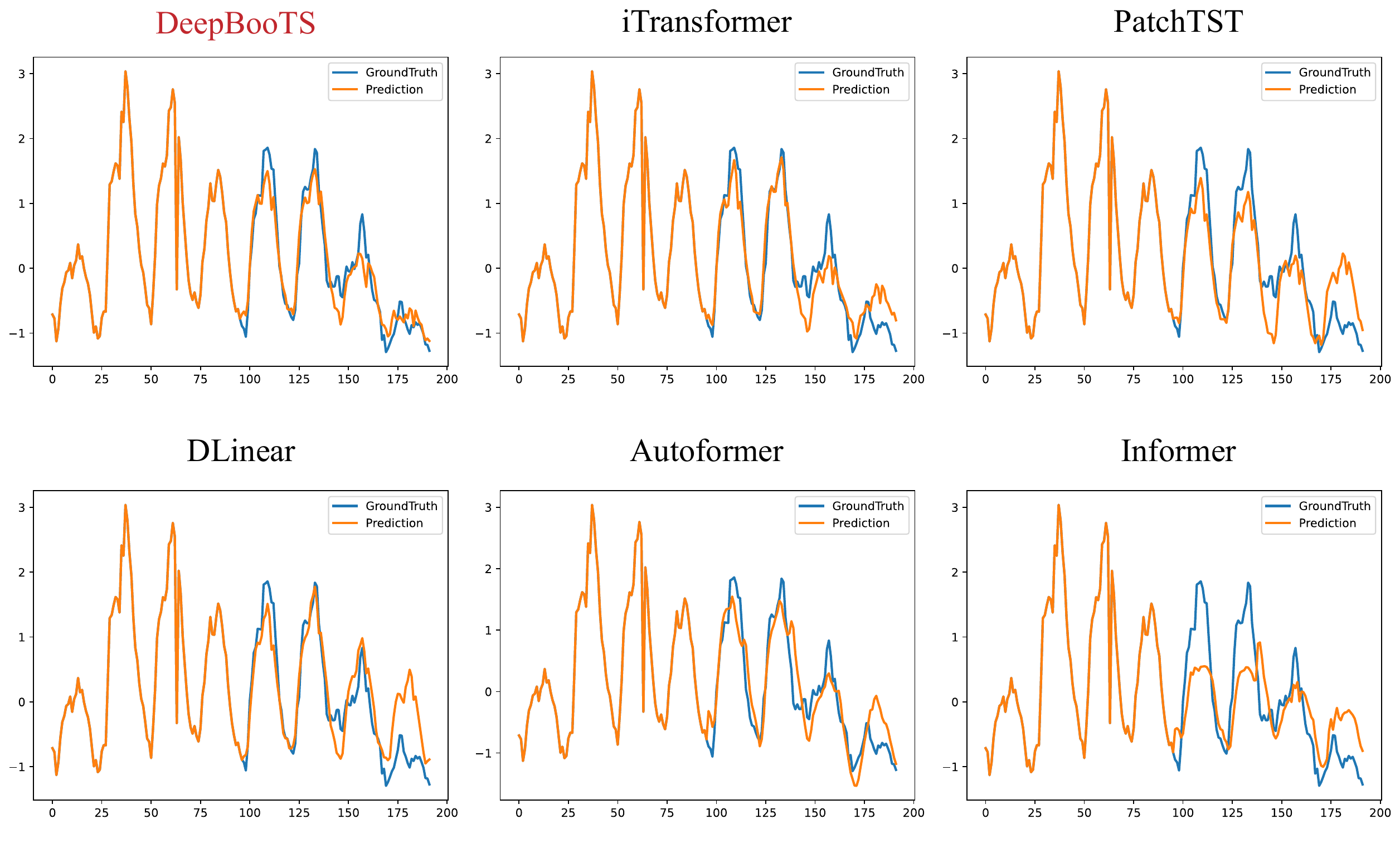}}
  \caption{Prediction cases from the Electricity dataset under the input-96-predict-96 setting.}
  \label{fig_elc} 
\end{figure*}

\begin{figure*}[!ht]
  \centering
  \centerline{\includegraphics[width=0.95\textwidth]{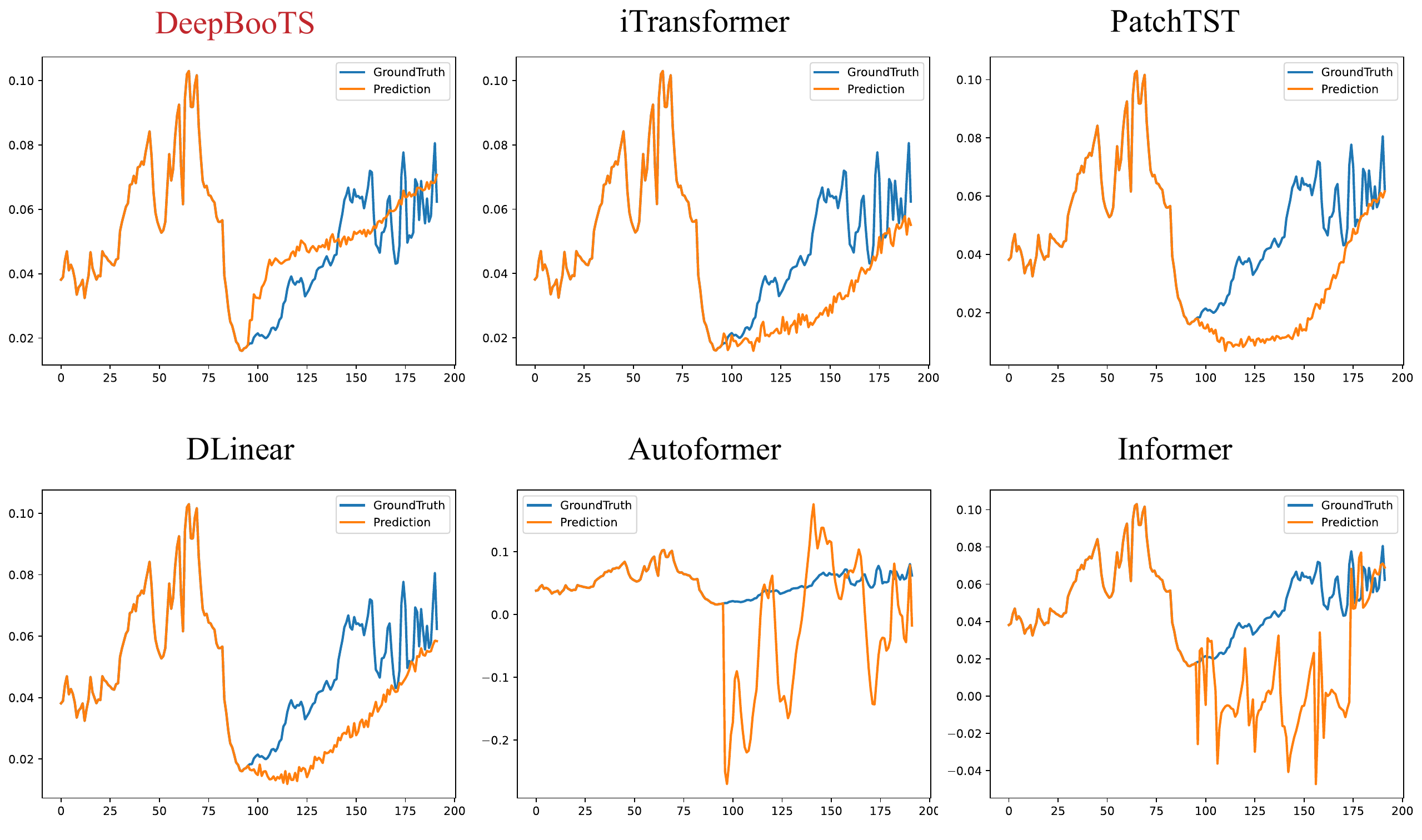}}
  \caption{Prediction cases from the Weather dataset under the input-96-predict-96 setting.}
  \label{fig_weather} 
\end{figure*}



\end{document}